%% file: IRL4RA_-_main.tex
\documentclass[twoside,letterpaper]{article}
\usepackage[linesnumbered,ruled,vlined]{algorithm2e}

\usepackage{mathtools}
\mathtoolsset{showonlyrefs=true}
\usepackage{cases}

\usepackage{amsthm, amsmath, amssymb, amsfonts, graphicx, epsfig}
\usepackage{accents}

\usepackage{natbib}

\usepackage{bbm}
\usepackage{mathrsfs}

\usepackage{cancel}

\usepackage{epstopdf}

\usepackage{graphicx}
\usepackage[font=sl,labelfont=bf]{caption}
\usepackage{subcaption}

\usepackage{float}
\restylefloat{table}

\usepackage{enumerate}

\usepackage[usenames,dvipsnames]{xcolor}

\usepackage[colorlinks=true, pdfstartview=FitV, linkcolor=blue,
            citecolor=blue, urlcolor=blue]{hyperref}

\usepackage{url}
\makeatletter\def\url@leostyle{%
 \@ifundefined{selectfont}{\def\UrlFont{\sf}}{\def\UrlFont{\scriptsize\ttfamily}}} \makeatother

\usepackage[framemethod=default]{mdframed}

\usepackage{array}

\usepackage[margin=2.5cm, letterpaper]{geometry}

\usepackage{multicol}
\usepackage{tikz,pgfplots}
\usetikzlibrary{shapes,arrows,trees,automata,positioning,matrix,decorations.pathreplacing,calc}
\pgfplotsset{compat=1.16}

\newtheorem{theorem}{Theorem}
\newtheorem{proposition}[theorem]{Proposition}
\newtheorem{lemma}[theorem]{Lemma}
\newtheorem{corollary}[theorem]{Corollary}
\newtheorem{assumption}[theorem]{Assumption}
\theoremstyle{definition}
\newtheorem{definition}[theorem]{Definition}

\theoremstyle{remark}
\newtheorem{remark}[theorem]{Remark}

\numberwithin{equation}{section}
\numberwithin{theorem}{section}

\definecolor{Red}{rgb}{1.0,0,0.0}
\definecolor{Blue}{rgb}{0,0.0,1.0}
\definecolor{Green}{rgb}{0.2,0.5,0.2}
\definecolor{Purple}{HTML}{9467bd}

\def\cE{\mathcal{E}}
\def\cF{\mathcal{F}}
\def\cG{\mathcal{G}}

\def\cJ{\mathcal{J}}

\def\cL{\mathcal{L}}
\def\cM{\mathcal{M}}

\def\cP{\mathcal{P}}

\def\cU{\mathcal{U}}

\def\cW{\mathcal{W}}

\def\bA{\mathbb{A}}

\def\bE{\mathbb{E}}

\def\bL{\mathbb{L}}

\def\bN{\mathbb{N}}

\def\bP{\mathbb{P}}
\def\bQ{\mathbb{Q}}
\def\bR{\mathbb{R}}
\def\bS{\mathbb{S}}
\def\bT{\mathbb{T}}

\def\bX{\mathbb{X}}

\def\fR{\mathfrak{R}}

\newcommand{\wt}{\widetilde}

\newcommand{\1}{\mathbbm{1}}            
\newcommand{\set}[1]{\{#1\}}            

\DeclareMathOperator{\dif}{d \!}        

\DeclareMathOperator*{\argmin}{arg\,min} 
\DeclareMathOperator*{\argmax}{arg\,max} 

\DeclareMathOperator{\avar}{\mathrm{AV}@\mathrm{R}}         
\DeclareMathOperator{\lip}{\mathrm{Lip}}

\DeclareMathOperator{\supp}{supp}

\DeclareMathOperator{\range}{range}

\usepackage[colorinlistoftodos,textsize=tiny]{todonotes}
\title{Eliciting Risk Aversion with Inverse Reinforcement Learning via Interactive Questioning}

\author{
    Ziteng Cheng 
    \thanks{Ordered alphabetically.}
	\thanks{Financial Technology Thrust, The Hong Kong University of Science and Technology (Guangzhou), China. \textbf{Email:} zitengcheng@hkust-gz.edu.cn. }
 \and
    Anthony Coache    
    \footnotemark[1]
	\thanks{Department of Mathematics, Imperial College London, United Kindom. \textbf{Email:} a.coache@imperial.ac.uk. }
 \and
     Sebastian Jaimungal      
    \footnotemark[1]
	\thanks{Department of Statistical Sciences, University of Toronto, Canada. \textbf{Email:} sebastian.jaimungal@utoronto.ca. } 
}

\date{
\today
}

\begin{document}

\maketitle




\smallskip

{\footnotesize
\begin{tabular}{l@{} p{350pt}}
 \hline \\[-.2em]
 \textsc{Abstract}: \ & We investigate a framework for robo-advisors to estimate non-expert clients' risk aversion using adaptive binary-choice questionnaires. We model risk aversion using cost functions and spectral risk measures in a static setting. We prove the finite-sample identifiability and, for properly designed questions, obtain a convergence rate of $\sqrt{N}$ up to a logarithmic factor, where $N$ is the number of questions. We introduce the notion of distinguishing power and demonstrate, through simulated experiments, that designing questions by maximizing distinguishing power achieves satisfactory accuracy in learning risk aversion with fewer than 50 questions. We also provide a preliminary investigation of an infinite-horizon setting with an additional discount factor for dynamic risk aversion, establishing qualitative identifiability in this case.

 \\[0.5em]
\textsc{Keywords:} \ & robo-advising, inverse reinforcement learning, active learning, design of experiments, spectral risk measures \\[0.5em]
\hline
\end{tabular}
}

\section{Introduction}\label{sec:Intro}

Robo-advisors aim to provide automated, customized wealth management services to a wide range of clients \citep[cf.][]{D'Acunto2021Robo, Capponi2022Personalized}. Providing tailored advice is contingent upon a clear understanding of the client's risk preferences, which is inherently subjective and varies by individual. However, accurately eliciting these preferences is challenging, as people often find it difficult to articulate their risk tolerance. This challenge is compounded by the robo-advisor's need for a quantitative performance criterion, which in turn requires a quantitative description of the client's risk preference.

Inverse reinforcement learning (IRL) offers a promising solution to this challenge. This field of study is dedicated to inferring agent's underlying objectives from their observed behavior. Specifically, IRL aims to estimate an agent's utility function by observing their assumed optimal policy within a known environment \citep{Ng2000Algorithms, Ng2004Apprenticeship, Ratliff2006Maximum, Ramachandran2007Bayesian, Ziebart2008Maximum}. As an illustration, suppose we have previously collected trajectories of cars on freeway entry ramps. Based on these trajectories, a typical risk-neutral IRL method would aim to estimate the utility functions of the drivers. However, this risk-neutrality assumption, often made for simplicity, may not accurately reflect real-world values. Indeed, meteorological events and road conditions may affect the safety and speed performances of drivers. A comprehensive model of human values should include components that describe risk aversion toward uncertainty in addition to utility functions.

As noted by \cite{Ng2000Algorithms}, a major issue in IRL is identifiability. This challenge becomes even more pronounced in a robo-advisor context, where modeling human risk preferences typically requires more than just a utility function, as evidenced by the Allais paradox \citep{Allais1953Le}. A further difficulty in applying IRL to a robo-advisor context stems from the client's role. The client, who corresponds to the agent in a standard IRL framework, is typically a non-expert with limited experience in wealth management. This situation presents two key issues: (i) there is little to no historical data on the client's past financial ``trajectory''; and (ii) when confronted with a sophisticated scenario comparable to real-world wealth management, the client is unlikely to determine the exact optimal policy that aligns with their own risk preferences.

To tackle these issues, we draw inspiration from an adaptive framework that allows the robo-advisor (i.e., the learner) to alter the environment in which the client (i.e., the agent) acts. The adaptive framework presents a series of evolving environments, with each alteration informed by previous interactions with the client. This framework has been widely used in psychology-related fields \citep{Wainer2000Computerized, Toubia2004Polyhedral, Cavagnaro2013Optimal, Fujita2023Adaptive} and is increasingly being adopted from an IRL perspective \citep{Lopes2009Active, Buening2022Interactive, Howes2023Towards}.

This study considers a risk preference modeled by a cost function of the state and a spectral risk measure (SRM). Popularized by \cite{Acerbi2002Spectral}, SRMs constitute a particularly useful class of convex risk measures. We acknowledge the multitude of existing models for human risk profile and the current lack of a consensus optimal framework; a list of relevant models may be found in, e.g., \cite{He2022Wisdom}. We refer to Remark~\ref{rem:SRMRationale} for our rationale of working with SRMs in this paper.

In our context, the alternative environments are presented as binary-choice questions based on monetary scenarios. We opt to employ simplistic binary-choice questions for several reasons. First, as we demonstrate in this paper, such simplistic questions are sufficient for eliciting risk profiles in a non-parametric manner, thereby offering new insights into the reach of the adaptive framework. Second, clients may find overly sophisticated questions overwhelming and difficult to handle, potentially reducing answer accuracy. Third, this simplicity helps mitigate the cost of designing an adaptive questionnaire.

Our study is conducted primarily in a static setting, with a preliminary investigation extended to a dynamic setting. We contribute to both theoretical and numerical aspects, as summarized below.
\begin{itemize}
\item Metric between risk aversion: To facilitate risk aversion modeling, we introduce a distance metric for SRMs in \eqref{eq:DefUtilde}. As SRMs can be characterized by a probability measure on $[0,1]$, under a mild condition, Proposition~\ref{prop:Utilde} establishes a two-sided bound relative to the Wasserstein-1 distance. This result reveals the totally boundedness of the risk aversion space under our metric, providing the necessary theoretical foundation for numerical implementation.
\item Identifiability: Theorem~\ref{thm:Separation} shows that for any two distinct risk aversions, there exists a binary question that elicits different choices. This finding leads to Theorem~\ref{thm:SuffFinite}, which provides guarantees that the adaptive framework can resolve the identifiability issue with $\varepsilon$-accuracy. As a consequence, we prove in Corollary~\ref{cor:RandomDesignConstitent} that designing questions uniformly at random leads to an asymptotically consistent estimation of the client's risk aversion.
\item Quantitative Rates: Beyond qualitative results, we define the distinguishing power of a binary question between two risk aversions in \eqref{eq:DefPsi} and \eqref{eq:DefXi}, and establish a corresponding lower bound and rate in Theorem~\ref{thm:QuantDisEnv} and Corollary~\ref{cor:LBXi}, respectively. Furthermore, under additional Lipschitz-continuity regularity, Theorem~\ref{thm:ConvRate} demonstrates that a properly designed sequence of questions achieves a convergence rate of square-root order, up to a logarithmic factor.
\item Numerical Methods: We devise IRL algorithms to design questions by maximizing the expected distinguishing power. The efficacy of the proposed methods is supported by two simulated experiments in Section~\ref{sec:Experiments}:
\begin{itemize}
\item The first experiment assumes a finite set of candidate risk aversions. A comparison with uniformly random sampled questions highlights the necessity of purposeful design for improved efficiency.
\item  The second experiment uses a proxy space that closely resembles the space of risk aversions (see Proposition~\ref{prop:PlambdaJ}). This method, where uncertainty in the client's risk aversion is represented by particles, has the ``true'' risk aversion formulated from a finer proxy space, while inference is conducted on a coarser subspace to reflect potential model misspecification. The results show that fewer than 50 questions are sufficient for satisfactory accuracy.
\end{itemize}
\item Qualitative results in more general settings: We conduct preliminary investigations in more general settings. In Section~\ref{subsec:MultipleChoice}, we show that if the cost function is known, a single multiple-choice question can distinguish between any finite set of candidate risk aversion profiles, provided the state space is sufficiently rich. In Section~\ref{subsec:InfHor}, we investigate a dynamic setting by additionally introducing a discount factor into the client's risk profile. For this case, we establish the existence of an analogous distinguishing binary-choice question.
\end{itemize}

The remainder of the paper is structured as follows. We first introduce some notations and the detailed setup for the static setting in Section~\ref{sec:Pre}. Section~\ref{sec:AuxQuantities} provides important quantities for the theoretical results within this work, including the optimality gap and distinguishing power. We then proceed to investigate identifiability in Section~\ref{sec:Iden} by establishing the existence of distinguishing questions. We derive under mild assumptions the explicit lower bound on the optimal distinguishing power in Section~\ref{sec:LowerBound}, and analyze under stronger assumptions the convergence rate of interactive questionnaires in Section~\ref{sec:ConvRate}. Section~\ref{sec:Experiments} is dedicated to the discussion of IRL algorithms and the empirical evaluation of the proposed design methods. Finally, we provide some discussions on potential further works in Section~\ref{sec:Extensions}, and conclude the paper in Section~\ref{sec:Conclusion}. The literature review and proofs are deferred to Appendices~\ref{sec:RelatedWorks} and~\ref{sec:Proofs}, respectively.

\section{Preliminaries}\label{sec:Pre}

\paragraph{Main goal.} In this paper, our primary objective is to investigate the IRL problem for an agent's risk aversion within an iterative experimental setting. This setting enables us to exercise complete control over the environment in which the agent operates. By observing the agent's policy, which is presumed to be optimal given their risk aversion, we aim to deduce the agent's risk aversion. In the remaining of this paper, following the convention in robo-advising, we refer to the agent as the client. Next, we give some notations for risk assessment to analyze this learning problem in the one-period case.

\paragraph{Notations.} $\delta_z$ is the Dirac measure at $z$. For a finite set $\bS$, we use $|\bS|$ to denote the cardinality. For a real-valued random variable $Z$ , we let $F_Z$ be the CDF of $Z$, and define $F^{-1}_Z(u) := \inf\{r\in\bR:F_Z(r)\ge u\}$. We let $\cP([0,1])$ be the set of Borel probability measures on $[0,1]$.

\paragraph{Spectral Risk Measures (SRMs).}  For any $\mu\in\cP([0,1])$ and real-valued integrable random variable $Z$, we define the following 
\begin{align}\label{eq:Defrhomu}
\rho_{\mu}(Z) := \int_0^1\avar_{1-\alpha}(Z)\;\mu(\dif\alpha),
\end{align}
where, for $\kappa\in(0,1]$,
\begin{align}\label{eq:DefAVaR}
\avar_\kappa(Z) := \inf_{r\in\bR}\left\{r+\frac{1}{\kappa}\bE((Z-r)_+)\right\},
\end{align}
i.e.,  $\avar$ is the average value at risk also called conditional value-at-risk \citep[cf.][]{Rockafellar2000Optimization}. In particular, $\rho_{\delta_0}$ is the expectation, i.e., $\rho_{\delta_0}(Z) = \avar_1(Z) = \mathbb{E}(Z)$. This class of risk measures is convenient to characterize trade-offs between different risk-aware objectives, for instance with a convex combination $\mu = \tau \,\delta_{0} + (1-\tau)\, \delta_{\kappa}$ (with $\tau\in(0,1)$, $\kappa\in(0,1]$) the agent emphasizes risk while still valuing loss.

\begin{remark}\label{rem:SRMRationale}
Our rationale for selecting SRMs, despite their potential limitations, is as follows:
\begin{itemize}
\item As a coherent risk measure, SRMs encourage risk diversification \citep[cf.][]{Artzner1999Coherent}, making them suitable for downstream tasks like portfolio management. In fact, SRMs constitute a key subclass of law-invariant convex risk measures \citep{Kusuoka2001Law, Jouini2006Law}, which provides a solid basis before venturing into more general risk measure formulations. An additional feature of SRMs is the discouragement of stochastic behavior \citep{Delage2019Dicesion}, which could be desirable for limiting operational risk.
\item SRMs can be characterized by a probability measure on $[0,1]$. The formulation of SRMs is inherently non-parametric,\footnote{Informally speaking, the richness of the ``parameter space'' is akin to that of an infinite-dimensional space.} yet remains tractable, providing an effective starting point for exploring non-parametric setups.
\end{itemize}
\end{remark}

Let 
\begin{align}\label{eq:Defsigma}
\sigma_\mu(\alpha):=\int_{[0,\alpha]} \frac1{1-r}\, \mu(\dif r), \quad \alpha\in[0,1).
\end{align}
The result below regards the properties of $\rho_\mu$, the proof of which can be found in, for example, \citep{Acerbi2002Spectral}, \citep[][Section 6.3.4]{Shapiro2021book}.
\begin{lemma}\label{lem:Spectral}
For $\mu\in\cP([0,1])$ satisfying $\mu(1)=0$, we let $\sigma_\mu$ be defined in \eqref{eq:Defsigma}. Then, $\sigma_\mu$ is nonnegative, nondecreasing, right continuous, and  $\int_0^1\sigma_\mu(\alpha)\dif \alpha =1$. Moreover, $\sigma_\mu$ characterizes $\mu$ and $\rho_\mu$ in the following way
\begin{align}
\mu([0,x]) &= (1-x)\sigma_{\mu}(x) + \int_0^{x} \sigma_\mu(\alpha)\dif\alpha,\qquad r\in[0,1)\label{eq:Idenmusigma}\\
\rho_\mu(Z) &= \int_0^1\sigma_\mu(\alpha)F^{-1}_Z(\alpha)\dif\alpha.
\label{eq:DefSRMOG}
\end{align}
\end{lemma}

It follows from Lemma~\ref{lem:Spectral} that $\sigma_\mu$ is a density function on $[0,1]$. We thereby introduce a corresponding probability measure on $[0,1]$, denoted by  $\tilde\mu$, with density $\sigma_\mu$. In particular, we have
\begin{align}\label{eq:Defmutilde}
\tilde\mu([0,x]) = \int_0^x \sigma_{\mu}(\alpha) \dif\alpha = \int_0^x \int_{[0,\alpha]}\frac{1}{1-r}\mu(\dif r), \quad x\in[0,1].
\end{align} 
By \eqref{eq:Idenmusigma}, there is a one-to-one correspondence between $\mu$ and $\tilde\mu$. For future reference, we introduce the following subsets of $\cP([0,1])$,
\begin{gather}
\cP_\lambda([0,1]):=\left\{\mu\in\cP([0,1]): \sigma_{\mu}(1)< \lambda\right\}, \quad \lambda\in(1,\infty],\\
\cP_{\lambda,\lambda'}([0,1]):=\left\{\mu\in\cP_\lambda([0,1]): \sigma_{\mu} \text{ is $\lambda'$-Lipschitz continuous} \right\}, \quad \lambda'\in(0,\infty).
\end{gather}
In particular, $\cP_\infty([0,1])$ consists of all $\mu$ with bounded $\sigma_\mu$.

\paragraph{Risk aversion modeling.} We consider a finite state space $\bX=\set{x_0,x_1,x_2}$.

\begin{remark}
While a continuum of gains and losses would more accurately reflect reality, we show in subsequent sections that this discrete framework suffices for estimating the underlying risk aversion.  Nevertheless, we acknowledge that certain research questions in continuous settings cannot be adequately addressed through discrete approximations. We will explore these extensions in future work. 
\end{remark}

We model the client's risk aversion by $(C_0,\mu_0)$, where $C_0:\bX\to[0,1]$ is a cost function of the state and $\mu_0\in\cP([0,1])$ characterizes the risk measure as listed in~\eqref{eq:Defrhomu}. In view of~\eqref{eq:DefSRMOG}, for an $\bX$-valued random variable, the client evaluates according to
\begin{align}\label{eq:IdenClient}
\rho_{\mu_0}\big(C_0(X)\big) = \int_0^1 \sigma_{\mu_0}(\alpha) F^{-1}_{C_0(X)}(\alpha)\dif\alpha.
\end{align}

In our monetary context, we assume the preference order is strict and known in advance. Following the translation-invariance and positive-homogeneity properties of SRMs \citep{Acerbi2002Spectral}, we impose a harmless normalization condition, resulting in:
\begin{align}\label{eq:Setupc0}
C_0(x_0) = 0 < C_0(x_1)=:c_0 < C_0(x_2)=1.
\end{align}
Under this setting, the only unknown component in $C_0$ is $c_0$. In what follows, given any $c\in[0,1]$, we automatically associate it with $C(x_1)=c$. We will write $(c,\mu)\in(0,1)\times\cP([0,1])$ instead of $(C,\mu)\in[0,1]^\bX\times\cP([0,1])$.

\paragraph{Interactive questioning.} 
We consider binary-choice problems with action space $\bA=\set{\text{A},\text{B}}$, where each element represents a choice label. We discuss in Section~\ref{sec:Extensions} how adding multiple choices may (or may not) improve the learning process.
Such questions can be represented by $G=(G^a)_{a\in\bA}\in\cP(\bX)^{\bA}$, where $G^a\in\cP(\bX)$ is a simplex on $\bX$. $G$ can be viewed as a matrix of size ${|\bX|\times|\bA|}$ where each column, corresponding to a $G^a$ for some $a\in\bA$, specifies the probability distribution over outcomes in $\bX$ under choice $a$. The majority of our analysis utilizes the following form of $G$,
\begin{align}\label{eq:DefGpq}
G_{p,q} = \bordermatrix{~ & \text{A} & \text{B} \cr x_0 & 1-p & 1-q \cr x_1 & p & 0 \cr x_2 & 0 & q\cr},\quad p,q\in[0,1].
\end{align}
In a robo-advisory context, a typical binary-choice question takes the following form:
\begin{center}
\fbox{\begin{minipage}{35em}
Let $\bX$ represent monetary losses of $x_0=\$0$, $x_1=\$500$, and $x_2=\$10,000$, respectively. Which of the following option would you choose? 
\begin{itemize}
\item[A.] $50\%$ chance of no loss ($x_0$) and $50\%$ chance of losing $\$500$ dollars ($x_1$)
\item[B.] $98\%$ chance of no loss ($x_0$) and $2\%$ chance of losing $\$10,000$ dollars ($x_2$)
\end{itemize}
This type of question corresponds to $G_{p,q}$ as in~\eqref{eq:DefGpq} with $p=0.5$ and $q=0.02$.
\end{minipage}}
\end{center}

Multiple rounds of questions will be conducted. At round $n$, we design a question $G_n=(G_n^a)_{a\in\bA}\in\cP(\bX)^\bA$. Provided $G_n$, the client in turn demonstrates his or her preferred options, that is the client provides the learner with
\begin{align}\label{eq:Defastar}
a_n^* \in \argmin_{a\in\bA} \rho_{\mu_0}\big(C_0(X_{G^a_n})\big), \quad \text{where } X_{G^a_n}\sim G^a_{n}.
\end{align}
Notation-wise, if $G=G_{p,q}$ for some $p,q\in[0,1]$, we set $X_{p,q}^{a} = X_{G^a_{p,q}}$, and accordingly, 
\begin{align}\label{eq:Exprrho}
\rho_{\mu}\big(C(X^{\text{A}}_{p,q})\big) = c\int^1_{1-p} \sigma_\mu(\alpha)\dif\alpha \quad\text{and}\quad \rho_{\mu}\big(C(X^{\text{B}}_{p,q})\big) = \int^1_{1-q} \sigma_\mu(\alpha)\dif\alpha.
\end{align}
Upon observing $a_n^*$, together with the observation from previous rounds, we design the next question $G_{n+1}$. This process repeats itself for a given number of rounds, and our goal is to determine $(C_0,\mu_0)$ based on previous interaction $\set{(G_n,a^*_n)}_{n=1}^N$, where $N\in\bN$ is the number of total rounds.

\begin{remark}
Condition~\eqref{eq:Defastar} implicitly assumes the client has precise preferences. However, this may not hold in practice, as clients may exhibit uncertainty when presented with options that lack clear distinction. While the investigation of uncertain/stochastic behaviors is deferred to avoid excessive technicality, some of the subsequent results may still be relevant, see, for example, Proposition~\ref{prop:g}.
\end{remark}

\section{Auxiliary quantities}\label{sec:AuxQuantities}

In this section, to aid our subsequent analysis, we introduce several auxiliary quantities. In Section~\ref{subsec:Distance}, as our problem involves the non-parametric estimation of $\mu_0\in\cP([0,1])$, we discuss a distance suitable for this context. In Section~\ref{subsec:DistPow}, we propose the core concepts to quantify which questions are more appropriate for quickly determining the client's risk aversion. Finally, in Section~\ref{subsec:IndiffLine}, we describe another key concept, namely the indifference curve for questions of the form~\eqref{eq:DefGpq}.

\subsection{Distance for \texorpdfstring{$\mu$}{mu}}\label{subsec:Distance}
Our goal involves model-agnostic estimation of $\mu_0$. The Wasserstein-$1$ distance, denoted by $\cW$, defined below in terms of the Kantorovich-Rubinstein duality \citep{kantorovich1958space}, appears to be suitable, 
\begin{align}\label{eq:DefW}
\cW(\mu,\mu') &:= \sup_{\|f\|_{\lip}\le 1} \left|\int_{0}^1 f(\alpha) \big(\mu(\dif\alpha)-\mu'(\dif\alpha)\big)\right|,
\end{align}
where $\|f\|_{\lip}:= \sup_{x,x'\in[0,1]}\frac{|f(x)-f(x')|}{|x-x'|}$ is the Lipschitz constant of $f$. In particular, it is well known that $\cP([0,1])$ with $\cW$ forms a compact (and thus totally bounded) metric space, enabling finite discretization with controllable approximation error.\footnote{To the best of our knowledge, the quantitative trade-off between discretization granularity and accuracy remains an open question.} We believe such discretization, while rarely implemented explicitly, serves as a crucial bridge between model-agnostic estimation under $\cW$ and modern computational methods that rely on discrete calculations.

In view of~\eqref{eq:IdenClient} and \eqref{eq:Setupc0}, we introduce below an alternative distance,  
\begin{align}\label{eq:DefUtilde}
\wt\cU(\mu,\mu') := \sup_{f\in\cF_{\uparrow}}\left| \int_{[0,1]} f(\alpha)\big(\tilde\mu(\dif \alpha)-\tilde\mu'(\dif \alpha)\big) \right|,\quad \tilde\mu,\tilde\mu'\in\cP([0,1]),
\end{align}
where $\tilde\mu$ is introduced in~\eqref{eq:Defmutilde} and $\cF_{\uparrow}$ is defined as the set of non-decreasing and right continuous functions $f:[0,1]\to[0,1]$.
Unlike $\cW$, the metric $\wt\cU$ directly quantifies how estimation errors propagate to the downstream task of evaluating risk on behalf of the client.

Proposition~\ref{prop:Utilde} below provides an equivalent form of $\wt\cU$ and bounds its relationship to $\cW$. The proof is deferred to Section~\ref{subsec:Pf:prop:Utilde}
\begin{proposition}\label{prop:Utilde}
For any $\mu,\mu'\in\cP([0,1])$, let $\tilde\mu,\tilde\mu'$ be as introduced in \eqref{eq:Defmutilde}, we have
\begin{align}\label{eq:IdenUtilde}
\wt{\cU}(\mu,\mu') = \sup_{p\in[0,1]}|\tilde\mu([1-p,1])-\tilde\mu'([1-p,1])| = \sup_{p\in[0,1]}\left|\int_{1-p}^{1}\big(\sigma_{\mu}(\kappa)-\sigma_{\mu'}(\kappa)\big)\dif\kappa\right|.
\end{align}
If we further assume $\mu,\mu'\in\cP_\lambda([0,1])$ for some $\lambda\ge 1$,\footnote{$\lambda\ge 1$ is needed as $\sigma_\mu$ is a density function on $[0,1]$, as argued by Lemma~\ref{lem:Spectral}.} then
\begin{align}\label{eq:BBWassUtilde}
\frac1{8}\wt\cU(\mu,\mu')^2 \le \cW(\mu,\mu') \le 2\sqrt{\lambda\wt{\cU}(\mu,\mu')} + 2\wt{\cU}(\mu,\mu').
\end{align}
\end{proposition}
In the light of the discussion above, $\wt\cU$ will serve as our primary metric on $\cP_\lambda([0,1])$ for all subsequent analysis.

\subsection{Optimality gap and distinguishing power}\label{subsec:DistPow}
We first introduce \textit{the optimality gap} of an action $a\in\bA$ under binary-choice question $G\in\cP(\bX)^\bA$ and risk aversion $(c,\mu)\in(0,1)\times\cP([0,1])$ as
\begin{align}\label{eq:DefOptGap}
\Phi\big(a,G,(c,\mu)\big) := \rho_{\mu}\big(C_\ell(X_{G^a})\big)-\min_{k\in\bA}\rho_{\mu}\big(C\left(X_{G^k}\right)\big).
\end{align}
The optimality gap $\Phi$ evaluates whether $(c,\mu)$ is consistent with the observed historical behavior, where the robo-advisor would typically substitute $a$ and $G$ based on history interactions.

An efficient estimation procedure hinges on the careful design of questions (or environments in more general IRL settings). This fundamentally depends on identifying a question that induces distinct client choices for any different $(c,\mu),(c',\mu')\in(0,1)\times\cP([0,1])$.\footnote{We say $(c,\mu),(c',\mu')\in(0,1)\times\cP([0,1])$ are different if one of the following is true: (i) $|c-c'|>0$; (ii) $\wt\cU(\mu,\mu')>0$.} For quantification purposes, we define below the \textit{distinguishing power} of a binary-choice question $G\in\cP(\bX)^\bA$ toward risk aversions $(c,\mu)$ and $(c',\mu')$:
\begin{align}\label{eq:DefPsi}
\Psi\big(G,(c,\mu),(c',\mu')\big) &:= \sqrt{\bigg(  \Big(\rho_{\mu}\big(C(X_{G^{\text{A}}})\big) - \rho_{\mu}\big(C(X_{G^{\text{B}}})\big)\Big)\Big(\rho_{\mu'}\big(C'(X_{G^{\text{B}}})\big) - \rho_{\mu'}\big(C'(X_{G^{\text{A}}})\big)\Big) \bigg)_{\!+}}.
\end{align}
It is straightforward to verify that
\begin{align}\label{eq:IdenPsi}
\Psi\big(G,(c,\mu),(c',\mu')\big) = \sqrt{\Phi\big({a^*}',G,(c,\mu)\big) \Phi\big(a^*,G,(c',\mu')\big)},
\end{align}
where 
\begin{align}
a^* \in \argmin_{a\in\bA}\rho_{\mu}\big(C(X_{G^{a}})\big) \quad\text{and}\quad {a^*}' \in \argmin_{a\in\bA}\rho_{\mu'}\big(C'(X_{G^{a}})\big).
\end{align}
Clearly, $\Psi>0$ if and only if $G$ elicits distinct client choices under $(c,\mu)$ and $(c',\mu')$. Additionally, achieving a large $\Psi$ requires obvious optimality under both $(c,\mu)$ and $(c',\mu')$.

Aside from the distinguishing power $\Psi$ defined in the absolute sense, we also define a relative version by dividing the risk values involved
\begin{align}\label{eq:DefXi}
\Xi\big(G,(c,\mu),(c',\mu')\big) := \frac{2\Psi(G,(c,\mu),(c',\mu')\big)}{ \sqrt{ \Big(\rho_{\mu}\big(C(X_{G^{\text{A}}})\big) + \rho_{\mu}\big(C(X_{G^{\text{B}}})\big)\Big)\Big(\rho_{\mu'}\big(C'(X_{G^{\text{B}}})\big) + \rho_{\mu'}\big(C'(X_{G^{\text{A}}})\big)\Big) } },
\end{align}
where the factor of $2$ in the numerator serves to normalize the expression, corresponding to the ratios of differences to sums. If the denominator is $0$, we set $\Xi:=0$.

The lemma below reveals the uniform continuity of the optimality gap $\Phi$ and distinguishing power $\Psi$, where, without loss of generality, we equip $\cP(\bX)^{\bA}$ with element-wise $1$-norm. We refer to Section~\ref{subsec:Pf:Lem:Cont} for the proof.
\begin{lemma}\label{lem:Cont}
Let $\lambda\ge 1$. Then, $\Phi$ and $\Psi$ are uniformly continuous on $\bA\times\cP(\bX)^\bA\times[0,1]\times\cP_\lambda([0,1])$ and $\cP(\bX)^\bA\times \big([0,1]\times\cP_\lambda([0,1])\big)^2$, respectively. In addition, $\Xi$ is continuous on the interior of $\cP(\bX)^\bA\times \big([0,1]\times\cP_\lambda([0,1])\big)^2$.
\end{lemma}

\subsection{Indifference curve}\label{subsec:IndiffLine}
Lastly, we introduce a key concept that underpins our main results. The \textit{indifference curve}, defined below, capture the critical trade-off for risk averse clients $(c,\mu)\in[0,1]\times\cP([0,1])$ when facing $G_{p,q}$,
\begin{align}\label{eq:Defg}
g_{c,\mu}(p):=\inf\left\{q\in[0,1]: \int_{1-q}^1\sigma_\mu(\kappa)\dif\kappa \ge c \int_{1-p}^1\sigma_\mu(\kappa)\dif\kappa\right\}, \quad p\in[0,1].
\end{align}
While our definition of the indifference curve is tailored to our setting, the underlying idea is fundamentally data-centric and model-agnostic. By collecting responses to binary choices governed by $G_{p,q}$, we can approximately identify the preference region for option A and thus infer the indifference curve. The concept of indifference curve is a cornerstone of many preceding works; we refer to~\cite{Cavagnaro2013Optimal} and the references therein. For the properties of indifference curves in our context, see Section~\ref{subsec:Pf:IndCurve}.

\section{Identifiability}\label{sec:Iden}

The interactive questioning scheme provides us with the means to identify the risk aversion of the agent. Identifiability of the agent's risk aversion can be achieved by establishing the existence of a distinguishing question for any two distinct risk aversions. Specifically, a distinguishing question leads to different client choices corresponding to the respective risk aversions. To rigorously establish the said existence, we make the following technical assumptions. 
\begin{assumption}\label{assump:Basic}
$(c_0,\mu_0)\in(0,1)\times\cP_\infty([0,1])$.
\end{assumption}

The boundedness imposed on $\sigma_{\mu_0}(1)$ plays a crucial technical role in establishing the subsequent results on the existence of a distinguishing environment (see the proof of Theorem~\ref{thm:Separation}). We recognize, however, the significance of exploring identifiability without this condition. The related results will be pursued in the future work.

Our first result, Theorem~\ref{thm:Separation}, regards the existence of a distinguishing question in the one-period case. For illustration, we provide an example of such environment in Figure~\ref{fig:DistinguishingGame} where the blue region, related to the indifference curve defined in~\eqref{eq:Defg}, is where questions make the optimal actions under $\rho_\mu$ and $\rho_{\mu'}$ distinct. In this example, we assume the two risk aversions share the same cost $c$ but have different Dirac $\mu$'s. The construction of a distinguishing question for proving Theorem~\ref{thm:Separation} is less straightforward, and we refer to Section~\ref{subsec:ProofSeparation} for the details. 
\begin{theorem}\label{thm:Separation}
For any $(c',\mu')\in(0,1)\times\cP_\infty([0,1])$ that is different from $(c_0,\mu_0)$, there exists $(p,q)\in[0,1]^2$ such that 
\begin{align*}
\argmin_{a\in\bA} \rho_{\mu_{0}}\left(C_{0}(X^{a}_{p,q})\right) \cap \argmin_{a\in\bA} \rho_{\mu'}\left(C'(X^{a}_{G_{p,q}})\right) = \emptyset.
\end{align*}
\end{theorem}

\begin{remark}\label{rmk:GenSep}
By modifying the proof of Theorem~\ref{thm:Separation}, it can be shown that, in a general case where $|\bX|\ge 3$, the existence of distinguishing $G$ still holds as long as $|C_0(\bX)|\ge 3$ and $\mu\in\cP_\infty([0,1])$. The details are available in Section~\ref{subsec:Pf:rmk:GenSep}.
\end{remark}

\begin{remark}
Following Remark~\ref{rmk:GenSep}, we note that to identify $\mu_0$, it is necessary to have $|C_0(\bX)|\ge 3$. Otherwise, the client would always select the option minimizing $\bP\big(C(X_{G^a})=1\big)$, regardless of $\mu_0$ or $\mu'$, making it impossible to construct a distinguishing $G$.
\end{remark}

\input{tikz-separating-game}

Beyond the pairwise identifiability result presented in Theorem~\ref{thm:Separation}, we further prove that the risk aversion parameter can be estimated with arbitrary precision using only finitely many interaction rounds.  To establish this, we strengthen Assumption~\ref{assump:Basic} by additionally imposing an upper bound on $\sigma_{\mu_0}(1)$.
\begin{assumption}\label{assump:sigmaBound}
There is $\lambda\ge 1$ such that $(c_0,\mu_0)\in(0,1)\times\cP_\lambda([0,1])$.
\end{assumption}

We endow $[0,1]\times\cP([0,1])$ with the sum metric
\begin{align}\label{eq:Defd}
d\big((c,\mu),(c',\mu')\big) := |c-c'| + \wt\cU(\mu,\mu').
\end{align}
The next result shows that for any desired accuracy level $\varepsilon>0$, one can construct a set of $(p,q)$ pairs, which depends only on $\lambda$ but not on the specific $(c_0,\mu_0)$, such that exhaustively querying this grid yields an estimate of $(c_0,\mu_0)$ within $\varepsilon$ accuracy. This follows from Theorem~\ref{thm:Separation}, along with the compactness of $[0,1]\times\cP_\lambda([0,1])$ and the uniform continuity of $\Psi$ under $d$. The proof is deferred to Section~\ref{subsec:Pf:thm:SuffFinite}.
\begin{theorem}\label{thm:SuffFinite}
Suppose Assumption~\ref{assump:sigmaBound} holds. For any $\lambda\ge 1$ and $\varepsilon>0$, we define 
\begin{align*}
B_{\lambda,\varepsilon}:= \left\{ (c,\mu)\in[0,1]\times\cP_\lambda(\bX) : d\big((c,\mu),(c_0,\mu_0)\big) < \varepsilon \right\}.
\end{align*}  
Then, there exist $N\in\bN$ and $p_1,q_1,\dots,p_N,q_N\in[0,1]$, independent of $(c_0,\mu_0)$, such that
\begin{align*}
\forall (c,\mu)\in B^{\mathsf{c}}_{\lambda,\varepsilon}, \quad \exists \, n\in\{1,\dots,N\} \text{ such that } a^*_n\notin \argmin_{a\in\bA} \rho_{\mu}\big(C(X^{a}_{p_n,q_n})\big).
\end{align*}
\end{theorem}

We will discuss in Remark~\ref{rmk:Nrate} the growth rate of the number of rounds $N$ in Theorem~\ref{thm:SuffFinite}. We end this section by establishing, as a corollary of Theorem~\ref{thm:SuffFinite}, that designing questions uniformly at random leads to an asymptotically consistent estimation. The proof is deferred to Section~\ref{subsec:Pf:cor:RandomDesignConstitent}. 
\begin{corollary}\label{cor:RandomDesignConstitent}
Suppose Assumption~\ref{assump:sigmaBound} holds. Let $(\Gamma_n)_{n\in\bN}$ be an IID sequence drawn from the uniform distribution on $\cP(\bX)^{\bA}$. Let $B_{\lambda,\varepsilon}$ as defined in Theorem~\ref{thm:SuffFinite}, and $(\alpha^*_n)_{n\in\bN}$ be a sequence of $\bA$-valued random variable satisfying\footnote{The existence of such random variable can be established by using the standard machinery of measurable selection \cite[cf.][Section 18]{Aliprantis2006book}. The same applies to $\cE_N$ defined below.}
\begin{align}\label{eq:Defalpha}
\alpha^*_n\in\argmin_{a\in\bA} \rho_{\mu_0}\big(C_0(X^{a}_{\Gamma_n})\big).
\end{align}
Consider
\begin{align}
\cE_N:=\sup\left\{\varepsilon\ge 0: \exists (c,\mu)\in B_{\lambda,\varepsilon} \text{ such that }\alpha^*_n\in\argmin_{a\in\bA}\rho_\mu\big(C(X^a_{\Gamma_n})\big),\, n=1,\dots,N\right\}.
\end{align}
Then, with probability $1$, $\cE_N$ converges to $0$ as $N\to\infty$.
\end{corollary}
\begin{remark}
While sampling randomly from $\cP(\bX)^\bA$ yields an asymptotically consistent estimator, its convergence speed is empirically slow. We illustrate this observation in our numerical experiments presented in Section~\ref{ssec:Experiments1}. Roughly speaking, one reason for this is that uniformly random designed questions often end up lacking distinguishing power for the remaining candidate risk aversions. 
\end{remark}

\section{A lower bound of distinguishing power}\label{sec:LowerBound}

Although Section~\ref{sec:Iden} establishes fundamental identifiability results when multiple questions are presented, it does not quantify how effectively these questions can discriminate between different risk aversions. To address this, we now derive a lower bound on the optimal distinguishing power given two different risk aversions. The proof of the theorem below is constructive and involves only $G_{p,q}$. We refer to Section~\ref{subsec:Pf:thm:QuantDisEnv} for details.
\begin{theorem}\label{thm:QuantDisEnv}
Let $(c,\mu),(c',\mu')\in[0,1]\times\cP_\lambda([0,1])$. The following statements are true:
\begin{itemize}
\item[(a)] Suppose $c>c'$, where $\varepsilon:=c-c'$, and let $\eta=\left\lceil \frac{\ln\left(\lambda\frac{(1+\varepsilon)(1-c')}{(1-\varepsilon)(1-c)^2}\right)}{\ln c - \ln c'} \right\rceil + 1$. Then, there exist $p,q\in(0,1)$ such that
\begin{align}
\Phi\big(G_{p,q}, (c,\mu), (c',\mu')\big) > (cc')^{\frac{\eta}{2}}\frac{c-c'}{\eta}
\end{align}
and
\begin{align}
\Xi\big(G_{p,q}, (c,\mu), (c',\mu')\big) > \left(\frac{c'}{c}\right)^{\frac{\eta}{2}}\frac{c-c'}{2\eta}.
\end{align}

\item[(b)] Suppose $c=c'\in(0,1)$ and consider $\tilde\mu\neq\tilde\mu'$, as defined in \eqref{eq:Defmutilde}, satisfying $\wt{\cU}(\tilde\mu,\tilde\mu') =: \varepsilon\in(0,1)$. Then, there exist $p,q\in(0,1)$ such that:
\begin{itemize}
\item[(i)] For $\varepsilon\ge 2(1-c)$,
\begin{align}
\Psi\big(G_{p,q},(c,\mu),(c',\mu')\big) >\frac{c\varepsilon^2 \ln c}{5 \ln c + 5\left(\ln\varepsilon-\ln\left(1+\frac\varepsilon4\right)\right)}
\end{align}
and 
\begin{align}
\Xi\big(G_{p,q},(c,\mu),(c',\mu')\big) >\frac{2c\varepsilon \ln c}{5 \ln c + 5\left(\ln\varepsilon-\ln\left(1+\frac\varepsilon4\right)\right)};
\end{align}
\item[(ii)] For $\varepsilon< 2(1-c)$,
\begin{align}\label{eq:LBPhiSameCDiffmu}
\Psi\big(G_{p,q},(c,\mu),(c',\mu')\big) > \frac{1}{2(1-c)}\frac{\varepsilon^3 c^{1+\frac1\varepsilon(1-c+\varepsilon)(\ln\lambda-\ln\varepsilon)}}{4\varepsilon\left(1+\frac{\ln\varepsilon-\ln2(1-c)}{\ln c}\right) + 8(1-c+\varepsilon)\ln\frac{\lambda}{\varepsilon}}
\end{align}
and
\begin{align}\label{eq:LBXiSameCDiffmu}
\Xi\big(G_{p,q},(c,\mu),(c',\mu')\big) > \frac{1}{2(1-c)}\frac{\varepsilon^2 }{4\varepsilon\left(1+\frac{\ln\varepsilon-\ln2(1-c)}{\ln c}\right) + 8(1-c+\varepsilon)\ln\frac{\lambda}{\varepsilon}}.
\end{align}
\end{itemize}
\end{itemize}
\end{theorem}

Theorem~\ref{thm:QuantDisEnv} establishes lower bounds for the distinguishing power. For notational simplicity, let $\Phi^*$ and $\Xi^*$ denote the lower bounds as listed in Theorem~\ref{thm:QuantDisEnv} for the absolute and relative distinguishing power, respectively. A straightforward analysis reveals that $\Phi^*$ exhibits exponential decay, whereas $\Xi^*$ decays at a polynomial rate, as can be verified through additional calculations below.
\begin{corollary}\label{cor:LBXi}
Consider the two cases of Theorem~\ref{thm:QuantDisEnv}:
\begin{itemize}
\item[(a)] When $\varepsilon=c-c'$ is sufficiently small, we have $\eta \sim \frac{1}{\varepsilon}\ln\left(\frac{\lambda}{1-c}\right)$, and thus
\begin{align}
\Xi^* \sim \left(1-\frac{\varepsilon}{c}\right)^{\frac{\eta}{2}}\frac{\varepsilon}{2\eta} \sim \left(\frac{1-c}{\lambda}\right)^{\frac1c} \frac{\varepsilon^2}{\ln\left(\frac{\lambda}{1-c}\right)}; 
\end{align}
\item[(b)] When $\varepsilon=\wt\cU(\mu,\mu')$ is sufficiently small, we have
\begin{align}
\Xi^* \sim \frac{\varepsilon^2}{(1-c)\ln\frac{\lambda}{\varepsilon}}.
\end{align} 
\end{itemize}
\end{corollary}

\begin{remark}\label{rmk:Nrate}
By incorporating the lower bound on $\Phi$ from Theorem~\ref{thm:QuantDisEnv} into the proof of Theorem~\ref{thm:SuffFinite}, we observe that $N$ in Theorem~\ref{thm:SuffFinite} grows at a rate of $e^{\frac{K}{\varepsilon}}$ for some $K>0$, modulo polynomial factors. However, this rate does not necessarily reflect the convergence if adaptive question design is used. The corresponding convergence analysis is addressed in the following section under stronger regularity assumptions.
\end{remark}

\section{Convergence rate with design}\label{sec:ConvRate}

In this section, we study the convergence rate of interactive questioning under stronger regularity assumption than Assumption~\ref{assump:sigmaBound}, where we further assume that $\sigma_{\mu_0}$ is Lipschitz continuous.
\begin{assumption}\label{assump:sigmaLip}
$(c_0,\mu_0)\in(0,1)\times\cP_{\lambda,\lambda'}([0,1])$.
\end{assumption}
\begin{remark}
As will become apparent subsequently, the proofs of the results in this section, particularly those in Section~\ref{subsec:Pf:prop:g}, remain valid, with moderate modifications, even under the weaker assumption of local Lipschitz continuity near $\alpha=1$. The stronger global Lipschitz condition is imposed here solely for clarity of exposition. We further conjecture that the convergence rate analysis may be extended to H\"older-type continuity, though such extensions are left for future work.
\end{remark}

Theorem~\ref{thm:ConvRate} establishes an upper bound on the convergence rate of iterative questioning. Its proof is given immediately below. The key steps involve estimating the indifference curve $g_{c_0,\mu_0}$ and then analyzing the inverse problem of recovering $(c_0,\mu_0)$ from this estimate. A perturbation analysis for this inverse problem is presented in Proposition~\ref{prop:g}, after Theorem~\ref{thm:ConvRate}.
\begin{theorem}\label{thm:ConvRate}
Suppose Assumption~\ref{assump:sigmaLip} holds. Then, there exists a designing strategy, independent of the specific choice of $(c_0,\mu_0)$, such that after $J\lceil\log_2 J\rceil$ rounds of interactions, any $(\hat c,\hat\mu)\in (0,1)\times\cP_{\lambda,\lambda'}([0,1])$ producing the same choices as in the interactions must satisfy 
\begin{align}
|c_0-\hat c| \sim O(J^{-\frac12}) \quad\text{and}\quad \wt\cU(\mu_0,\hat\mu)\sim O(J^{-\frac12})
\end{align}
\end{theorem}

\begin{proof}
Consider the uniform grid $p_j=\frac{j}{J},\,j=0,1,\dots,J$. Note that $g_{c_0,\mu_0}(0)=0$ by Lemma~\ref{lem:g} (c). For each $p_j$, it takes $\lceil\log_2 J\rceil$ bisections to estimate $g_{c_0,\mu_0}(p_j)$ to an accuracy of $J^{-1}$. In total, this requires $J\lceil\log_2 J\rceil$ rounds of interactions. Since the choices produced by $(\hat c,\hat\mu)$ are assumed to match the observed interactions, the corresponding indifference curve $g_{\hat c,\hat\mu}$ must fall within the region constrained by these interactions, which leads to
\begin{align}
\left|g_{c_0,\mu_0}(p_j) - g_{\hat c,\hat\mu}(p_j)\right| \le \frac1J, \quad j=1,\dots,J.
\end{align}
Additionally, for any $p\in[0,1]$, we let $\bar p$ be the rounding to the closest grid point $p_j$. Note that $g_{\hat c,\hat\mu}$ is $1$-Lipschitz continuous due to Lemma~\ref{lem:g} (f). By combining the above, we have
\begin{align}
\left\| g_{c_0,\mu_0} - g_{\hat c,\hat\mu} \right\|_\infty &\le \sup_{p\in[0,1]}\left\{ \left| g_{c_0,\mu_0}(p) - g_{c_0,\mu_0}(\bar p) \right| + \left| g_{c_0,\mu_0}(\bar p) - g_{\hat c,\hat\mu}(\bar p) \right| + \left| g_{\hat c,\hat \mu}(\bar p) - g_{\hat c,\hat\mu}(p) \right| \right\} \\
&= \frac1J + \sup_{p\in[0,1]}\left| g_{c_0,\mu_0}(\bar p) - g_{\hat c,\hat\mu}(\bar p) \right| = \frac1J + \max_{j=1,\dots,N} \left| g_{c_0,\mu_0}(p_j) - g_{\hat c,\hat\mu}(p_j) \right| = \frac{2}{J}.
\end{align} 
Invoking Proposition~\ref{prop:g} completes the proof.
\end{proof}

The proof of Proposition~\ref{prop:g} is deferred to Section~\ref{subsec:Pf:prop:g}.
\begin{proposition}\label{prop:g}
Consider $(c,\mu), (c',\mu') \in (0,1)\times\cP_{\lambda,\lambda'}([0,1])$ with $c\ge c'$. Let us denote $\varepsilon:=\left\| g_{c,\mu} - g_{c',\mu'} \right\|_\infty$. Then,
\begin{align}\label{eq:UBAbsDiffc}
\left| c-c' \right| \le \left(1 + \frac{\lambda'}{\lambda}\right)\sqrt{\varepsilon}
\end{align}
and
\begin{align}\label{eq:UBAbsDiffmu}
\wt\cU(\mu,\mu') \le \fR_{c,\mu}(\varepsilon):=\frac52\lambda'c^{\lfloor \frac{\ln\varepsilon}{2\ln c} \rfloor} \!+\! \left(\lambda+\left(\lambda \!+\!\lambda'c^{\lfloor\frac{\ln\varepsilon}{2\ln c}\rfloor}\right)^2\! \right)\!\! \left( \frac{\sqrt{\varepsilon}}{(1-c)c} \!+\! \frac{\lambda'\varepsilon \ln \varepsilon}{2\sigma_\mu\big(1-g_{c,\mu}(1)\big)(1-c) \ln c} \right).
\end{align}
\end{proposition}
\begin{remark}
Note that $\fR_{c,\mu}(\varepsilon) \sim O(\sqrt{\varepsilon})$.
\end{remark}
\begin{remark}
In this remark, we derive an upper bound for $\sigma_{\mu}(1-g_{c,\mu}(1))^{-1}$. Note that by Lemma~\ref{lem:g} (a), we have $\int_{1-g_{c,\mu}(1)}^1\sigma_\mu=c\int_0^1\sigma_\mu=c$. With the monotonicity of $\sigma_\mu$ (i.e., Lemma~\ref{lem:Spectral}) and the upper bound on $\sigma_{\mu_0}(1)$ (i.e., Assumption~\ref{assump:sigmaBound}), we obtain
$$
1-c = 1-\int^1_{1-g_{c,\mu}(1)}\sigma_\mu(\kappa)\dif\kappa =  \int_0^{1-g_{c,\mu}(1)} \sigma_\mu(\kappa)\dif\kappa \le \sigma_\mu(1-g_{c,\mu}(1)) (1-c\lambda^{-1}).
$$
In view of~\eqref{eq:UBAbsDiffdg}, we have
\begin{align}
\sigma_{\mu}(1-g_{c,\mu}(1))^{-1} \le \frac{\lambda-c}{\lambda(1-c)} = 1+ \frac{c}{1-c}\left(1-\frac1\lambda\right).
\end{align}
Whether $\sigma_{\mu}(1-g_{c,\mu}(1))^{-1}$ is closed to the theoretical upper bound in practice requires empirical investigation.
\end{remark}
\begin{remark}
Following the proof of~\eqref{eq:UBAbsDiffc} in Section~\ref{subsec:Pf:prop:g}, it takes $O(-\ln\varepsilon)$ rounds to estimate $c_0$ with an accuracy of $\varepsilon$, without needing the knowledge of $\mu_0$. Leveraging this result, if we are dealing with $\bX=[0,1]$ and an $L$-Lipschitz $C_0:\bX\to[0,1]$, then to obtain an estimate of $C_0$ with accuracy of $\varepsilon$ under $\|\cdot\|_\infty$ by estimating $C_0$ on a grid of $x\in\bX$ with bisections, it requires $O(-L\varepsilon^{-1}\ln\varepsilon)$ rounds.
\end{remark}

\section{Simulated experiments}\label{sec:Experiments}

We propose designing questions by maximizing the distinguishing power as defined in \eqref{eq:DefPsi}. For validation, we conduct two simulated experiments under different settings. In Section~\ref{ssec:Experiments1}, we evaluate the performance of our question selection mechanisms against uniformly random designed questions for a finite set of candidate client's risk aversion. In Section~\ref{ssec:Experiments2}, we employ a more realistic setting where the client's risk aversion is synthesized from a rich parametric space, while estimation is performed on a coarser subspace. As shown in Proposition~\ref{prop:PlambdaJ}, the parametric space then closely resembles $(0,1)\times\cP_\lambda([0,1])$. Executable code for both experimental setup is available at \url{https://github.com/acoache/irl-for-risk-aversion}.

\subsection{Experiment I}\label{ssec:Experiments1}

In this section, we propose a method that designs the question for a next round based on outcomes of previous interactions for a finite set of candidate risk aversions, denoted by $$\set{(c_0,\mu_0),(c_1,\mu_1),\dots,(c_L,\mu_L)}.$$ This method hinges on a measure to convert the questions and actions of the client from previous rounds into a probability on $\set{0,1,\dots,L}$, reflecting the learner's confidence on the candidate risk aversions. Let $(G_n)_{n=1}^N$ be the questions and $(a^*_n)_{n=1}^N$ be the optimal action generated according to~\eqref{eq:Defastar} up until a certain round $N$. We introduce below the probability on the candidate risk aversions:
\begin{align}\label{eq:DefQ}
\bQ_N(\set{\ell}) \propto \exp\left(-k\sum_{n=1}^N\Phi(a^*_n;G_n,(c_\ell,\mu_\ell))\right), \quad \ell = 0,1,\ldots,L,
\end{align}
where $k>0$ is a learning parameter.

Utilizing the measure $\bQ_N$ possesses some advantages for discrete distributions. On one hand, it integrates the intuitive learning procedure that eliminates candidate risk aversions individually (occasionally achieving collateral elimination), aligning closely with Theorem~\ref{thm:Separation}. On the other hand, it facilitates the development of a scheme that enables the bulk elimination of candidate risk aversions. Finally, although not explicitly addressed in this paper, we believe that an approach like the one we propose here could potentially provide some degree of robustness against errors arising from the agent selecting a suboptimal action.

One approach to determine the next question is to randomly sample the next question $G_{N+1}$ from the uniform distribution on $\cP(\bX)^{\bA}$. In that case, using Corollary~\ref{cor:RandomDesignConstitent}, we have with probability $1$ that $\lim_{N\to\infty}\bQ_N(\{0\})=1$ and $\lim_{N\to\infty}\bQ_N(\{\ell\})=0$ for all $\ell \neq 0$. This naive approach, however, does not use history interactions into consideration and hence converges quite slowly. A more natural choice to design $G_{N+1}$ is to distinguish the two risk aversions in which the learner has the most confidence, namely
\begin{align}\label{eq:DesignGOnetoOne}
G_{N+1} \in \argmax_{G\in\cP(\bX)^\bA} \Psi\Big(G,(c_{i^*_N},\mu_{i^*_N}),(c_{j^*_N},\mu_{j^*_N})\Big),
\end{align}
where $(i^*_N,j^*_N)$ is the pair of entries with the largest and second largest probabilities assigned by $\bQ_N$.\footnote{In the case of a tie for the largest probability, we arbitrarily select two of the tied risk aversions. In the case of a tie for the second largest probability, we take the largest value and arbitrarily select one of the second largest probabilities.}
In view of Theorem~\ref{thm:Separation}, for any $i\neq j$, there exists a distinguishing $G$, and thus for such $G$, we have $\Psi(G,(c_i,\mu_i),(c_j,\mu_j))>0$.
We note that with a sufficiently large learning parameter $k$ in~\eqref{eq:DefQ}, the design based on~\eqref{eq:DesignGOnetoOne} resembles the elimination procedure that we randomly pick two risk aversions and then find a separating question in the line of Theorem~\ref{thm:Separation}. 

However, when ${\bQ_N}$ is evenly spread over $\set{0,1,\dots,L}$, the client's true risk aversion $\{0\}$ may not belong to $\left\{i^*_N,j^*_N\right\}$. Consequently, optimizing~\eqref{eq:DesignGOnetoOne} may not yield a question with a strong distinguishing capability. Bearing this in mind, we put forward an alternative criterion for designing $G_{N+1}$,
\begin{align}\label{eq:DesignGBatch}
G_{N+1} \in \argmax_{G\in\cP(\bX)^\bA} \bE\Big[ \Psi(G,(c_\eta,\mu_\eta),(c_\zeta,\mu_\zeta)) \Big],
\end{align}
where $\eta\sim\bQ_{N}$ and $\zeta\sim\bQ_{N}|\zeta\neq\eta$. One can interpret the design rule in~\eqref{eq:DesignGOnetoOne} has a greedy approach, while~\eqref{eq:DesignGBatch} allows the learner to explore other risk aversions. This method may also help eliminate a batch of candidate risk aversions. It remains unknown to us whether there exists an question that fully distinguishes one set of risk aversions from its complement.

In this set of experiments, we validate the theoretical result derived in Corollary~\ref{cor:RandomDesignConstitent} and our proposed question design approaches.
To this end, we investigate the convergence behaviors of the learning algorithms when selecting $G\in\cP(\bX)^\bA$ (i) fully at random; (ii) according to~\eqref{eq:DesignGOnetoOne}, i.e. by choosing a question that maximizes the distinguishing power between the largest probabilities assigned by $\bQ_{N}$; and (iii) according to~\eqref{eq:DesignGBatch}, i.e. by choosing a question that maximizes the expected distinguishing power under $\bQ_{N}$.

We next describe the setup to benchmark all different environment design strategies. This setting is performed for 25 runs, which allows us to describe the convergence behaviors of the different approaches. As mentioned in Section~\ref{sec:Pre}, we consider $|\bX| = 3$, $|\bA| = 2$, and cost functions satisfying~\eqref{eq:Setupc0}. For each run, we fix a set of 500 transition probabilities $G\in\cP(\bX)^\bA$ representing the questions which the learner, or robo-advisor, can choose from at every round. For simplicity, here, we consider risk measures of the form~\eqref{eq:Defrhomu} with $\mu = \tau\; \delta_{0} + (1-\tau)\; \delta_{\kappa}$, where the values of $\tau$ and $\kappa$ are specified below. This characterizes a trade-off between risk-averse behaviors from the $\avar$ at level $\kappa$, and risk-seeking behaviors from the (risk-neutral) expectation. The finite set of risk aversions is thus composed of tuples $(C_{\ell}, \kappa_{\ell}, \tau_{\ell})$ for $\ell \in \{0,1,\ldots,L\}$.\footnote{The code notebook may be easily extended to other risk measures of the form~\eqref{eq:Defrhomu}, e.g. linear combinations of $\avar$s at different thresholds, and larger state and action spaces. In practice, we notice a decrease in convergence speed when using a larger state space, but no significant difference with more actions.}

In our first experiments, this set contains 36 tuples by taking the Cartesian product of three distinct cost functions, four different $\kappa$'s and three different $\tau$'s, more specifically
\begin{equation*}
    \kappa \in \{0.2, 0.3, 0.4, 0.5\}; \quad \tau \in \{0.25, 0.5, 0.75\}; \quad \mbox{and} \quad C(x_1) \in \{ 0.3, 0.5, 0.7\}.
\end{equation*}
The client's risk aversion is given by $\kappa_0 = 0.3$, $\tau_0 = 0.25$, and $c_0=0.5$, more precisely $C_0(\bX) = [0.0, 0.5, 1.0]$.
Figure~\ref{fig:AllNrounds-OnePeriod-Set1a} shows that all $\bQ_N(\{0\})$ converge to the Dirac measure for the true risk aversion irrespective of the question design approach. Here, we set the learning parameter of $\bQ_N$, as defined in~\eqref{eq:DefQ}, to $k=4$. Uniformly sampling environments on $\cP(\bX)^\bA$ also converges to the Dirac measure, which confirms Corollary~\ref{cor:RandomDesignConstitent}, but it takes up to 5000 rounds. In addition, we observe that both the question design approaches according to~\eqref{eq:DesignGOnetoOne} and~\eqref{eq:DesignGBatch} converge quickly to the true risk aversion, with a slight advantage for the method maximizing the expected distinguishing power under $\bQ_N$. The slowest convergence is obtained when using a uniformly sampled questions, which showcases the importance for the learner of carefully choosing questions to efficiently and quickly discover the client's risk aversion.

The learner may wish to quickly attain strong confidence on the client's risk aversion without waiting for the algorithm to fully converge. For instance, there is a limited number of questions a robo-advisor may ask to a potential client, and the learner cannot expect an individual to answer hundreds of questions. In practice, this can be achieved by tuning the learning rate $k$ of the measure $\bQ_N$. For a fixed number of rounds, increasing the learning rate leads to faster convergence to the true risk aversion, as illustrated in Figure~\ref{fig:AllLrs-OnePeriod-Set1a}. Still, the learner must carefully choose the learning rate to trade-off between convergence and optimality -- small learning rates take many rounds to converge, but too large of a learning rate may indicate a non-optimal risk aversion for suboptimal actions.

\input{fig-gibbs-oneperiod-set1}

We now take a closer look at the choice of $G_n$ during the learning algorithm for the different methods. We display the evolution of $G_n$ at every round $n$ for a specific run in Figures~\ref{fig:Games-OnePeriod-Set1a-OnetoOne} and~\ref{fig:Games-OnePeriod-Set1a-Batch}, where each point corresponds to one of the many questions the learner may choose from. It is interesting to note the exploration-exploitation pattern \citep[see e.g.,][]{sutton2018reinforcement} with the environment design approaches that are not uniform. It appears that the learner explores different questions at the beginning of the learning phase, and then focuses on a small subset of the available questions to refine its estimation of the client's risk aversion. There is some variability seen in Figures~\ref{fig:Games-OnePeriod-Set1a-OnetoOne} and~\ref{fig:Games-OnePeriod-Set1a-Batch}. Indeed, the algorithm using the design rule~\eqref{eq:DesignGOnetoOne} constantly alternates between two environments, because the second largest value of $\bQ_N$ changes at each round, as illustrated in Figure~\ref{fig:Gibbs-OnePeriod-Set1a-OnetoOne}. As well, for the approach using the design rule \eqref{eq:DesignGBatch}, the exploitation pattern reappears once the measure $\bQ_N$ attributes most of the weight on a single value. In Figures~\ref{fig:Gibbs-OnePeriod-Set1a-OnetoOne} and~\ref{fig:Gibbs-OnePeriod-Set1a-Batch}, we observe the values $\bQ_N$ for all risk aversion candidates in for a specific run, which shows that $\bQ_N(\ell)$ converges to zero for all $\ell \in \{1,\ldots,L\}$. Figure~\ref{fig:OnePeriod-Set1b} shows similar behaviors when using $k=10$. With larger learning rates, the learner correctly identifies the client's true risk aversion with 90\% confidence in approximately 30 questions.

\input{fig-oneperiod-set1}

An interesting question to ask is: what happens if the true risk aversion does not belong in the set of candidates risk aversions of the learner? In other words, what happens if the model is misspecified? To illustrate this scenario, we now fix the cost function $C_0(\bX) = [0.0, 0.5, 1.0]$ as well as $\tau_0 = 0$, but vary the $\kappa$ parameter. More precisely, we take 21 evenly spaced numbers over the interval $[0.1, 0.9]$ and set the true parameters $\kappa_{0}$ to 0.24, which does not belong to the set of risk aversion candidates. Figure~\ref{fig:Gibbs-OnePeriod-Set2} shows the evolution of $\bQ_N$ for the different question design approaches. When the client's risk aversion is not part of the set of risk candidates, the algorithm struggles to choose the truth at the beginning but eventually converges to the closest risk aversion available. On the other hand, uniformly sampling environments on $\cP(\bX)^\bA$ fails to decide on a single risk aversion, which is in line with Corollary~\ref{cor:RandomDesignConstitent}. This indicates that a learner may prefer designing questions according to \eqref{eq:DesignGBatch} for better identifying the client's risk aversion in settings where there is misspecification.

\input{fig-oneperiod-set2-msf}

\subsection{Experiment II}\label{ssec:Experiments2}
In this section, we aim to show the performance of our approach in a more challenging setting than the finite set of candidate risk aversions in Section~\ref{ssec:Experiments1}. More precisely, we want to model $\mu$ such that we allow more flexibility when characterizing the risk preferences of the client.  To do so, we first define for $\lambda\ge 1$ and $J\in\bN$
\begin{align*}
\cP_{\lambda}^J([0,1]):= \left\{\mu\in\cP_\lambda([0,1]): \sigma_\mu \text{ is constant on} \left[\frac{j-1}{J}, \frac{j}J\right),\, j=1,\dots,J \right\}.
\end{align*}
This set serves as a decent proxy for $\cP_\lambda([0,1])$, as shown in the result below.
\begin{proposition}\label{prop:PlambdaJ}
For any $\lambda\ge 1$, $J\in\bN$, and $\mu\in\cP_\lambda([0,1])$, there is a $\check\mu\in\cP_\lambda^J([0,1])$ such that $\wt\cU(\mu,\check\mu) \le \frac{\lambda}{J}$.
\end{proposition}
\begin{proof}
For $\mu\in\cP_{\lambda}([0,1])$, we construct $\check\mu\in \cP_{\lambda}^J([0,1])$ by letting
\begin{align}
\sigma_{\check\mu}\left(\frac{j-1}{J}\right) := \frac1J\int_{\frac{j-1}{J}}^{\frac{j}{J}}\sigma_{\mu}(\kappa)\dif\kappa,\quad j=1,\dots,J.
\end{align} 
It follows from \eqref{eq:IdenUtilde} that 
\begin{align}\label{eq:UBUtilde}
\wt\cU(\mu,\check\mu) \le \sum_{j=1}^J\int_{\frac{j-1}{J}}^{\frac{j}{J}} \left|\sigma_{\mu}(\kappa)-\sigma_{\check\mu}\left(\frac{j-1}{J}\right)\right|\dif\kappa.
\end{align}
By the monotonicity in Lemma~\ref{lem:Spectral}, we have
\begin{align}\label{eq:UBAbsDiffsigma}
\left|\sigma_{\mu}(\kappa)-\sigma_{\check\mu}\left(\frac{j-1}{J}\right)\right| \le \sigma_\mu\left(\frac{j}{J}\right) - \sigma_\mu\left(\frac{j-1}{J}\right),\quad \kappa\in\left[\frac{j-1}{J}, \frac{j}J\right],\quad j=1,\dots,J.
\end{align}
Substituting~\eqref{eq:UBAbsDiffsigma} into the right-hand side of~\eqref{eq:UBUtilde} readily completes the proof.
\end{proof}

While elicitation could be performed following the idea in the proof of Theorem~\ref{thm:ConvRate}, the number of interactions required may still be excessive. As a rough estimate, assuming the Lipschitz constant of $g_{c_0,\mu_0}$ is $0.5$, achieving an accuracy of $0.1$ would require $0.5/0.1^2=50$ grid points. The total number of interactions would be considerably larger when accounting for subsequent bisections. To support the practical application of the proposed framework, we need to strike for better efficiency.

We propose a particle-based algorithm with $K\in\bN$ in which each particle, indexed by $\ell=1,\ldots,K$, keeps track of a candidate risk aversion composed of a cost $c_\ell = C_\ell(x_1)$ and a measure $\mu_\ell\in \cP_{\lambda}^J([0,1])$ (and thus the spectral measure $\sigma_{\mu_\ell}$ in view of Lemma~\ref{lem:Spectral}). The algorithm consists of three core components: initialization, update, and removal of particles. For each particle $\ell\in\{0,\ldots,K\}$, we initially assign a value for the cost $c_\ell \sim \textrm{Uniform}(0,1)$ and randomly sample a measure $\mu_\ell\in \cP_{\lambda}^J([0,1])$, which in turn fixes $\sigma_{\mu_\ell}$, as a starting point. Then, for a given number of iterations, the algorithm updates the particles with a policy-gradient-inspired approach. More specifically, for each risk aversion candidate $\ell$ at a given round $N$, we compute the loss
\begin{equation}\label{eq:PenaltyParticle}
\cL(a_N^*, G_N, (C_\ell,\mu_\ell)) = \left(a^*_N\times \frac{\rho_{\mu_\ell}\big(C_\ell(X_{G_N^{\text{A}}})\big) - \rho_{\mu_\ell}\big(C_\ell(X_{G_N^{\text{B}}})\big)}{\rho_{\mu_\ell}\big(C_\ell(X_{G_N^{\text{A}}})\big) + \rho_{\mu_\ell}\big(C_\ell(X_{G_N^{\text{B}}})\big)} \right)_+.
\end{equation}
This represents a relative version of the optimality gap in \eqref{eq:DefOptGap}. Averaging over all particles and rounds played yields the loss function that we aim to optimize:
\begin{equation}\label{eq:FullLoss}
\cL_N = \frac{1}{N K} \sum_{n=1}^{N} \sum_{\ell=1}^{K} \cL\big(a_{n}^*, G_{n}, (C_\ell,\mu_\ell)\big) + \frac{0.01}{\lfloor0.1J\rfloor K} \frac{\sum_{j=\lceil0.9J\rceil}^{J} \sum_{\ell=1}^{K} \mu_\ell\left(\frac{j-1}{J}\right)^2}{\sum_{j=1}^{J} \sum_{\ell=1}^{K} \mu_\ell\left(\frac{j-1}{J}\right)},
\end{equation}
where the second term acts as a smoothness penalty for $\mu(\alpha)$ with $\alpha\in(0.9, 1)$. We modify the particles according to
\begin{equation}\label{eq:MoveCParticle}
c_\ell \leftarrow \min \Big( 1, \max \big( c_\ell - \nu_c \nabla_{c_\ell} \cL_N + \varsigma_c U_{c}, 0 \big) \Big)
\end{equation}
and
\begin{equation}\label{eq:MoveMuParticle}
\mu_\ell\left(\frac{j-1}{J}\right) \leftarrow \mu_\ell\left(\frac{j-1}{J}\right) - \nu_{\mu,j} \nabla_{\mu_\ell} \cL_N + \varsigma_{\mu,j} U_{\mu,j}, \quad j=1,\ldots,J,
\end{equation}
where $\nu_c,\nu_{\mu,j} \geq 0$ represent learning rates that control the exploitation of the gradient descent, $\varsigma_c,\varsigma_{\mu,j} \geq 0$ parameters to regulate the random exploration, and $U_c\sim\textrm{Uniform}(-0.5, 0.5)$ and $U_{\mu,j}\sim\textrm{Uniform}(0, 1)$ are IID. In the code implementation, we normalize the increments of $\mu_\ell$ to avoid potential conflicting gradient signals. Finally, if a particle $\ell\in\{0,\ldots,K\}$ is on average performing poorly such that
\begin{equation}\label{eq:RemoveParticle}
\frac{1}{N} \sum_{n=1}^{N} \cL\big(a_{n}^*, G_{n}, (C_\ell,\mu_\ell)\big) > 0,
\end{equation}
we then delete it from the system.

At each round $N$, we determine the next question by maximizing the relative distinguishing power defined in \eqref{eq:DefXi} with all particles:
\begin{equation}\label{eq:DesignParticle}
G_{N} \in \argmax_{G} \bE\left[ \Xi(G,(c,\mu),(c',\mu')) \right] \approx \argmax_{G} \frac{1}{K} \sum_{\ell=1}^{K} \Xi\Big(G, (c_{s(\ell)},\mu_{s(\ell)}), (c_{\bar{s}(\ell)},\mu_{\bar{s}(\ell)})\Big),
\end{equation}
where $G$ is of the form \eqref{eq:DefGpq} and $s,\bar{s}$ are permutations of the particle indices. This resembles the criterion~\eqref{eq:DesignGBatch} introduced in Section~\ref{ssec:Experiments1}. For the maximization step, we perform a brute-force search over all pairs $p>q$ with a step size of $1\%$. This $1\%$ increment is not merely a hyper-parameter for numerical accuracy; it also accounts for humans' limited perceptual resolution in evaluating probabilistic differences. It is worth noting that a $1\%$ change in probability was also used in the renowned Allais paradox experiment \citep{Allais1953Le}. Note that in our implementation, the brute-force search constitute less than $5\%$ of the total computation time, as the majority of the cost is spent on moving the particles.

Upon receiving the binary-choice question, The client then shows the corresponding optimal action. The learner subsequently performs two update cycles while replenishing eliminated risk aversion candidates during the first update. The approximation of the client's true risk aversion is given by taking the average over particles. We describe the full procedure in Algorithm~\ref{algo:particle}, where we alternate between updating $c$ and $\mu$ separately and perform two cycles with an elimination-replenishment procedure. These prevent getting stuck in local optima and reflect that we are dealing with the cost estimation from a distributional perspective.

\begin{algorithm}[htbp]
\caption{Particle-based algorithm for $c_0\in[0,1]$ and $\mu_0 \in \cP([0,1])$}
\label{algo:particle}
\KwIn{size of grid $J\in\bN$, $K\in\bN$ particles, $N\in\bN$ rounds, $M\in\bN$ updates;}
Initialize randomly $K$ particles with $c_\ell \in (0,1)$ and $\mu_\ell\in \cP_{\lambda}^J([0,1])$\;
\For{each round $n = 1, \ldots, N$}{
    The learner designs a question $G_n$ according to~\eqref{eq:DesignParticle}\;
    \lIf{$G_n$ does not have distinguishing power, i.e. $\Xi<0$}{\textbf{break}}
    The client selects his or her optimal action $a^*_n$ via~\eqref{eq:Defastar}\;
    \For{each update $m = 1, \ldots, M$}{
        Update risk via~\eqref{eq:MoveMuParticle} with $\nu_{\mu,j}=\frac{(J-j+1)}{J^3\sqrt{n}} \sum_{i=1}^{J} \sqrt{ \frac{1}{K} \sum_{\ell=1}^{K} \left(\mu_\ell\Big(\frac{i-1}{J}\Big) - \bar{\mu}\Big(\frac{i-1}{J}\Big)\right)^2 }$ and $\varsigma_\mu=0$\;
        \lIf{the number of particles satisfying~\eqref{eq:RemoveParticle} is low}{\textbf{break}}
    }
    Eliminate risk aversion candidates if they satisfy~\eqref{eq:RemoveParticle}\;
    Create new particles to obtain $K$ candidates\;
    Update the costs using~\eqref{eq:MoveCParticle} with $\nu_c=\max\left(0.05, 2 \sqrt{\frac{1}{K} \sum_{\ell=1}^{K} (c_\ell - \bar{c})^2} \right)$ and $\varsigma_c=\nu_c/2$\;
    \For{each update $m = 1, \ldots, 5M$}{
        Update risk via~\eqref{eq:MoveMuParticle} with $\nu_{\mu,j}=\frac{(J-j+1)}{J^3\sqrt{n}} \sum_{i=1}^{J} \sqrt{ \frac{1}{K} \sum_{\ell=1}^{K} \left(\mu_\ell\Big(\frac{i-1}{J}\Big) - \bar{\mu}\Big(\frac{i-1}{J}\Big)\right)^2 }$ and $\varsigma_\mu=0$\;
        \lIf{the number of particles satisfying~\eqref{eq:RemoveParticle} is low}{\textbf{break}}
    }
    Eliminate risk aversion candidates if they satisfy~\eqref{eq:RemoveParticle}\;
    Calculate $\bar{c} = \frac{1}{K}\sum_{\ell=1}^{K} c_\ell$ and $\bar{\sigma}(\frac{j-1}{J}) = \frac{1}{K}\sum_{\ell=1}^{K} \sigma_{\mu_\ell}(\frac{j-1}{J}), \ j=1,\ldots,J$\;
}
\KwOut{Approximation $\bar{c}$ and $\bar{\sigma}$}
\end{algorithm}

For this setup, we consider $J=100$ and use $K=1000$ particles, $M=1000$ updates and a maximum of $N=50$ rounds. We assume that the client's risk aversion is parameterized on a finer grid $\mu_0\in \cP_{\lambda}^{J_0}([0,1])$ where $J_0 = 4J$, while the learner's risk aversion belongs to the coarser space $\cP_{\lambda}^J([0,1])$ to reflect potential model misspecification. This can be easily extended to $\mu_0\in \cP_{\lambda}([0,1])$ without loss of generality. Empirically, the algorithm usually ends after approximately 20-40 rounds, as there are no distinguishing questions anymore for the learner, with a reported execution time of approximately 5-10 minutes using a Tesla T4 GPU on Google Colaboratory. During the first few rounds, note that the algorithm does not perform the full $M$ and $5M$ updates. Increasing the number of particles seems to improve the accuracy but reduces the computational efficiency.

Figure~\ref{fig:Set4} shows results for a client's risk aversion given by $c_0=0.3$ and a random piecewise constant $\sigma_{\mu_0}$. We can observe that the learner correctly identifies both the cost and the risk preferences of the client in Figure~\ref{fig:Sigma-Set4}. The errors between the learner's estimation and the true risk aversion, shown in Figure~\ref{fig:ErrorsConvergence-Set4}, decreases steadily after each round of questions, as the learner gathers more information on the client's optimal behavior. Here, we use the absolute difference to compute the error on $c$ and \eqref{eq:DefUtilde} to calculate the error on $\mu$. Finally, we also show in~\ref{fig:GamePower-Set4} the distinguishing power as a function of $p,q$ in a similar manner to Figure~\ref{fig:DistinguishingGame}. Designing the next question by maximizing the (relative) distinguishing power allows environment where different risk aversion candidates result in different optimal actions, improving the odds of quickly identifying the truth. Figure~\ref{fig:GameSelection-Set4} describes exactly which matrix $G_{p,q}$ is chosen at each round to differentiate between risk aversions and the true indifference curve. As expected, the questions in the first rounds (in darker shades of blue) are located near the indifference curve and, as the algorithm pinpoints the client's true risk aversion, questions in the later rounds (in light green and yellow) are exactly on the indifference curve.

\input{fig-oneperiod-set4-particle}

We repeat this exercise with a continuous $\mu_0$ in Figures~\ref{fig:Set4-1} and~\ref{fig:Set4-2}. In both cases, the learning algorithm correctly identifies the qualitative features of the client's risk aversion $\sigma_0$, the cost $c_0$, and the indifference curve. The approach appears to perform slightly better with a concave risk aversion than a convex one, although the difference is not substantial.\footnote{The architecture of the code notebook may be easily extended to the individual use cases, for instance by modifying hyperparameters, including alternative risk aversions or adding application-dependent modules.}

Additionally, we run a repeated experiments analysis to verify the efficiency of our algorithm for $50$ different risk-aversion candidates. This analysis has an execution time of approximately 2.5 hours on a server equipped with a 20 vCPU Intel(R) Xeon(R) Platinum 8470Q and a vGPU with 48GB of memory. Figure~\ref{fig:Set5} presents a scatterplot with the marginal distributions for the final errors in both the cost and risk preference. The color indicates the value of $c_0$, chosen uniformly in $[0,1]$, while the size gives the number of questions asked by the algorithm to achieve this accuracy. Interestingly enough, there do seem to be a relationship between the magnitude of the errors and the value of the client's true cost. We can also see that asking more questions reduces the overall error on the client's risk aversion, as one would expect. Even though we allow a maximum of 75 questions, the algorithm often ends early with less than 40 queries (see Table~\ref{tab:Set5}). The right-hand side of Table~\ref{tab:Set5} shows that larger errors are primarily due to risk aversions with very small or very large costs. In practice, we also notice that $\bar{\sigma}$ captures the qualitative features of $\sigma_0$ even when reporting larger errors, and adding random perturbations to the smoothness penalty in \eqref{eq:FullLoss} seems to improve learning. Performing a full ablation study to understand the contribution of each hyperparameter is important and left for future works.

\input{fig-oneperiod-set5-repeated}

\newpage

\section{Discussion on potential future works} \label{sec:Extensions}
In this section, we provide some discussions on potential further works.

\subsection{Multiple choices may improve learning}\label{subsec:MultipleChoice}

A natural follow-up to Theorem~\ref{thm:Separation} is to question whether there exists a scenario where the agent's risk aversion can be distinguished with only one demonstration. Proposition~\ref{prop:MultiChoiceSeparation} below demonstrates that under the setup where $C(x)=x$, this is possible, provided the state and action space is large enough. The proof is provided in Section~\ref{subsec:Proofprop:MultiChoiceSeparation}. However, it remains an open question whether there exists an environment that distinguishes risk aversion with a single demonstration in other settings. The corresponding design problem may also be of interest.
\begin{proposition}\label{prop:MultiChoiceSeparation} 
Consider $\Upsilon=\set{\mu_1,\dots,\mu_L}$. Suppose $C(x)=x$. Then, there exist a state space $\bX=\set{x_0,\dots,x_{2L}}$ and an action space $\bA$ with $|\bA|=L$, and a multiple-choice question $G\in\cP(\bX)^{\bA}$ such that 
\begin{align*}
\set{a_\ell} = \argmin_{a\in\bA} \rho_{\mu_\ell}\left(X_{G^{a}}\right), \quad \ell=1,\dots,L,
\end{align*}
where $X_{G^{a}}\sim G^{a}$.
\end{proposition}

It is important to note, however, that multiple-choice questions may overwhelm clients, potentially reducing answer accuracy -- a consideration beyond the scope of this paper. Additionally, potential convergence and stability issues warrant further investigation. We refer to \cite{Zhu2022Contextual} for discussions on related topics.

\subsection{Infinite-horizon case}\label{subsec:InfHor}
Human decision-making typically involves long-term planning and dynamic adjustments. However, considering every potential outcome is impractical, so individuals often cope by discounting uncertainties in the distant future. While this may be a simplification, we model such disregard with a discount factor $\gamma\in(0,1)$. Building on the set up of the static case in Section~\ref{sec:Pre}, we characterize the client's risk aversion with the triplet $(c_0,\mu_0,\gamma_0)\in(0,1)\times\cP([0,1])\times(0,1)$, where $c_0$ characterizes a stationary state-dependent cost function as in~\eqref{eq:Setupc0}, $\mu_0$ is a probability measure on $(0,1)$ playing a role similar to~\eqref{eq:Defrhomu}, and $\gamma_0$ represents the discount factor. Adopting the risk-averse dynamic programming framework of \cite{Ruszczynski2010Risk}, we detail our model of client's risk evaluation in Definition~\ref{def:DRM}.

Theorem~\ref{thm:Separation2} is the infinite-horizon counterpart to Theorem~\ref{thm:Separation}. This indicates that for any two distinct risk aversion profiles, there must exist a binary choice question (see Section \ref{subsec:ProofSeparation2} for detailed settings) that leads to different choices. The proof is provided in Section~\ref{subsec:ProofSeparation2}.
\begin{theorem}\label{thm:Separation2}
For any $(c',\mu',\gamma')\in(0,1)\times\cP([0,1])\times(0,1)$ that is different from $(c_0,\mu_0,\gamma_0)$, there exists $H\in\big(\cP(\bX)\times\cM(\bX)\big)^\bA$ such that
\begin{align*}
\argmin_{a\in\bA} \varrho_{c_0,\mu_0,\gamma_0}\left(H^a\right) \cap \argmin_{a\in\bA} \varrho_{c',\mu',\gamma'}\left(H^a\right) = \emptyset.
\end{align*}
\end{theorem}

The analogue of Theorem~\ref{thm:SuffFinite}, Corollary~\ref{cor:RandomDesignConstitent}, and Theorem~\ref{thm:QuantDisEnv} in this infinite-horizon setting can be established using arguments that parallel the static case. For the extension of Theorem~\ref{thm:QuantDisEnv} in particular, equations~\eqref{eq:InfHorStateHomo1} and~\eqref{eq:InfHorStateHomo2} can be applied when $(c_0,\mu_0)\neq(c',\mu')$, while equations~\eqref{eq:InfHorg1} and~\eqref{eq:InfHorg2} handle the case where $(c_0,\mu_0)=(c',\mu')$ but $\gamma_0\neq \gamma'$.

\subsection{A Bayesian framework}
In both the static and infinite-horizon settings, the following Bayesian framework for IRL is a reasonable alternative \citep[cf.][]{Buening2022Environment} to approach this learning problem. One reason to adopt the Bayesian framework is that it naturally takes into account sub-optimality, as illustrated below under the finite candidate setup (cf. Section~\ref{ssec:Experiments1}). Assume that the agent, given risk aversion $\ell$, chooses action $a$ with probability 
\begin{align*}
\frac{h(V(a,\ell,G))}{\sum_{i\in\bA}h(V(i,\ell,G))},
\end{align*}
where $h:\bR\to(0,\infty)$ is nonincreasing. Suppose additionally that the choices are independent across different rounds. After $N$ rounds of interactions, under the Bayesian paradigm that 
\begin{align*}
\text{posterior} \propto \text{likelihood} \times \text{prior} ,
\end{align*}
with uniform prior, the posterior distribution for a risk-aversion candidate $\ell$ is proportional to
\begin{align*}
\prod_{n=1}^N\frac{h(V(a_n^*,\ell,G_n))}{\sum_{i\in\bA}h(V(i,\ell,G_n))}.
\end{align*}
It is possible to consider $h$ unknown and integrate the estimation of $h$ as a part of the IRL. It would be interesting to study the identifiability and the convergence of posterior distributions under this setting. However, formulating a non-parametric Bayesian framework is challenging because it requires working with non-trivial probability measures on $\cP([0,1])$.

\subsection{Function approximation}
The challenge of the optimization problem in the infinite-horizon case lies in the handling of the dynamic programming equation $V$ in \eqref{eq:Evaluation}, as both the optimality gap and distinguishing power then involve the optimal value function. For finite spaces $\bX$ and $\bA$, one can make use of fixed-point iteration using a brute force approach for all possible environments and risk aversion candidates. Instead, another approach consists of using neural network as function approximators of $V$.
Using neural networks helps mitigate the cost of computing the distinguishing power for the learner every time a new environment $G$ appears. This ultimately becomes important when the action space $\bA$ and/or state space $\bX$ are continuous. Neural network structures are also known to be universal approximators \citep[cf.][]{pinkus1999approximation} which allows the estimation of $V$ to any arbitrary accuracy given a sufficiently large neural net.

One may consider a neural net, denoted $V^{\theta}$, parametrized by some parameters $\theta$ that takes as inputs a risk aversion characterization $(c_\ell, \mu_{\ell}, \gamma_\ell)$ as well as an environment $G$, and outputs $V(a,\ell,G)$ for all actions $a \in \bA$.
Depending on the class of dynamic risk measures under study, the neural network may be trained with a nested simulation framework to approximate general risk aversion characterizations under coherent risk measures \citep[see e.g.,][]{tamar2016sequential}.
Alternatively, one could focus on elicitable dynamic risk measures \citep[see e.g.,][]{coache2023conditionally,pesenti2025risk}, such as subclasses of spectral and distortion risk measures, and make use of strictly consistent scoring functions to efficiently approximate the dynamic risk without any nested simulation. 
We leave as future work a formal validation of this methodology for continuous states with function approximators.

\section{Concluding Remarks}\label{sec:Conclusion}

In this paper, we propose an IRL framework for eliciting an agent's risk preferences in an interactive manner, tailored to robo-advising questionnaires. We mainly focus on the static one-period case, where we study identifiability and convergence rate, but also conduct a preliminary analysis in the dynamic infinite-horizon case. In both settings, we prove the existence of a distinguishing question for any two risk aversions and derive lower bounds for the (relative) distinguishing power. In addition, we provide algorithms to solve this class of IRL problem. Our first set of experiments showcases that, for a finite set of candidate risk aversions, our design rules for selecting questions results in faster convergence than choosing the games uniformly. Our second set of experiments highlights that such approaches may be adapted to allow richer parametric spaces, with a suitable empirical efficiency in fewer than 50 rounds of interactions.

\section*{Acknowledgments}

ZC acknowledges financial support from the Guangzhou-HKUST(GZ) Joint Funding Program (Grant No. 2024A03J0630). AC acknowledges support from the Fonds de recherche du Québec -- Nature et technologies (B2X-270105).  SJ acknowledges support from the Natural Sciences and Engineering Research Council of Canada (RGPIN-2018-05705) and the University of Toronto's Data Science Institute.

\appendix

\newpage
\bibliographystyle{abbrv}
\bibliography{refs}

\newpage
\section{Related works}\label{sec:RelatedWorks}

The prevailing belief that behavioral demonstrations reflect human values has led to the development of IRL, a field dedicated to understanding an agent's objectives through their behavior. Specifically, IRL aims to estimate an agent's utility function by observing their assumed optimal policy within a known environment \citep[cf.][]{Ng2000Algorithms, Ng2004Apprenticeship, Ratliff2006Maximum, Ramachandran2007Bayesian, Ziebart2008Maximum}. We also refer to \cite{Arora2021Survey, Adams2022Survey, Chan2022Inverse} for surveys on recent progress of IRL.

This type of IRL problem is of great significance in behavioral prediction, where one gathers agent's reactions to a known environment in order to: (i) determine agent's preferences; or (ii) predict agent's reactions to other situations \citep[cf.][]{bain1995framework,ross2011reduction,osa2018algorithmic}. In this paper, we focus on the former. This is because we are ultimately interested in understanding the client's risk preferences in a robo-advisor context (see Section 1 for an overview), rather than merely replicating their (possibly sub-optimal) behavior as it is done in behavioral cloning.

Most studies in IRL operate under the assumption that the agent is risk-neutral. Even in this restricted setting, IRL still suffers from a lack of identifiability \citep{Ng2000Algorithms}. Numerous studies have attempted to address this issue. For instance, \cite{Fu2018Learning} proposes using an adversarial reward learning formulation. \cite{Guo2021Learning} and the references therein showcase that learning from exploring agents alleviates the identifiability issue given a limited number of agent demonstrations. \cite{Cao2021Identifiability} considers an entropy-regularized utility function and analyzes the sufficient condition for identifiability. \cite{Kim2021Reward} embeds the domain of a Markov decision process model into a graph and reasons about how properties of the graph relate to identifiability. \cite{Viktor2023Invariance} formally characterizes the partial identifiability of the reward function and analyze the impact of partial identifiability.  Several works, including \cite{Metelli2023Towards, Zhao2024Is}, develop theory for the feasibility set of reward functions and examine related convergence rates. \cite{ellis2024generalized} tackles the identifiability issue by considering acquisition functions that allow learning the reward function up to a behavioral equivalence class.

Despite the identifiability issue presented, in light of a series of empirical studies confronting the usage of expected utility in risk aversion modeling \citep{Allais1953Le, Ellsberg1961Risk, Kahneman1979Prospect}, we must consider components beyond expectation. The study of human risk aversion toward uncertainty has been extensively conducted in the fields of psychology, economics, finance, and marketing. We refer to \cite{He2022Wisdom} for a comprehensive list of such models.

A great number of studies utilize questionnaires for eliciting risk aversion \citep[cf.][and the references therein]{Charness2012Experimental, Harrison2016Cumulative, Rieger2017Estimating, Mata2018Risk, Horisch2018AreRisk, lHaridon2019All, Charness2021Experimental, Thompson2022Measuring}. The method of dynamic experimentation can be used to improve the efficiency of questionnaire. Instead of a set of pre-fixed questions in a regular questionnaire, dynamic experimentation presents different questions in real time based on the subject's previous responses. We provide a few works that are closely related to our setup. Within a hierarchical Bayesian framework, \cite{Toubia2013Dynamic} models the risk profile using cumulative prospect theory \citep{Tversky1992Dynamic} and designs the elicitation questions by maximizing the norm of the Hessian of the posterior distribution at its mode. As noted by \cite{Myunga2013Tutorial}, under suitable conditions, customizing a potentially highly sophisticated utility function for question design enables the problem to be recast in a Bayesian framework as one of adaptive experimental design \citep[cf.][]{Rainforth2024Modern, Huan2024Optimal, Greenhill2020Bayesian}. \cite{Cavagnaro2013Optimal} introduces an optimal design targeting indifference curves, which is tested on risk aversion models including weighted expected utility and cumulative prospect theory. \cite{Fidanoski2023ztree} proposes a z-tree method that train in advance a dynamic querying strategy to accelerate the execution of \cite{Toubia2013Dynamic}.  In \cite{Majumdar2017Risk}, the authors use coherent risk measures, introduced in the seminal works of \cite{Artzner1999Coherent, Delbaen2002Coherent} for their mathematically sound properties, to model risk aversion. The proposed method is further investigated in \cite{Singh2018Risk} from semi- and non-parametric perspectives. They consider cases with both known and unknown costs. For the known cost scenario, they propose a linear programming approach to elicit the risk profile and prove its consistency. This computational approach is then extended to the case of unknown costs. The numerical efficiency of this method was improved in \cite{Chen2019Active} through the adoption of an additional active learning component, which allows for querying the agent for additional demonstrations. \cite{Fujita2023Adaptive} considers a mixture of risk aversion models and employ a selection criterion based on Fisher information matrices. More experimental studies on this adaptive framework can be found in \cite{Amin2016Towards, Amin2017Repeated, Buening2022Environment, Buening2022Interactive, Wang2025Bayesian} and the references therein, although these are primarily based on expected utility. Nonetheless, the majority of these methods are either designed for parametric models or necessitate substantial, case-specific modifications to be applied in a nonparametric setting. We point out that our data consists of binary-choice questions rather than direct measurements. While binary responses are common in experimental design, our goal is to learn a nonparametric risk profile $(C_0,\mu_0)$ from such comparisons, which alters the nature of the question design and subsequent inference. This viewpoint demands a specially tailored approach that, in essence, reduces the problem to a computational task.

To be comprehensive, we highlight below a few more studies that elicit risk aversions modeled by convex risk measures. Although these studies fall outside the particular adaptive IRL framework we consider, they offer valuable alternative insights on the matter. \cite{Kalinchenko2013Calibrating} calibrates the risk preferences of investors based on generalized capital asset pricing model with data from put-option prices on the S\&P 500. \cite{Cox2014Utility} infers the utility function from optimal consumption and investment strategies. It analyzes the existence of a solution in continuous-time settings and offers insightful analysis on the identifiability issue. \cite{Grechuk2014Inverse} and \cite{Grechuk2016Inverse} examine the inverse portfolio problem for risk measures, exploring the existence of solutions for the inverse problem across different settings. For inference based on various forms of observations, we refer to \cite{Ratliff2020Inverse,Dai2025Preference} and the references therein. \cite{Alsabah2020Robo} models risk aversion using a versatile parametric family. They then learn the risk aversion over time by observing the agent's portfolio choices in different market environments under a certain invertibility assumption. While not limited to inferring the agent's objective, \cite{Li2018Minimizing} formulates a constrained optimization problem based on a series of assessments that infer the agent's risk aversion, with solution methods for such constrained optimization further studied in \cite{Li2021Minimizing}.

\section{Proofs}\label{sec:Proofs}

\subsection{An important technical lemma}\label{subsec:Pf:IndCurve}
In this section, we investigate the concept of indifference curve $g_{c,\mu}(p)$ introduced in \eqref{eq:Defg}. As we will see, the function $g_{c,\mu}$ plays a crucial role in proving several of our main results.

Let
\begin{align}\label{eq:Defrmu}
\underline p_\mu:=\inf\left\{p\in\bR:\int_{1-p}^1\sigma_\mu(\kappa)\dif\kappa=1\right\}.
\end{align}
By Lemma~\ref{lem:Spectral}, the definition in \eqref{eq:Defrmu} in particular implies that 
\begin{align}\label{eq:MonoIntsigma}
\text{$p\mapsto\int^1_{1-p}\sigma_\mu(\kappa)$ is strictly increasing on $[0,\underline p_{\mu})$ and constant on $[\underline p_\mu,1]$.}
\end{align}
The lemma below reveals the properties of $g_{c,\mu}$.
\begin{lemma}\label{lem:g}
Let $c\in(0,1]$ and $\mu\in\cP([0,1])$ with $\mu(\set{1})=0$. The following statements are true:
\begin{itemize}
\item[(a)] For any $p\in[0,1]$, we have
\begin{align}\label{eq:IdenIntg}
\int_{1-g_{c,\mu}(p)}^1 \sigma_\mu(\kappa)\dif\kappa = c\int_{1-p}^1\sigma_\mu(\kappa)\dif\kappa.
\end{align}
\item[(b)] For $p\in[0,1]$, if $q<g_{C,\mu}(p)$, then
\begin{align}
\int^1_{1-q}\sigma_\mu(\kappa)\dif\kappa < c\int^1_{1-g_{c,\mu}(p)}\sigma_\mu(\kappa)\dif\kappa.
\end{align}
The analogue is true with ``$<$'' replaced by ``$>$''.
\item[(c)] $g_{c,\mu}$ is strictly increasing in $[0, \underline p_\mu)$, and constant in $[\underline p_\mu,1]$.
\item[(d)] $g_{C,\mu}(p)=0$ only at $p=0$, and $g_{C,\mu}(1)=\inf\set{q\in[0,1]:\int_{1-q}^1\sigma_\mu(\kappa)\dif\kappa\ge c}<\underline p_\mu$.
\item[(e)] For any $p\in[0,1]$, we have $g_{c,\mu}(p) \le c p$.
\item[(f)] $g_{c,\mu}$ is $c$-Lipschitz continuous,\footnote{This implies that $g_{c,\mu}$ is almost everywhere differentiable and satisfies $g_{c,\mu}(b)-g_{c,\mu}(a) = \int_{a}^b g'_{c,\mu}(p)\dif p$ for any $a,b\in[0,1]$. See, e.g., \cite[Chapter 5]{Royden1988book}} and satisfies for almost every $p\in[0,1]$ that
\begin{align}\label{eq:IDsigma}
\sigma_{\mu}(1-p) = c^{-1}  \sigma_{\mu}\big(1-g_{c,\mu}(p)\big)\dot g_{c,\mu}(p).
\end{align}
\end{itemize}

\end{lemma}
\begin{proof}
\textbf{(a)} In view of \eqref{eq:Defg}, leveraging the definition of infimum, there exists $(q_i)_{i\in\bN}$ such that 
\begin{align}
\int_{1-q_i}^1\sigma_\mu(\kappa)\dif\kappa \ge c \int_{1-p}^1\sigma_\mu(\kappa)\dif\kappa, \quad i\in\bN
\end{align}
and $\lim_{i\to\infty} q_i = g_{c,\mu}(p)$. The above together with the continuity of $p\mapsto\int^1_{1-p}\sigma_\mu(\kappa)\dif\kappa$ implies 
\begin{align}\label{eq:UBIntg}
\int_{1-g_{c,\mu}(p)}^1\sigma_\mu(\kappa)\dif\kappa  = \lim_{i\to\infty} \int_{1-q_i}^1\sigma_\mu(\kappa)\dif\kappa \ge c \int_{1-p}^1\sigma_\mu(\kappa)\dif\kappa.
\end{align}
On the other hand, by \eqref{eq:Defg} again, for any $q'<g_{C,\mu}(p)$, we must have 
\begin{align}
\int_{1-q'}^1\sigma_\mu(\kappa)\dif\kappa  < c \int_{1-p}^1\sigma_\mu(\kappa)\dif\kappa.
\end{align}
By continuity, we have 
\begin{align}\label{eq:LBIntg}
\int_{1-g_{c,\mu}(p)}^1\sigma_\mu(\kappa)\dif\kappa = \lim_{q'\to g_{c,\mu}(p)-}\int_{1-q'}^1\sigma_\mu(\kappa)\dif\kappa \le c \int_{1-p}^1\sigma_\mu(\kappa)\dif\kappa.
\end{align}
By combining \eqref{eq:UBIntg} and \eqref{eq:LBIntg}, we conclude the proof.
\\
\textbf{(b)} Note that, by Lemma~\ref{lem:Spectral}, $\sigma_\mu\ge 0$ and $\int_{1-p}^1\sigma_\mu(\kappa)\dif\kappa\le 1$ for $p\in[0,1]$. Therefore, 
\begin{align}\label{eq:gUpperBound}
g_{C,\mu}(p) \le \inf\bigg\{q\in[0,1]:\int_{1-q}^1\sigma_\mu(\kappa)\dif\kappa\ge c\bigg\} < \underline p_\mu.
\end{align}
The last inequality in \eqref{eq:gUpperBound} is indeed strict; otherwise, we yield the contradiction
\begin{align*}
1=\int_{1-\underline p_\mu}^1\sigma_\kappa\dif\kappa = c 
\end{align*}
from \eqref{eq:Defrmu} and statement (a). Lastly, by combining statement (a) and the fact that $p\mapsto\int_{1-p}^1\sigma(\kappa)\dif\kappa$ is strictly increasing in $[0,\underline p_\mu)$ by construction \eqref{eq:Defrmu}, we yield the strict monotonicity of $g_{c,\mu}$ on $[0,\underline p_\mu)$. The rest of the statement follows from that $p\mapsto\int_{1-p}^1\sigma(\kappa)\dif\kappa$ is constant in $[\underline p_\mu, 1]$ due to \eqref{eq:Defrmu}. \\
\textbf{(c)} In view of \eqref{eq:Defg} and \eqref{eq:MonoIntsigma}, $g_{c,\mu}$ must be constant in $[\underline p_\mu,1]$. Regarding the strict monotonicity of $g_{C,\mu}$ in $[0,\underline p_\mu)$, we proceed by contradiction. Suppose there exist $p,p'$ such that $0\le p< p'<\underline p_\mu$ and $g_{c,\mu}(p)\ge g_{c,\mu}(p')$. Then, by statement (a), we have
$$c\int_{1-p}^1\sigma_\mu(\kappa)\dif\kappa = \int_{1-g_{c,\mu}(p)}^1 \sigma_{\mu}(\kappa)\dif\kappa \ge \int_{1-g_{c,\mu}(p')}^1 \sigma_{\mu}(\kappa)\dif\kappa = c \int_{1-g_{c,\mu}(p')}\sigma(\kappa)\dif\kappa,$$ 
which contradicts \eqref{eq:MonoIntsigma}. \\
\textbf{(d)} It is clear from \eqref{eq:Defg} that $g_{C,\mu}(0)=0$. For $p>0$, by Lemma~\ref{lem:Spectral}, we must have $\sigma_\mu(\kappa)>0$ on a neighborhood of $\kappa=1$, and thus $g_{c,\mu}(p)>0$. The second part of statement (d) was proved in \eqref{eq:gUpperBound}.\\
\textbf{(e)} Due to the monotonicity of $\sigma_\mu$ in Lemma \ref{lem:Spectral}, $p\mapsto \int^1_{1-p}\sigma_\mu$ is concave. It follows that 
\begin{align}
c\int^1_{1-p}\sigma_\mu(\kappa)\dif\kappa = (1-c)\int^1_{1-0}\sigma_\mu(\kappa)\dif\kappa + c\int^1_{1-p}\sigma_\mu(\kappa)\dif\kappa \le \int^1_{1-cp}\sigma_\mu(\kappa)\dif\kappa.
\end{align}
In view of the definition of $g_{c,\mu}$ in \eqref{eq:Defg}, we conclude the proof.\\
\textbf{(f)} Once we establish the Lipschitz continuity of $g_{c,\mu}$, we can differentiate both sides of \eqref{eq:IdenIntg} with respect to $p$ to derive the identity. The remainder of the proof focuses on the Lipschitz continuity of $g_{c,\mu}$. Let $\delta>0$ be sufficiently small. By statement (a) (d) (e) and that $\sigma_\mu$ is nondecreasing as stated in Lemma~\ref{lem:Spectral}, we have
\begin{align*}
&\int_{1-(g_{c,\mu}(p)+c\delta)}^1 \sigma_\mu(\kappa)\dif\kappa = c\int_{1-p}^1 \sigma_\mu(\kappa)\dif\kappa + \int_{1-(g_{c,\mu}(p)+c\delta)}^{1-g_{c,\mu}(p)}\sigma_\mu(\kappa)\dif\kappa\\
&\quad\ge c \int_{1-p}^1 \sigma_\mu(\kappa)\dif\kappa + \int^{1-cp}_{1-(cp+c\delta)}\sigma_\mu(\kappa)\dif\kappa \ge c \int_{1-p}^1 \sigma_\mu(\kappa)\dif\kappa + \int^{1-p}_{1-(p+c\delta)}\sigma_\mu(\kappa)\dif\kappa.
\end{align*}
By change of variable,
\begin{align}
\int^{1-p}_{1-p-c\delta}\sigma_\mu(\kappa)\dif\kappa \stackrel{\kappa=cy+(1-c)(1-p)}{=} c \int_{1-p-\delta}^{1-p} \sigma_\mu\big(cy+(1-c)(1-p)\big) \dif y \ge c \int^{1-p}_{1-p-\delta}\sigma_\mu(y)\dif y,
\end{align}
where, in the last inequality, we have used the monotonicity of $\sigma_\mu$ in Lemma~\ref{lem:Spectral} and the fact that $cy+(1-c)(1-p) \ge y$ for $y\in[0,1-p]$. By combining the above, we yield
\begin{align}
\int_{1-(g_{c,\mu}(p)+c\delta)}^1 \sigma_\mu(\kappa)\dif\kappa \ge c\int^1_{1-p}\sigma_\mu(\kappa)\dif\kappa.
\end{align}
This together with definition~\ref{eq:Defg} and statement (c) implies
\begin{align*}
g_{c,\mu}(p)+c\delta \ge g_{c,\mu}(p+\delta) > g_{c,\mu}(p),
\end{align*}
i.e., $g_{c,\mu}$ is $c$-Lipschitz continuous. 
\end{proof}

\subsection{Proof of Proposition~\ref{prop:Utilde}}\label{subsec:Pf:prop:Utilde}

\begin{proof}[Proof of \eqref{eq:IdenUtilde}]
For $f\in\cF_{\uparrow}$, we introduce the extended inverse $f^{-1}(u):=\inf\set{\alpha\in[0,1]:f(\alpha)\ge u}$ with the convention that $\inf\emptyset=\infty$. For $J\in\bN$, define
\begin{align*}
f^J(\alpha) := \frac1J\sum_{j=1}^{J-1} \1_{[f^{-1}(\frac{j}{J}),1]}(\alpha), \quad\alpha\in[0,1].
\end{align*}
Since $f$ is monotone and bounded within $[0,1]$, we have $\|f-f^J\|_\infty \le J^{-1}$. Consequently,
\begin{align*}
\wt{\cU}(\tilde\mu,\tilde\mu') &\le \sup_{\cF_{\uparrow}}\left| \int_{[0,1)} f^J(\alpha)\big(\tilde\mu(\dif \alpha)-\tilde\mu'(\dif \alpha)\big) \right| + \frac2J\\
&= \frac1J \sup_{\cF_{\uparrow}} \sum_{j=1}^{J-1} \left| \tilde\mu\big([f^{-1}(\frac{j}{J}),1]\big) - \tilde\mu'\big([f^{-1}(\frac{j}{J}),1]\big) \right| + \frac2J\\
&\le \sup_{t\in[0,1]}\big|\tilde\mu([t,1])-\tilde\mu'([t,1])\big| + \frac2J.
\end{align*}
Letting $J\to\infty$, we have proved one direction of the domination. For the other direction, note that $\1_{[\alpha,1]}\in\cF_\uparrow$. The proof is complete.
\end{proof}

\begin{proof}[Proof of \eqref{eq:BBWassUtilde}]

To show the first inequality, let us consider a random variable $Z$ bounded within $[0,1]$. Denote $F^{-1}_Z$ the inverse CDF of $Z$. By \eqref{eq:DefAVaR}, the function $\kappa\mapsto\avar_\kappa(Z)$ is bounded by $[0,1]$ and locally $\kappa^{-1}$-Lipschitz continuous. Using Lemma~\ref{lem:Spectral} and the fact that $\avar_{\kappa}(Z)\in[0,1]$, it follows that for any $y\in(0,1)$,
\begin{align}
&\left| \int_0^1 F^{-1}_Z(\alpha)\left(\tilde\mu(\alpha)-\tilde\mu'(\alpha)\right) \right| =  \left| \int_0^1 \avar_{1-\alpha}(Z) \left(\mu(\dif\alpha)-\mu'(\dif\alpha)\right) \right| \\
&\quad \le \left| \int_0^1 \avar_{(1-\alpha)\vee y}(Z) \left(\mu(\dif\alpha)-\mu'(\dif\alpha)\right) \right| + 2y \le \frac{1}{y}\cW(\mu,\mu') + 2 y.
\end{align} 
Maximizing the right hand side above at $y^*=\sqrt{\frac12\cW(\mu,\mu')}$, we yield the first inequality.

Regarding the second inequality, by \eqref{eq:DefW},  \eqref{eq:Idenmusigma}, \eqref{eq:Defmutilde}, and \eqref{eq:IdenUtilde}, we have
\begin{align}\label{eq:UBWmu}
\cW(\mu,\mu') &= \int_0^1\big| \mu([0,x]) - \mu'([0,x]) \big|\dif x \le \int_0^1 (1-x) \big|\sigma_\mu(x)-\sigma_{\mu'}(x)\big|\dif x + \int_0^1\big| \tilde\mu([0,x]) - \tilde\mu'([0,x]) \big|\dif x\nonumber \\
&\le \int_0^1|\sigma_\mu(x)-\sigma_{\mu'}(x)|\dif x + 2\wt\cU(\mu,\mu').
\end{align} 
The rest of the proof is dedicated to bounding the first term in the right hand side of \eqref{eq:UBWmu}. 

To proceed, we claim that $\sigma_\mu-\sigma_{\mu'}$ switches sign at most countably many times. More precisely, there exist $(I^+_n)_{n\in\bN}, (I^-_n)_{n\in\bN}$ such that $I^\pm_n$'s are disjoint open intervals within $[0,1]$ and $\sigma_\mu-\sigma_{\mu'}$ is strictly positive (resp. negative) in $I^+_n$ (resp. $I^-_n$). Moreover, $(I^+_n)_{n\in\bN}, (I^-_n)_{n\in\bN}$ should satisfy that, with $$I^0:=\set{\kappa\in[0,1]:\sigma_\mu(\kappa)=\sigma_{\mu'}(\kappa)},$$ 
we have $[0,1]\setminus\left(I^0 \cup \bigcup_{n\in\bN} I^+_n \cup \bigcup_{n\in\bN} I^-_n \right)$ has $0$ Lebesgue measure. To show this, we let\footnote{The left and right limit must exist as $\sigma_\mu$ is monotone and bounded as argued by Lemma~\ref{lem:Spectral}.} 
\begin{align}
S_{\text{jump}} := \left\{\kappa\in[0,1]: \lim_{y\to\kappa-}\sigma_{\mu}(\kappa)\neq\lim_{y\to\kappa+}\sigma_{\mu}(\kappa) \text{ or } \lim_{y\to\kappa-}\sigma_{\mu'}(\kappa)\neq\lim_{y\to\kappa+}\sigma_{\mu'}(\kappa) \right\}
\end{align}
i.e., $S_{\text{jump}}$ is the set of jump points of $\sigma_\mu$ and $\sigma_\mu'$. By Lemma~\ref{lem:Spectral}, $\sigma_\mu$ and $\sigma_{\mu'}$ are bounded and monotone, and thus $S_{\text{jump}}$ must be countable. It follows that $\sigma_\mu-\sigma_{\mu'}$ is continuous on $[0,1]\setminus S_J$. We construct $I^+$ as
$$I^+:=\left\{x\in[0,1]\setminus S_{\text{jump}}: \sigma_\mu(x)-\sigma_\mu(x') > 0 \right\}.$$
Note that $I^+$ is open due to the continuity above. If $I^+$ is not empty, then, by \cite[Proposition 1.5.8]{Royden1988book}, $I^+$ can be expressed as a countable union of disjoint intervals, which reveals the existence of $(I^+_n)_{n\in\bN}$. We apply the analogous argument to $I^-$.  Finally, since the countable set $S_{\text{jump}}$ has $0$ Lebesgue measure, we readily prove the claim.

Following the claim above, we merge $(I^+_n)_{n\in\bN}, (I^-_n)_{n\in\bN}$ into $(I_n)_{n\in\bN}$ and denote $I_n=(a_n,b_n)$. Consequently,
\begin{align*}
\int_0^1|\sigma_\mu(x)-\sigma_{\mu'}(x)|\dif x = \sum_{n=1}^\infty \int_{I_n}|\sigma_\mu(x)-\sigma_{\mu'}(x)|\dif x,
\end{align*}
where the sum on the right hand side above must converge as $\|\sigma_\mu-\sigma_{\mu'}\|_\infty<\infty$. Therefore, for any $k\in\bN$, there exists $N_k\in\bN$ such that 
\begin{align}\label{eq:UBIntAbsDiffsigma}
\int_0^1|\sigma_\mu(x)-\sigma_{\mu'}(x)|\dif x \le \sum_{n=1}^{N_k} \int_{I_n}|\sigma_\mu(x)-\sigma_{\mu'}(x)|\dif x + \frac1k.
\end{align}
By \eqref{eq:Defmutilde} and the fact that $\sigma_\mu-\sigma_{\mu'}$ does not switch sign on $I_n$, for any $n\in\bN$ we have 
\begin{align}
\int_{I_n}|\sigma_\mu(x)-\sigma_{\mu'}(x)|\dif x &= \big|\mu([a_n,b_n)) - \mu'([a_n,b_n))\big| \\
& \le \big|\mu([a_n,1]) - \mu'([a_n,1]) \big| - \big|\mu([b_n,1]) - \mu'([b_n,1]) \big| \le 2\wt\cU(\mu,\mu'), \label{eq:UBIntInAbsDiffsigma}
\end{align}
where the last inequality follows from Proposition~\ref{prop:Utilde}.

In view of \eqref{eq:UBIntAbsDiffsigma}, we continue to investigate the auxiliary problem below
\begin{align*}
\max_{N\in\bN,\; \ell_n, m_n \in\bR_+, n=1,\dots,N } \sum_{n=1}^N \ell_n m_n
\end{align*}
subject to 
\begin{gather}
\sum_{n=1}^N \ell_n = 1,\quad \sum_{n=1}^N m_n = 2\lambda,\quad \ell_n m_n \le 2\delta,\quad\ell_n,m_n \ge 0,\; n=1,\dots,N,
\end{gather}
where $\delta$ stands for $\wt\cU(\mu,\mu')$, $\ell_n$ and $m_n$ are dummy variables for $b_n-a_n$ and $\sup_{x\in I_n}|\sigma_{\mu}(x)-\sigma_{\mu'}(x)|$, respectively. Above, the constrain that  $\sum_{n=1}^N m_n = 2\lambda$ is imposed because either $\sigma_\mu$ or $\sigma_\mu'$ must increase by the amount of $\sup_{x\in I_n}|\sigma_{\mu}(x)-\sigma_{\mu'}(x)|$ during $I_n$. 

Momentarily excluding the constrain that $\ell_n m_n \le 2\delta$, the first order condition of the corresponding Lagrangian provides a unique solution. But the solution corresponds to the minimum. We thus search the boundary of the region enclosed by $\ell_n m_n \le 2\delta$ and $\ell_n,m_n\ge 0$ for $n=1,\dots,N$. 

In principle, we should exhaust the cases of $\ell_n m_n=2\delta,\,n=1,\dots,N_0$ for $N_0=1,\dots,N$. However, the sub problem
\begin{align*}
\max \left\{2N_0\delta + \sum_{n=N_0+1}^N \ell_n m_n\right\}
\end{align*} 
subject to
\begin{align*}
\sum_{n=N_0+1}^N\ell_n = 1 - \sum_{n=1}^{N_0}\ell_n,\quad\sum_{n=N_0+1}^N m_n=1-\sum_{n=1}^{N_0}m_n,\quad \ell_n m_n\le 2\delta,\quad\ell_n,m_n\ge0,\quad n=N_0+1,\dots,N,
\end{align*}
due to similar reasoning as above, must attain its maximum at the boundary of region enclosed by $\ell_n m_n\le 2\delta$ and $\ell_n,m_n\ge 0$ for $n\ge N_0+1$, unless $(1 - \sum_{n=1}^{N_0}\ell_n)(1-\sum_{n=1}^{N_0}m_n)<2\delta$. Therefore, we proceed to investigate the further auxiliary problem
\begin{align}\label{eq:MaxN}
\max N \in\bN
\end{align}
subject to
\begin{align*}
\sum_{n=1}^N\ell_n=1, \quad \sum_{n=1}^N m_n = 2\lambda,\quad \ell_n m_n = 2\delta,\quad n=1,\dots,N.
\end{align*}
We rewrite the constrain into
\begin{align*}
\sum_{n=1}^N\ell_n=1,\quad \sum_{n=1}^N\ell_n^{-1}=\frac{\lambda}{\delta}.
\end{align*}
Note that given an $\check N\in\bN$, the quantity $\ell^*_n=\check N^{-1}$ minimizes $\sum_{n=1}^{\check N}\ell_n^{-1}$ under the constrain $\sum_{n=1}^{\check N}\ell_n=1$. Consequently, $N^*=\lfloor \sqrt{\frac{\lambda}{\delta}} \rfloor$ attains \eqref{eq:MaxN}. This together with \eqref{eq:UBIntAbsDiffsigma} and \eqref{eq:UBIntInAbsDiffsigma} implies 
\begin{align*}
\int_0^1|\sigma_\mu(x)-\sigma_{\mu'}(x)|\dif x \le 2\sqrt{\lambda\delta} + \frac1k.
\end{align*}
Letting $k\to\infty$, we conclude the proof.
\end{proof}

\subsection{Proof of Lemma~\ref{lem:Cont}}\label{subsec:Pf:Lem:Cont}

We first establish the compactness of $\cP_\lambda([0,1])$.
\begin{lemma}\label{lem:Compact}
$\cP_\lambda([0,1])$ is compact under $\wt\cU$.
\end{lemma}
\begin{proof}
In view of \eqref{eq:BBWassUtilde}, it is sufficient to show the compactness of $\cP_\lambda([0,1])$ under $\cW$. It is well known that $\cW$ metrizes weak topology in $\cP([0,1])$ and, under which, $\cP([0,1])$ is compact; see, e.g., \cite[Corollary 6.13]{Villani2008book} for related arguments. Since $\cP_\lambda([0,1])\subset\cP([0,1])$, what is left to show is the closedness of $\cP_\lambda([0,1])$. To this end, we let $(\mu_n)_{n\in\bN}\subset\cP_\lambda([0,1])$ be a sequence converging weakly to $\mu^*$. Note that, for any $\delta\in(0,1)$, $\mu|_{[0,1-\delta]}$ also converges weakly. Consequently,
\begin{align*}
\int_0^\delta \frac{1}{1-\alpha}\mu^*(\dif \alpha) = \lim_{n\to\infty} \int_0^{1-\delta} \frac{1}{1-\alpha} \mu(\dif\alpha) \le \lambda.
\end{align*}
Letting $\delta\to0+$ and invoking the monotone convergence theorem, we obtain that $\mu^*\in\cP_\lambda([0,1])$.
\end{proof}

We are ready to prove Lemma~\ref{lem:Cont}.
\begin{proof}[Proof of Lemma~\ref{lem:Cont}]
Thanks to \eqref{eq:BBWassUtilde}, it is sufficient to prove the statement with $\cP_\lambda([0,1])$ equipped with $\cW$. In addition, by Lemma~\ref{lem:Compact} and the fact the continuity on compact set implies uniform continuity, we only need to prove the joint continuity in the arguments.

We first analyze the continuity of $\Phi$. To this end, note that the joint continuity of $(G,c)\mapsto\rho_{\mu}(C(X_{G^a}))$ is an immediate consequence of \eqref{eq:Defrhomu}. Since $\bA$ is discrete, we proceed to prove the continuity of $\Phi$ with respect to $\mu\mapsto\rho_{\mu}(C(X_{P}))$, uniformly in $(a,G,c)$. It follows from \eqref{eq:DefAVaR} that $\kappa\mapsto\avar_\kappa(C(X_{G^a}))$ is $(1-\kappa)^{-2}$-Lipschitz continuous. For $\delta\in(0,1)$, we let $$f^\delta_{a,G,c}(\kappa):=\1_{[0,1-\delta]}(\kappa)\avar_\kappa(C(X_{G^a})) + \1_{(1-\delta,1]}(\kappa)\avar_\delta(C(X_P)).$$ 
By \eqref{eq:Defrhomu}, for any $\mu,\mu'\in\cP_\lambda([0,1])$, we have
\begin{align*}
&\left|\rho_{\mu}(C(X_P)) - \rho_{\mu'}(C(X_P))\right| \\
&\quad\le \left|\int_{0}^1 f^\delta_{P,C}(\alpha) \big(\mu(\dif\alpha)-\mu'(\dif\alpha)\big) \right| +  \left|\int_{0}^1 \left(\avar_\alpha(C(X_P)) - f^\delta_{P,C}(\alpha)\right) \big(\mu(\dif\alpha)-\mu'(\dif\alpha)\big) \right| \\
&\quad\le \delta^{-2}\cW(\mu,\mu') + 2\lambda\delta,
\end{align*}
which indicates a modulus of continuity $w\mapsto \inf_{\delta\in(0,1)}\left\{\delta^{-2}w+2\lambda\delta\right\}=3\lambda^{\frac23}w^{\frac13}$, regardless of other arguments. After a straightforward application of triangle inequality, we obtain the joint continuity of $\Phi$. Since sum and minimization preserve continuity, the joint continuity of $\Phi$ follows automatically from \eqref{eq:DefOptGap}.

As for $\Psi$, following \eqref{eq:IdenPsi}, it can be verified that
\begin{align}
\Psi\big(G,(c,\mu),(c',\mu')\big) = \sqrt{\Phi\big(\text{A},G,(c,\mu)\big) \Phi\big(\text{B},G,(c',\mu')\big)} + \sqrt{\Phi\big(\text{B},G,(c,\mu)\big) \Phi\big(\text{A},G,(c',\mu')\big)}.
\end{align}
By combining this representation with the continuity of $\Phi$, we yield the continuity of $\Psi$. The continuity of $\Xi$ follows analogously.
\end{proof}

\subsection{Proof of Theorem~\ref{thm:Separation}}\label{subsec:ProofSeparation}


\begin{proof}[Proof of Theorem~\ref{thm:Separation}]
It is sufficient to prove that for different $(c,\mu)$ and $(c',\mu')$, there is always a $p\in(0,1)$ such that $g_{c,\mu}(p)\neq g_{c',\mu'}(p)$. Indeed, without loss of generality, we assume $g_{c,\mu}(p)<g_{c',\mu'}(p)$. Pick $q\in(g_{c,\mu}(p),g_{c',\mu'}(p))$. Then, by Lemma~\ref{lem:g} (b), we have $\rho_\mu(C(X_{G^{\text{A}}_{p,q}})) < \rho_\mu(C(X_{G^{\text{B}}_{p,q}}))$, while $\rho_{\mu'}(C'(X_{G^{\text{A}}_{p,q}})) > \rho_{\mu'}(C'(X_{G^{\text{B}}_{p,q}}))$, i.e., $G_{p,q}$ is the distinguishing environment.

For the sake of neatness, below we write 
\begin{align}\label{eq:Defh}
h_{c,\mu}(p) := \int_{1-p}^1 \sigma_{\mu}(\kappa)\dif\kappa, \quad p\in[0,1].
\end{align}
Note that $h_{c,\mu}$ is concave due to the monotonicity of $\sigma_\mu$ as listed in Lemma~\ref{lem:Spectral}, and satisfies $h(0)=0$ and $h(1)=1$.
We will prove by contradiction in two different cases. 

\textbf{Case 1, $c\neq c'$.} We first consider the case where $c\neq c'$, while $\mu$ and $\mu'$ may or may not be different. Without loss of generality, we assume $c<c'$. We define $p_1=1$ and $p_{k+1}=g_{c,\mu}(p_{k})$ for $k\in\bN$. It follows from Lemma~\ref{lem:g} (e) that $(p_k)_{k\in\bN}$ is positive and strictly decreasing. Moreover, $h_{c,\mu}(p_1)=1$ by \eqref{eq:Defh} and Lemma~\ref{lem:Spectral}, and thus $h_{c,\mu}(p_k) = c^{k-1}$ for $k\in\bN$ by Lemma~\ref{lem:g} (a). Similarly, $h_{c',\mu'}(p_k)={c'}^{k-1}$ for $k\in\bN$. It follows that
\begin{align*}
\frac{h_{c',\mu'}(p_{k})-h_{c',\mu'}(p_{k+1})}{p_k-p_{k+1}} = \frac{{c'}^{k-1}-{c'}^{k}}{p_k-p_{k+1}} = \frac{{c'}^{k-1}-{c'}^{k}}{{c'}^{k-1}(1-c)} \frac{{c'}^{k-1}}{{c}^{k-1}} \frac{{c}^{k-1}-{c}^{k}}{p_k-p_{k+1}} = \frac{1-c'}{1-c} \frac{{c'}^{k-1}}{{c}^{k-1}} \frac{h_{c,\mu}(p_{k}) - h_{c,\mu}(p_{k+1})}{p_k-p_{k+1}}.
\end{align*}
For $k\ge 3$, we have $p_{k+1}<p_{k}<\cdots<p_1=1$. We continue to the concavity of $h_{c',\mu'}$ and yield
\begin{align*}
\frac{h_{c',\mu'}(p_{k})-h_{c',\mu'}(p_{k+1})}{p_k-p_{k+1}} \ge \frac{1-c'}{1-c} \frac{{c'}^{k-1}}{{c}^{k-1}}\frac{h_{c,\mu}(p_1)-h_{c,\mu}(p_2)}{p_1-p_2} = \frac{{c'}^{k-1}}{{c}^{k-1}}\frac{1-c'}{p_1-p_2}.
\end{align*}
In view of \eqref{eq:Defh}, by the monotonicity of $\sigma_\mu$ as listed in Lemma~\ref{lem:Spectral}, we have 
\begin{align}
\limsup_{p\to 1-}\sigma_{\mu'}(p) \ge \sigma_{\mu'}(1-p_{k+1}) \ge    \frac{h_{c',\mu'}(p_{k})-h_{c',\mu'}(p_{k+1})}{p_k-p_{k+1}} \ge \frac{{c'}^{k-1}}{{c}^{k-1}}\frac{1-c'}{p_1-p_2}.
\end{align}
Since $k$ is arbitrary, under the setting that $c<c'$, we yield a contradiction to Assumption~\ref{assump:Basic}.

\textbf{Case 2, $c=c'$ but $\mu\neq\mu'$.} We now suppose $c=c'$, but $\mu\neq\mu'$. We define $p_1=1$ and $p_{k+1}=g_{c,\mu}(p_{k})$ for $n\in\bN$. Similarly as before, $(p_k)_{k\in\bN}$ is positive and strictly decreasing. Additionally, $h_{c,\mu}(p_k)=h_{c,\mu'}(p_k)=c^{k-1}$. We must have $\lim_{k\to\infty}p_k=0$; otherwise by Lemma~\ref{lem:g} (e), we have $g_{C,\mu}(p)=0$ for some $p>0$, contradicting Lemma~\ref{lem:g} (d).
 
We claim that there must be a $\hat p\in(0,1)$ such that $h_{c,\mu}(\hat p)\neq h_{c,\mu'}(\hat p)$. We will show this claim by contradiction. Suppose otherwise, then by \eqref{eq:Defh}, differentiation of integration formula \cite[Theorem 5.3.10]{Royden1988book}, and the right continuity of $\sigma_\mu,\sigma_{\mu'}$ from Lemma~\ref{lem:Spectral}, we have $\sigma_{\mu}=\sigma_{\mu'}$. By Lemma~\ref{lem:Spectral} again, we have $\mu([0,\alpha])=\mu'([0,\alpha])$ for $\alpha\in[0,1)$. It follows from monotone class lemma \cite[Section 4.4, Lemma 4.13]{Aliprantis2006book} that $\mu=\mu'$, contradicting the setting that $\mu\neq\mu'$. 

We define $\hat c:=h_{c,\mu}(\hat p)$ and $\hat c':=h_{c,\mu'}(\hat p)$. Without loss of generality, we assume that $\hat c < \hat c'$. Moreover, since $(p_k)_{k\in\bN}$ is strictly decreasing and $\lim_{k\to\infty}p_k=0$, we have $\hat p\in(p_{k_0+1},p_{k_0})$ for some $k_0\in\bN$. Note that 
\begin{align}\label{eq:cOrder}
c^{k_0}<\hat c<\hat c'< c^{k_0-1}.
\end{align}
We let $\bar p_{k}:=p_{k_0+k-1}$ for $k\in\bN$. Additionally, we define $\hat p_1:=\hat p$ and $\hat p_{k+1}=g_{C,\mu}(\hat p_k)$ for $k\in\bN$. By Lemma~\ref{lem:g} (a), we have $h_{c,\mu}(\hat p_k)=c^{k-1}\hat c$. The above together with Lemma~\ref{lem:g} (a) again implies 
\begin{align}\label{eq:hqValues}
h_{c,\mu'}(\bar p_k)=h_{c,\mu'}(\bar p_k)= c^{k_0+k-2}, \quad h_{c,\mu}(\hat q_k)=\hat c c^{k-1},\quad h_{c,\mu'}(\hat q_k)=\hat c' c^{k-1}.
\end{align}
By combining \eqref{eq:cOrder}, \eqref{eq:hqValues} and the monotonicity of $h_{c,\mu}$, we have 
\begin{align}\label{eq:pOrder}
\hat p_{k+1} < \bar p_{k+1} < \hat p_k < \bar p_k.
\end{align}
It follows from \eqref{eq:pOrder}, the concavities of $h_{c,\mu'}$, and \eqref{eq:hqValues} that 
\begin{align*}
\frac{h_{c,\mu'}(\bar p_{k+1}) - h_{c,\mu'}(\hat p_{k+1})}{\bar p_{k+1} - \hat p_{k+1}} \ge \frac{h_{c,\mu'}(\hat p_{k}) - h_{c,\mu'}(\bar p_{k+1})}{\hat p_{k} - \bar p_{k+1}} = \frac{\hat c' c^{k-1} - c^{k_0+k-1}}{\hat p_{k} - \bar p_{k+1}}.
\end{align*}
Recall that $\hat c<\hat c'$. By \eqref{eq:hqValues}, \eqref{eq:pOrder} and the concavity of $h_{c,\mu}$ in \eqref{eq:Defh}, we yield
\begin{align*}
&\frac{h_{c,\mu'}(\bar p_{k+1}) - h_{c,\mu'}(\hat p_{k+1})}{\bar p_{k+1} - \hat p_{k+1}} \ge  \frac{\hat c c^{k-1} - c^{k_0+k-1}}{\hat p_{k} - \bar p_{k+1}} = \frac{h_{c,\mu}(\hat p_{k}) - h_{c,\mu}(\bar p_{k+1})}{\hat p_{k} - \bar p_{k+1}} \ge \frac{h_{c,\mu}(\bar p_{k}) - h_{c,\mu}(\hat p_{k})}{\bar p_{k} - \hat p_{k}}\\
&\quad = \frac{c^{k_0+k-2} - \hat c c^{k-1}}{\bar p_{k} - \hat p_{k}} = \frac{c^{k_0-1}-\hat c}{c^{k_0-1}-\hat c'} \frac{c^{k_0+k-2} - \hat c' c^{k-1}}{\bar p_{k} - \hat p_{k}} =  \frac{c^{k_0-1}-\hat c}{c^{k_0-1}-\hat c'} \frac{h_{c,\mu'}(\bar p_{k}) - h_{c,\mu'}(\hat q_{k})}{\bar p_{k} - \hat p_{k}}.
\end{align*}
By induction, we have 
\begin{align*}
\frac{h_{c,\mu'}(\bar q_{k}) - h_{c,\mu'}(\hat q_{k})}{\bar p_{k} - \hat p_{k}} \ge \left(\frac{c^{k_0-1}-\hat c}{c^{k_0-1}-\hat c'}\right)^{k-1} \frac{h_{c,\mu'}(\bar q_{1}) - h_{c,\mu'}(\hat p_{1})}{\bar p_{1} - \hat p_{1}} = \left(\frac{c^{k_0-1}-\hat c}{c^{k_0-1}-\hat c'}\right)^{k-1}\frac{c^{k_0-1}-\hat c}{\bar p_{k_0} - \hat p}.
\end{align*}
This together with \eqref{eq:Defh} and the monotonicity of $\sigma_\mu$ in Lemma~\ref{lem:Spectral} implies
\begin{align}
\sigma_{\mu'}(1) \ge \sigma_{\mu'}(1-\hat p_k) \ge \frac{h_{c,\mu'}(\bar p_{k}) - h_{c,\mu'}(\hat p_{k})}{\bar p_{k} - \hat p_{k}} \ge \left(\frac{c^{k_0-1}-\hat c}{c^{k_0-1}-\hat c'}\right)^{k-1}\frac{c^{k_0-1}-\hat c}{\bar p_{k_0} - \hat p},
\end{align}
where we note $\frac{c^{k_0-1}-\hat c}{c^{k_0-1}-\hat c'}>1$ due to \eqref{eq:cOrder}. Since $k$ is arbitrary, we yield a contradiction against Assumption~\ref{assump:Basic}.
\end{proof}

\subsection{Proof of Remark~\ref{rmk:GenSep}}\label{subsec:Pf:rmk:GenSep}
\begin{proof}
Let $\bX=\set{x_0,\dots,x_d}$ and consider two different $C,C':\bX\to[0,1]$. Suppose 
\begin{align}
0=C(x_0)\le C(x_1) \le \dots \le C(x_d)
\end{align}
and similarly for $C'$. Assume that $|C(\bX)|,|C'(\bX)|\ge 3$. In the case where there is an $x_i$ such that $C(x_i)\neq C'(x_i)$ and $C(x_i), C'(x_i)\in(0,1)$, the distinguishing $G$ can be constructed following Section~\ref{subsec:ProofSeparation}. 

We now address the complementary case. In this case, there must exists an $x_i$ satisfying either $C(x_i)\in\set{0,1}$, or $C'(x_i)\in\set{0,1}$, but not both; otherwise we must have $|C(\bX)|=|C'(\bX)|=2$. This is further divided into two (potentially overlapping) subcases for analysis:

\paragraph{Case R1.} Suppose there is an $x_i$ such that $C(x_i)\in(0,1)$ and $C'(x_i)=0$. We then consider a restricted state space $\set{x_0,x_i,x_d}$ and let 
\begin{align*}
\check G_{p} = \bordermatrix{~ & \text{A} & \text{B} \cr x_0 & 1-p & p \cr x_i & p & 0 \cr x_d & 0 & 1-p\cr}.
\end{align*}
It is clear that, under $C'$, regardless of $\mu'$, action A is preferred for any $p\in[0,1)$. On the other hand, it is straightforward to verify that, under $C$, there is a $p$ sufficiently close to $1$ such that action B is preferred regardless of $\mu$ or $\mu'$. This constructs the distinguishing question.

\paragraph{Case R2.} Suppose there is an $x_i$ such that $C(x_i)\in(0,1)$ and $C'(x_i)=1$. Similarly as before, we consider a restricted state space $\set{x_0,x_i,x_d}$ and let 
\begin{align*}
G_{p,q} = \bordermatrix{~ & \text{A} & \text{B} \cr x_0 & 1-p & 1-q \cr x_i & p & 0 \cr x_d & 0 & q\cr}. 
\end{align*}
The rest of the construction is identical to that for Case R1. Specifically, in Case R2, we have $g_{C',\mu'}(p)=p$.
\end{proof}

\subsection{Proof of Theorem~\ref{thm:SuffFinite}}\label{subsec:Pf:thm:SuffFinite}

\begin{proof}[Proof of Theorem~\ref{thm:SuffFinite}]
In this proof, we fix $\lambda$ and $\varepsilon$ and proceed by contradiction. Suppose that, for any $N\in\bN$ and $G_1,\dots,G_N\in\cP(\bX)^\bA$, there exists $(c,\mu)\in B^{\mathsf{c}}_{\lambda,\varepsilon}$ such that for $n=1,\dots, N$,
\begin{align*}
a^*_n \in \argmin_{a\in\bA} \rho_{\mu}(C(X^a_{G_n}))
\end{align*}
for some $a^*_n\in\argmin_{a\in\bA} \rho_{\mu_0}(C_0(X^a_{G_n}))$. Let $\cG^k$ be a $\frac1k$-net of $\cP(\bX)^\bA$, that is, $\cG^k$ is a subset of $\cP(\bX)^{\bA}$ such that $\bigcup_{G\in\cG^\delta} B_{\frac1k}(G) \supseteq \cP(\bX)^{\bA}$, where $B_{\frac1k}(G)$ is an open ball centered at $G$ with radius ${\frac1k}$ under, e.g., $1$-norm. By the opposite hypothesis, there is $(c_k,\mu_k)\in B^{\mathsf{c}}_{\lambda,\varepsilon}$ such that 
\begin{align*}
\min_{G\in\cG^k}\Psi\big(G,(c_k,\mu_k),(c_0,\mu_0)\big) = 0.
\end{align*}
In view of Lemma~\ref{lem:Compact}, without loss of generality, we assume $\set{(c_k,\mu_k)}_{k\in\bN}$ converges to $(c_*,\mu_*)$. By the setting that $(c_k,\mu_k)\in B^{\mathsf{c}}_{\lambda,\varepsilon}$ and the continuity of metric, we must have  $(c_*,\mu_*)\neq(c_0,\mu_0)$. In view of Lemma~\ref{lem:Cont}, we let $\beta$ be the uniform modulus of continuity of $G\mapsto\Psi\big(G,(c,\mu),(c',\mu')\big)$, i.e.,
\begin{align*}
\beta(\delta) := \sup_{(c,\mu),(c',\mu')\in[0,1]^{\bX}\times\cP_\lambda([0,1])}\; \sup_{\|G-\hat G\|_1\le\delta}\left| \Psi\big(G,(c_0,\mu'),(c_0,\mu')\big) - \Psi\big(\hat G,(c,\mu),(c',\mu')\big)\right|.
\end{align*}
Consequently, 
$$\min_{G\in\cP(\bX)^{\bA}} \Psi\big(G,(c_k,\mu_k),(c_0,\mu_0)\big) \le \beta\left(\frac1k\right),\quad k\in\bN.$$
Moreover, by Lemma~\ref{lem:Cont} and \cite[Lemma 17.29 and Lemma 17.30]{Aliprantis2006book}, $(c,\mu)\mapsto\min_{G\in\cP(\bX)^{\bA}}\Psi\big(G,(c,\mu),(c_0,\mu_0)\big)$ is continuous. It follows that
\begin{align}
\min_{G\in\cP(\bX)^{\bA}} \Psi\big(G,(c_*,\mu_*),(c_0,\mu_0)\big) = \lim_{k\to\infty}\min_{G\in\cP(\bX)^{\bA}} \Psi\big(G,(c_k,\mu_k),(c_0,\mu_0)\big)   \le \lim_{k\to\infty}\beta\left(\frac1k\right) = 0,
\end{align}
which contradicts Theorem~\ref{thm:Separation}.
\end{proof}

\subsection{Proof of Corollary~\ref{cor:RandomDesignConstitent}}\label{subsec:Pf:cor:RandomDesignConstitent}
\begin{proof}
Consider $\varepsilon=j^{-1}$, where $j\in\bN$. Let $\Gamma^j_1=G_{p^j_1,q^j_1},\dots,\Gamma^j_{N^j}=G_{p^j_{N^j},q^j_{N^j}}$ be as listed in Theorem~\ref{thm:SuffFinite}. In view of Lemma~\ref{lem:Compact} and Lemma~\ref{lem:Cont}, there must be $\delta_j>0$ such that
\begin{align}
\min_{i=1,\dots,N^j}\inf_{(c,\mu)\in B_{\lambda,\frac1j}^\mathsf{c}} \Psi\left(\Gamma^j_i,(c,\mu),(c_0,\mu_0)\right) \ge \delta_j.
\end{align}
By Lemma~\ref{lem:Cont} again, there is a radius $r_j>0$ such that 
\begin{align}
\inf_{\Gamma\in B_{r_j}(\Gamma^j_i), i=1,\dots,N^j}\inf_{(c,\mu)\in B_{\lambda,\frac1j}^\mathsf{c}} \Psi\left(\Gamma,(c,\mu),(c_0,\mu_0)\right) \ge \frac{\delta_j}{2}.
\end{align}
Note that under the uniform distribution on $\cP(\bX)^\bA$, denoted by $m$, we have $m\big(B_{r_j}(\Gamma^j_i)\big)>0$ for $i=1,\dots,N^j$. By second Borel-Cantelli lemma, $B_{r_j}(\Gamma^j_n)$ will be visited infinitely often with probability 1. It follows that
\begin{align}
\bP\left(\limsup_{n\to\infty} \cE_n \le\frac1j\right) \ge \bP\left(B_{r_j}(\Gamma^j_i),\,i=1,\dots,N^j \text{ are visited at least once}\right) = 1.
\end{align}
Letting $j\to\infty$ and invoking the continuity of probability, we conclude the proof.
\end{proof}

\subsection{Proof of Theorem~\ref{thm:QuantDisEnv}}\label{subsec:Pf:thm:QuantDisEnv}
\begin{proof}
Recall the definition of $\tilde \mu$ in \eqref{eq:Defmutilde}. In this section, for the sake of neatness, we let
\begin{align*}
\wt H(z) := \tilde\mu([1-z,1]) = \int^1_{1-z}\sigma_{\mu}(\alpha)\dif\alpha,\quad \wt H'(z):=\tilde\mu'([1-z,1])=\int^1_{1-z}\sigma_{\mu'}(\alpha)\dif\alpha,\quad z\in[0,1].
\end{align*}
Clearly, $\wt H(0)=0$ and $\wt H(1)=1$. In addition, $\frac{d}{d z}\wt H(z)=\sigma_\mu(1-z)$ and $\wt H$ is concave. The analogue is true for $\wt H'$. 

\textbf{(a)} Let $c,c',\eta$ be as defined in Theorem~\ref{thm:QuantDisEnv} (a).
In view of the continuity of $\wt H$, for $k=1,\dots,\eta$, we let
\begin{align}
p_k := \inf\left\{z\in[0,1]: \wt H(z) = c^k \left(1- k \frac{\varepsilon}{\eta}\right) \right\}.
\end{align}
Note that $p_0=1$. To find $p,q$, we consider the following iterative procedure. We start with $k=1$. If  
\begin{align}\label{eq:CCont}
\wt H'(p_k) < c'^k \left(1 + k\frac{\varepsilon}{\eta}\right),
\end{align}
we continue to the next $k$. Otherwise, we terminate with $\wt H'(p_k) \ge c'^k \left(1 + k\frac{\varepsilon}{\eta}\right)$, and set $p=p_{k-1}$, $q=p_k$. Given that this iterative procedure stops at a $k<\eta$, we must have
\begin{align}
&\rho_{\mu}\big(C(X^{\text{A}}_{p,q})\big) - \rho_{\mu}\big(C(X^{\text{B}}_{p,q})\big) = c \wt H(p_{k-1}) - \wt H(p_k) = c^k\frac{\varepsilon}{\eta}
\end{align}
and
\begin{align}
&\rho_{\mu'}\big(C'(X^{\text{B}}_{p,q})\big) - \rho_{\mu'}\big(C'(X^{\text{A}}_{p,q})\big) = \wt H'(p_k) - c' \wt H'(p_{k-1}) > c'^k\frac{\varepsilon}{\eta}
\end{align}
Consequently,
\begin{align*}
\Phi\big(G_{p,q}, (C,\mu), (C',\mu')\big) > (cc')^{\frac{k}{2}}\frac{\varepsilon}{\eta}
\end{align*}
and 
\begin{align*}
\Xi\big(G_{p,q}, (C,\mu), (C',\mu')\big) > \left(\frac{c'}{c}\right)^{\frac{k}{2}}\frac{\varepsilon}{2\eta}.
\end{align*}

To complete the proof, we show that the iterative procedure above must terminate before the $\eta$-th iterations. Without loss of generality, we suppose the iterative procedure will continue after $k=1$. Therefore,
\begin{align}\label{eq:HypoUBHprime}
\wt H'(p_1) < c' + c'\frac{\varepsilon}{\eta} \le c'+\frac{\varepsilon}{2}=\frac{1}{2}(c+c').
\end{align}
It follows from the concavity of $\wt H'$ that
\begin{align}\label{eq:LBSlopeC}
\frac{\wt H(p_k)-\wt H(p_{k-1})}{p_k-p_{k-1}} &= \frac{\wt H(p_{k-1})-\wt H(p_k)}{\wt H'(p_{k-1})-\wt H'(p_k)} \frac{\wt H'(p_k)-\wt H'(p_{k-1})}{p_k-p_{k-1}} \ge \frac{\wt H(p_{k-1})-\wt H(p_k)}{\wt H'(p_{k-1})-\wt H'(p_k)} \frac{\wt H'(p_k)-p_0}{p_k-p_0}\nonumber\\
&\ge \frac{\wt H(p_{k-1})-\wt H(p_k)}{\wt H'(p_{k-1})-\wt H'(p_k)} \frac{\wt H'(p_1)-1}{p_1-1} \ge \frac{\wt H(p_{k-1})-\wt H(p_k)}{\wt H'(p_{k-1})-\wt H'(p_k)} (1-c) .
\end{align}
Recall the continuation condition \eqref{eq:CCont}. If the procedure is not terminated at $k\le \eta$, we have
\begin{align}\label{eq:LBRatioC}
\frac{\wt H(p_{k-1})-\wt H(p_k)}{\wt H'(p_{k-1})-\wt H'(p_k)} &\ge \frac{c^{k-1}\left(1 - (k-1)\frac{\varepsilon}{\eta}\right) - c^k \left(1 - k\frac{\varepsilon}{\eta}\right)}{c'^{k-1}\left(1+(k-1)\frac{\varepsilon}{\eta}\right)  -  c'^{k} \left(1 + k\frac{\varepsilon}{\eta}\right)} = \frac{(1-c)c^{k-1}\left(1-(k-1)\frac{\varepsilon}{\eta}\right) + c^k\frac{\varepsilon}{\eta}}{(1-c')c'^{k-1}\left(1+(k-1)\frac\varepsilon\eta\right)-c'^k\frac{\varepsilon}{\eta}} \nonumber\\
&> \frac{1-c}{1-c'}\frac{1-\varepsilon}{1+\varepsilon} \left(\frac{c}{c'}\right)^{k-1}.
\end{align}
In view of \eqref{eq:HypoUBHprime}, \eqref{eq:LBSlopeC} and \eqref{eq:LBRatioC}, since the iterative procedure is not terminated at $\eta$, we must have
\begin{align}
\frac{\wt H(p_k)-\wt H(p_{k-1})}{p_k-p_{k-1}} > \frac{(1-c)^2}{1-c'}\frac{1-\varepsilon}{1+\varepsilon} \left(\frac{c}{c'}\right)^{\eta-1}  \ge \lambda,
\end{align}
which contradict the hypothesis that $\lambda_\mu,\lambda_{\mu'}\le\lambda$. The proof for $C\neq C'$ is readily complete.

\textbf{(b)} Consider the setting of Theorem~\ref{thm:QuantDisEnv} (b). The structure of the proof is similar to part (a) but requires more delicate analysis. We conduct the proof in three steps. In Step 1, we present a procedure for constructing the distinguishing binary-choice question. In Step 2, we show the validity of this construction procedure. In Step 3, we combine the previous results to eventually obtain the claimed lower bounds.

\textbf{Step 1.} We consider the following iterative procedure.  
\begin{itemize}
\item Let $\tilde\eta\ge 2$. The specific value of $\tilde\eta$ will be specified in \eqref{eq:Defetatilde1} and \eqref{eq:Defetatilde2}, depending on specific cases. 
\item Let $p^\diamond_0=1$. In view of the continuity of $\wt H$, for $k\in\bN$, we let 
\begin{align}\label{eq:DefpDiamond}
p^\diamond_k := \inf\left\{z\in[0,1]: \wt H(z) = c^k \left(1 + k\frac{\varepsilon}{\tilde\eta} \right) \right\}.
\end{align}
Additionally, if 
\begin{align}\label{eq:DiamondCont}
\wt H'(p^\diamond_k) > c^k \left(1 - k\frac{\varepsilon}{\tilde\eta} \right),
\end{align}
we proceed to $k+1$. Otherwise, we stop with $\wt H'(p^\diamond_k) \le c^k \left(1 - k\frac{\varepsilon}{\tilde\eta} \right)$, 
and set $p=p^\diamond_{k-1}$ and $q=p^\diamond_{k}$ as in $G_{p,q}$. In view of \eqref{eq:Exprrho}, at stopping, under $(C,\mu)$, we have 
\begin{align*}
\rho_{\mu}\big( C(X^{\text{B}}_{p,q}) \big) - \rho_{\mu}\big( C(X^{\text{A}}_{p,q}) \big) = \wt H(p^\diamond_k) - c \wt H(p^\diamond_{k-1}) = c^k \frac{\varepsilon}{\tilde\eta}.
\end{align*}
In contrast, under $(C,\mu')$
\begin{align*}
\rho_{\mu'}\big( C(X^{\text{A}}_{p,q}) \big) - \rho_{\mu'}\big( C(X^{\text{B}}_{p,q}) \big) = c\wt H'(p^\diamond_{k-1}) - \wt H(p^\diamond_k) > c^k\frac{\varepsilon}{\tilde\eta}.
\end{align*}
Consequently, 
\begin{align}\label{eq:LBDiamondTer}
\Psi\big(G_{p,q},(C,\mu),(C,\mu')\big) > c^k\frac{\varepsilon}{\tilde\eta}.
\end{align}
\item We introduce one more stopping rule in addition to the previous one. In view of the continuity of $\wt H,\wt H'$, we let 
\begin{align}\label{eq:DefpStar0}
p^*_0 := \sup\left\{z\in[0,1]: |\wt H(z)-\wt H'(z)|=\varepsilon  \right\}.
\end{align}
Without loss of generality, we suppose 
\begin{align}\label{eq:DefbetaStar}
\wt H(p^*_0) > \wt H'(p^*_0).
\end{align}
Clearly, $\hat\beta^*_0=\check\beta^*_0+\varepsilon$. For $k\in\bN$, we define 
\begin{align}\label{eq:DefpStar}
p^*_k:=\inf\left\{z\in[0,1]: \wt H(z)=c^k \left( \wt H(p^*_0)  - k\frac{\varepsilon}{\tilde\eta} \right) \right\}.
\end{align}
Additionally, if 
\begin{align}\label{eq:StarCont}
\wt H'(p^*_k) < c^k\left(\wt H'(p^*_0) +k\frac{\varepsilon}{\tilde\eta}\right)
\end{align}
we continue to $k+1$. Otherwise, we stop with $\wt H'(p^*_k) \ge c^k\left(\wt H'(p^*_0)+k\frac{\varepsilon}{\tilde\eta}\right)$, and set $p=p^*_{k-1}$ and $q=p^*_k$ as in $G_{p,q}$. In view of \eqref{eq:Exprrho}, at stopping, under $(C,\mu)$, we have 
\begin{align*}
\rho_{\mu}\big( C(X^{\text{A}}_{p,q}) \big) - \rho_{\mu}\big( C(X^{\text{B}}_{p,q}) \big) = c\wt H(p^*_{k-1}) - \wt H(p^*_k) = c^k \frac{\varepsilon}{\tilde\eta}.
\end{align*}
In contrast, under $(C,\mu')$,
\begin{align*}
\rho_{\mu'}\big( C(X^\text{B}_{p,q}) \big) - \rho_{\mu'}\big( C(X^\text{A}_{p,q}) \big) = \wt H'(p^*_{k}) - c\wt H'(p^*_{k-1}) > c^k\frac{\varepsilon}{\tilde\eta}.
\end{align*}
Consequently, 
\begin{align}\label{eq:LBDStarTer}
\Psi\big(G_{p,q},(C,\mu),(C,\mu')\big) > c^k\frac{\varepsilon}{\tilde\eta}.
\end{align}
\end{itemize}

\textbf{Step 2.} We will show that the iterative procedure in Step $1$ must terminate within finitely many iterations, and obtain the lower bound in accordance to \eqref{eq:LBDiamondTer} or \eqref{eq:LBDStarTer}. Recall $p^\diamond_0=1$. Let $n_0\in\bN\cup\set{0}$ be the integer such that 
\begin{align}\label{eq:Defn0}
p^\diamond_{n_0+1} < p^*_0 < p^\diamond_{n_0}.
\end{align}
Note that if the inequality is not strict, by \eqref{eq:DefpDiamond} and \eqref{eq:DiamondCont}, we have immediate termination. Below we conduct the proof for two cases.

\emph{Step 2, Case (i).} We first investigate the case where $\varepsilon\ge 2(1-c)$. In this case, we set 
\begin{align}\label{eq:Defetatilde1}
\tilde\eta := 4(n_0+1).
\end{align}
Note that $\varepsilon\ge 2c^{n_0}(1-c)$ and $\frac{c^{n_0}(1+c)}{4}<\frac12$ as $c\in[0,1]$. Consequently,
$$\left(1-\frac{c^{n_0}(1+c)}{4}\right)\varepsilon>\frac12\varepsilon\ge c^{n_0}(1-c),$$ 
and thus
\begin{align}\label{eq:UBdiffcepsilon}
c^{n_0}\left(1+\frac{\varepsilon}{4}\right) - c^{n_0+1} \left(1-\frac{\varepsilon}{4}\right) = c^{n_0}\left(1-c\right) + c^{n_0} (1+c)\frac{\varepsilon}{4} < \varepsilon.
\end{align} 
Given that the iterative procedure is not terminated at $p^\diamond_{n_0}$, by applying the monotonicity of $\wt H'$ together with \eqref{eq:Defn0}, \eqref{eq:DefpStar0}, the monotonicity of $\wt H$ together with \eqref{eq:Defn0}, \eqref{eq:DefpDiamond}, \eqref{eq:UBdiffcepsilon}, and \eqref{eq:Defetatilde1}, we yield
\begin{align}\label{eq:BUwtHprime}
\wt H'(p^\diamond_{n_0+1}) &\le \wt H'(p^*_0) = \wt H(p^*_0) - \varepsilon \le \wt H(p^\diamond_{n_0}) - \varepsilon \le c^{n_0}\left(1+\frac{\varepsilon}{4}\right) - \varepsilon < c^{n_0+1} \left(1-\frac{\varepsilon}{4}\right) < c^{n_0+1}\left(1-(n_0+1)\frac{\varepsilon}{\tilde\eta}\right).
\end{align} 
In view of the continuation condition \eqref{eq:DiamondCont}, the construction procedure must terminate at iteration $n_0+1$. We thus set $p=p^\diamond_{n_0}$ and $q=p^{\diamond}_{n_0+1}$. Since $\wt H'(p^\diamond_{n_0+1})\ge 0$, following \eqref{eq:BUwtHprime}, we also have 
\begin{align}
\frac{5}{4}c^{n_0}-\varepsilon\ge c^{n_0}\left(1+\frac{\varepsilon}{4}\right) - \varepsilon\ge 0\quad\text{and}\quad n_0 \le \frac{\ln\varepsilon-\ln(1+\frac{\varepsilon}{4})}{\ln c}
\end{align} 
In view of \eqref{eq:LBDiamondTer}, we yield
\begin{align}
\Psi\big(G_{p,q},(C,\mu),(C,\mu')\big) > c^{n_0+1} \frac{\varepsilon}{4(n_0+1)} \ge \frac{c\varepsilon^2 \ln c}{5 \ln c + 5\left(\ln\varepsilon-\ln(1+\frac{\varepsilon}4)\right)}.
\end{align}

\emph{Step 2, Case (ii).} Next, we suppose $\varepsilon< 2(1-c)$. In fact, we suppose $\varepsilon< 2c^{n_0}(1-c)$, otherwise the exact same reasoning as in Case 1 shows that termination must occur at iteration $n_0+1$. In this case, we define 
\begin{align}\label{eq:Defetatilde2}
\tilde\eta := 4(n_0+1) + 8 \frac{\ln\lambda-\ln\varepsilon}{\ln\left(1+\frac{\varepsilon}{c^{n_0}(1-c)}\right)}. 
\end{align}
We will show that the iterative procedure will terminate after $n_0+1+k^\dagger$ iterations, where $k^\dagger$ satisfies
\begin{align}\label{eq:UBkDagger}
k^\dagger \le  \frac{\ln\lambda-\ln\varepsilon}{\ln\left(1+\frac{\varepsilon}{c^{n_0}(1-c)}\right)}.
\end{align}
Clearly, for $k\le k^\dagger$, we have 
\begin{align}\label{eq:UBn0plus2k}
\frac{n_0+1+2k}{\tilde\eta} \le \frac14.
\end{align}
In addition, due to the construction above and the monotonicity of $\wt H$, we have 
\begin{align}\label{eq:Bddp}
p^*_{k} < p^\diamond_{n_0+k} < p^*_{k-1} < p^\diamond_{n_0+k-1}, \quad 1\le k\le k^\dagger.
\end{align}

Suppose at some $k\in[1,\frac{\eta}{4}]$, neither termination condition is met. Note that 
\begin{align}\label{eq:IdenSlope}
\frac{\wt H'(p^*_{k}) - \wt H'(p^\diamond_{n_0+k})}{p^*_{k}-p^\diamond_{n_0+k}} = \frac{\wt H'(p^*_{k}) - \wt H'(p^\diamond_{n_0+k})}{\wt H(p^*_{k}) - \wt H(p^\diamond_{n_0+k})} \frac{\wt H(p^*_{k}) - \wt H(p^\diamond_{n_0+k})}{p^*_{k}-p^\diamond_{n_0+k}}.
\end{align}
For the first factor on the right hand side of \eqref{eq:IdenSlope}, by \eqref{eq:DefpDiamond}, \eqref{eq:DiamondCont}, \eqref{eq:DefpStar}, and \eqref{eq:StarCont}, we have
\begin{align}\label{eq:LBCase2Ratio}
&\frac{\wt H'(p^*_{k}) - \wt H'(p^\diamond_{n_0+k})}{\wt H(p^*_{k}) - \wt H(p^\diamond_{n_0+k})} 
>  \frac{c^{k} \left(c^{n_0} - (n_0+k)\frac{\varepsilon}{\tilde\eta} \right) - c^{k} \left( \wt H'(p^*_0) + k\frac{\varepsilon}{\tilde\eta} \right)}{c^{k} \left(c^{n_0} + (n_0+k)\frac{\varepsilon}{\tilde\eta} \right) - c^{k} \left( \wt H(p^*_0) - k\frac{\varepsilon}{\tilde\eta} \right)} = \frac{c^{n_0}-\wt H'(p^*_0)-\frac{n_0+2k}{\tilde\eta}\varepsilon}{c^{n_0}-\wt H'(p^*_0)-\varepsilon+\frac{n_0+2k}{\tilde\eta}\varepsilon}.
\end{align}
For the second factor on the right hand side of \eqref{eq:IdenSlope}, in view of \eqref{eq:Bddp}, by the concavity of $\wt H$, we have 
\begin{align}\label{eq:LBCase2Concave}
\frac{\wt H(p^*_{k})-\wt H(p^\diamond_{n_0+k})}{p^*_{k}-p^\diamond_{n_0+k}} \ge \frac{\wt H(p^\diamond_{n_0+k}) - \wt H(p^*_{k-1})}{p^\diamond_{n_0+k}-p^*_{n_0+k-1}} = \frac{\wt H(p^\diamond_{n_0+k}) - \wt H(p^*_{k-1})}{\wt H'(p^\diamond_{n_0+k}) - \wt H'(p^*_{k-1})} \frac{\wt H'(p^\diamond_{n_0+k}) - \wt H'(p^*_{k-1})}{p^\diamond_{n_0+k}-p^*_{k-1}}.
\end{align}
For the first factor on the right hand side of \eqref{eq:LBCase2Concave}, it follows from \eqref{eq:DefpDiamond}, \eqref{eq:DiamondCont}, \eqref{eq:DefpStar}, and \eqref{eq:StarCont} that
\begin{align}\label{eq:LBCase2Ratio2}
\frac{\wt H(p^\diamond_{n_0+k}) - \wt H(p^*_{k-1})}{\wt H'(p^\diamond_{n_0+k}) - \wt H'(p^*_{k-1})} 
&> \frac{c^{k-1}\left(\wt H(p^*_0)-(k-1)\frac{\varepsilon}{\tilde\eta}\right) - c^k \left(c^{n_0} + (n_0+k)\frac{\varepsilon}{\tilde\eta}\right)}{c^{k-1}\left(\wt H'(p^*_0)+(k-1)\frac{\varepsilon}{\tilde\eta}\right) - c^k \left(c^{n_0} - (n_0+k)\frac{\varepsilon}{\tilde\eta} \right)} \nonumber\\
&= \frac{\wt H'(p^*_0) - c^{n_0+1} + \left(1-\frac{(k-1)+c(n_0+k)}{\tilde\eta}\right) \varepsilon}{\wt H'(p^*_0) - c^{n_0+1} + \frac{(k-1)+c(n_0+k)}{\tilde\eta}\varepsilon}.
\end{align}
For the second factor on the right hand side of \eqref{eq:LBCase2Concave}, by the monotonicity and concavity of $\wt H'$, we have
\begin{align}\label{eq:LBCase2Concave2}
\frac{\wt H'(p^\diamond_{n_0+k}) - \wt H'(p^*_{k-1})}{p^\diamond_{n_0+k}-p^*_{k-1}} \ge \frac{\wt H'(p^*_{k-1}) - \wt H'(p^\diamond_{n_0+k-1})}{p^*_{k-1}-p^\diamond_{n_0+k-1}}.
\end{align}
Putting together \eqref{eq:IdenSlope}, \eqref{eq:LBCase2Ratio}, \eqref{eq:LBCase2Concave}, \eqref{eq:LBCase2Ratio2}, and \eqref{eq:LBCase2Concave2}, we yield
\begin{align*}
&\frac{\wt H'(p^*_{k}) - \wt H'(p^\diamond_{n_0+k})}{\wt H(p^*_{k}) - \wt H(p^\diamond_{n_0+k})} \\
&\quad \ge \frac{c^{n_0}-\wt H'(p^*_0)-\frac{n_0+2k}{\tilde\eta}\varepsilon}{c^{n_0}-\wt H'(p^*_0)-\varepsilon+\frac{n_0+2k}{\tilde\eta}\varepsilon} \frac{\wt H'(p^*_0) - c^{n_0+1} + \left(1-\frac{(k-1)+c(n_0+k)}{\tilde\eta}\right) \varepsilon}{\wt H'(p^*_0) - c^{n_0+1} + \frac{(k-1)+c(n_0+k)}{\tilde\eta}\varepsilon} \frac{\wt H'(p^*_{k-1}) - \wt H'(p^\diamond_{n_0+k-1})}{p^*_{k-1}-p^\diamond_{n_0+k-1}}\\
&\quad \ge
\left(1+\frac{\frac12\varepsilon}{c^{n_0}-\wt H'(p^*_0)-\frac{3}{4}\varepsilon} \right) \left(1+\frac{\frac12\varepsilon}{\wt H'(p^*_0)-c^{n_0+1}+\frac1{4}\varepsilon}\right) \frac{\wt H'(p^*_{k-1}) - \wt H'(p^\diamond_{k-1})}{p^*_{k-1}-p^\diamond_{k-1}},
\end{align*}
where we have used \eqref{eq:UBn0plus2k} in the last inequality. Recall that $c^{n_0+1}\le \wt H(p^*_0) \le c^{n_0}$. By taking the maximum of the factors on the left hand side below, we obtain
$$\left(1+\frac{\frac12\varepsilon}{c^{n_0}-\wt H'(p^*_0)-\frac{3}{4}\varepsilon} \right) \left(1+\frac{\frac12\varepsilon}{\wt H'(p^*_0)-c^{n_0+1}+\frac1{4}\varepsilon}\right) \ge 1 + \frac{\varepsilon}{c^{n_0}-c^{n_0+1}-\frac12\varepsilon} \ge 1+\frac{\varepsilon}{c^{n_0}(1-c)},$$
where we note the first equality is attained when the two factors coincides. With induction, we obtain
\begin{align}
\frac{\wt H'(p^*_{k}) - \wt H'(p^\diamond_{n_0+k})}{\wt H(p^*_{k}) - \wt H(p^\diamond_{n_0+k})} \ge \left(1+\frac{\varepsilon}{c^{n_0}(1-c)}\right)^{k}\frac{\wt H'(p^*_0)-c^{n_0}}{p^*_0-p^\diamond_{n_0}} \ge \left(1+\frac{\varepsilon}{c^{n_0}(1-c)}\right)^{k}\varepsilon,
\end{align}
where, in the last inequality, we have used the fact below due to the concavity of $\wt H$ and \eqref{eq:Defn0},
\begin{align*}
\frac{c^{n_0}-\wt H'(p^*_{k})}{p^\diamond_{n_0}-p^*_0} \ge \frac{1-\wt H'(p^*_{k})}{1-p^*_0} \ge \varepsilon.
\end{align*}  
Since $\sigma_{\mu'}\le\lambda$, a termination must occur after a finite number of iterations. By solving for the maximum of such $k$, we validate \eqref{eq:UBkDagger}. 

\textbf{Step 3.} Lastly, we simplify the involved terms to obtain the claimed lower bounds. By combining \eqref{eq:LBDStarTer} and \eqref{eq:UBkDagger}, there exists a distinguishing question $G$ such that
\begin{align}\label{eq:UBPsi}
\Psi\big(G,(C,\mu),(C,\mu')\big) > c^{n^0+1+k^\dagger} \frac{\varepsilon}{\tilde\eta} > c^{n_0+1+\frac{\ln\lambda-\ln\varepsilon}{\ln\left(1+\frac{\varepsilon}{c^{n_0}(1-c)}\right)}} \varepsilon \Bigg/\left(4(n_0+1)+8\frac{\ln\lambda-\ln\varepsilon}{\ln\left(1+\frac{\varepsilon}{c^{n_0}(1-c)}\right)}\right). 
\end{align}
We proceed to lower bound the right hand side with simpler terms. First note that $\varepsilon<2c^{n^0}(1-c)$, thus $c^{n_0} > \frac{\varepsilon}{2(1-c)}$ and $n_0<\frac{\ln\varepsilon-\ln(1-c)-\ln 2}{\ln c}$. In addition, since $\ln(1+x)>\frac{x}{1+x}$ for $x>0$, we yield $\ln\left(1+\frac{\varepsilon}{c^{n_0}(1-c)}\right) > \frac{\varepsilon}{c^{n_0}(1-c)+\varepsilon} > \frac{\varepsilon}{1-c+\varepsilon}$.
By combining the above, we obtain \eqref{eq:LBPhiSameCDiffmu}. Regarding the relative distinguishing power, in view of \eqref{eq:DefpDiamond}, \eqref{eq:DiamondCont}, \eqref{eq:DefpStar}, \eqref{eq:DiamondCont}, \eqref{eq:UBn0plus2k} and \eqref{eq:Bddp}, at termination, the value functions involves in $G$ is upper bounded by $\frac{5}{4}c^{n_0+k^\dagger}$, and \eqref{eq:LBXiSameCDiffmu} follows with similar calculations as before.
\end{proof}

\subsection{Proof of Proposition~\ref{prop:g}}\label{subsec:Pf:prop:g}
\begin{proof}[Proof of \eqref{eq:UBAbsDiffc}]
We claim that
\begin{align}\label{eq:BBgoverp}
\left(1-\frac{\lambda' p}{2\lambda}\right)c \le \frac1p g_{c,\mu}(p) \le c, \quad p\in[0,1).
\end{align} 
Indeed, for the first inequality, by hypothesis and monotonicity of $\sigma_\mu$ in Lemma~\ref{lem:Spectral}, we have 
\begin{align*}
c\left(\sigma_\mu(1) p - \frac12\lambda' p^2\right) \le c\int_{1-p}^1\sigma_\mu(\kappa)\dif\kappa = \int_{1-g_{c,\mu}(p)}^1\sigma_\mu(\kappa)\dif\kappa \le \sigma_\mu(1) g_{c,\mu}(p).
\end{align*}
Note that we must have $\sigma_\mu(1)\ge 1$, otherwise $\sigma_\mu$ would not be a density function, as alluded by Lemma~\ref{lem:Spectral}. We obtain the first inequality by dividing both hand sides with $\sigma_\mu(1)p$. The second inequality is an immediate consequence of Lemma~\ref{lem:g} (e).

To finish the proof, we observe that, by triangle inequality,
\begin{align}
|c_0-\hat c| &\le \left| c_0-\frac{1}{\sqrt{\varepsilon}}g_{c_0,\mu_0}(\sqrt{\varepsilon}) \right| + \frac{1}{\sqrt{\varepsilon}} \left| g_{c_0,\mu_0}(\sqrt{\varepsilon}) - g_{\hat c,\hat\mu}(\sqrt{\varepsilon}) \right| + \left| \frac{1}{\sqrt{\varepsilon}}g_{\hat c,\hat\mu}(\sqrt{\varepsilon}) - \hat c \right|.
\end{align}
Combining this with \eqref{eq:BBgoverp} and the setting that $c,c'\in[0,1]$, we conclude the proof.
\end{proof}

\begin{proof}[Proof of \eqref{eq:UBAbsDiffmu}]
For rotational convenience, we introduce $g^{\circ 0}_{c,\mu}$ as the identity function, and $g^{\circ j}_{c,\mu}:=g^{\circ j-1}_{c,\mu}\circ g_{c,\mu}$ for $j\in\bN$, i.e., $g^{\circ j}_{c,\mu}$ is the $j$-fold composition of $g_{c,\mu}$. 

We first establish a useful identity, namely \eqref{eq:IDIntsigma} below. By Lemma~\ref{lem:g} (f), we have
\begin{align}
\int_{1-p}^1 \sigma_{\mu}(\kappa)\dif\kappa = \int_{1-p}^1 \sigma_{\mu}(1-(1-\kappa))\dif\kappa = c^{-1} \int_{1-p}^1 \sigma_{\mu}\big(1-g_{c,\mu}(1-\kappa)\big)\dot g_{c,\mu}(1-\kappa)\dif\kappa.
\end{align}
Next, recall the definition of $\underline p_\mu$ in \eqref{eq:Defrmu}. Observe that $\dot g_{c,\mu}(p)\ge c\frac{\sigma_{\mu}(1-p)}{\sigma_\mu(1)}>0$ for $p\in[0,\underline p_{\mu})$, due to Lemma~\ref{lem:g} (f) and the monotonicity of $\sigma_\mu$ in Lemma~\ref{lem:Spectral}. An application of the change of variable formula shows that the set below has $0$ Lebesgue measure
\begin{align}
&\left\{\kappa\in[0,\underline p_{\mu}): \sigma_\mu\big(1-g_{c,\mu}(1-\kappa)\big) \neq \sigma_\mu\big(1-g_{c,\mu}\circ g_{c,\mu}(1-\kappa)\big) \dot g_{c,\mu}\big(g_{c,\mu}(1-\kappa)\big) \right\}.
\end{align}
Meanwhile, for $\kappa\in(\underline p_{\mu},1]$, by Lemma~\ref{lem:g} (c), we have $\dot g_{c,\mu}(\kappa)=0$. It follows from Lemma~\ref{lem:g} (f) again that 
\begin{align}
\int_{1-p}^1 \sigma_{\mu}(\kappa)\dif\kappa = c^{-2} \int_{1-p}^1 \sigma_{\mu}\big(1-g^{\circ 2}_{c,\mu}(1-\kappa)\big) \dot g_{c,\mu}\big(g^{\circ 1}(1-\kappa)\big) \dot g_{c,\mu}(1-\kappa)\dif\kappa
\end{align}
By induction, for $J\in\bN$, we have 
\begin{align}
\int_{1-p}^1 \sigma_{\mu}(\kappa)\dif\kappa &= c^{-J} \int_{1-p}^1 \sigma_{\mu}\big(1-g^{\circ J}_{c,\mu}(1-\kappa)\big) \prod_{j=0}^{J-1} \dot g_{c,\mu}\big(g^{\circ j}_{c,\mu}(1-\kappa)\big) \dif\kappa \\
&= c^{-J} \int_{0}^p \sigma_{\mu}\big(1-g^{\circ J}_{c,\mu}(\kappa)\big) \prod_{j=0}^{J-1} \dot g_{c,\mu}\big(g^{\circ j}_{c,\mu}(\kappa)\big) \dif\kappa. \label{eq:IDIntsigma}
\end{align}
The analogue is also true for $(c',\mu')$. 

Substituting in \eqref{eq:IDIntsigma}, expanding with telescoping sum, then invoking the fact that $|\dot g| \le c$ and $\sigma_\mu(1)\le\lambda$ due to Lemma~\ref{lem:g} and Assumption~\ref{assump:sigmaBound}, respectively, we yield
\begin{align}
&\left|\int_{1-p}^1 \big(\sigma_{\mu'}(\kappa)-\sigma_{\mu}(\kappa)\big) \dif\kappa\right|\\
&\quad\le \left| \int_{0}^p \Big( c'^{-J}\sigma_{\mu'}\big(1-g^{\circ J}_{c',\mu'}(\kappa)\big) - c^{-J}\sigma_{\mu}\big(1-g^{\circ J}_{c,\mu}(\kappa)\big) \Big) \prod_{j=0}^{J-1} \dot g_{c',\mu'}\big(g^{\circ j}_{c',\mu'}(\kappa)\big) \dif\kappa \right| \\
&\qquad \begin{multlined}
+ \sum_{j=0}^{J-1} \left| \int_{0}^p  c^{-J}\sigma_{\mu}\big(1-g^{\circ J}_{c,\mu}(\kappa)\big) \left(\prod_{i=0}^{j-1}\dot g_{c',\mu'}\big(g^{\circ i}_{c',\mu'}(\kappa)\big)\right) \Big( \dot g_{c',\mu'}\big(g^{\circ j}_{c',\mu'}(\kappa)\big) - \dot g_{c,\mu}\big(g^{\circ j}_{c,\mu}(\kappa)\big) \Big) \right.\\ \left. \left(\prod_{i=j+1}^{J-1}\dot g_{c,\mu}\big(g^{\circ i}_{c,\mu}(\kappa)\big)\right) \dif\kappa \right|\end{multlined}\\
&\quad\le c'^J \left| \int_{0}^p \Big( c'^{-J}\sigma_{\mu'}\big(1-g^{\circ J}_{c',\mu'}(\kappa)\big) - c^{-J}\sigma_{\mu}\big(1-g^{\circ J}_{c,\mu}(\kappa)\big) \Big) \big) \dif\kappa \right| \\
&\qquad \begin{multlined}
+ \lambda \sum_{j=0}^{J-1} c^{-J}\left| \int_{0}^p \left(\prod_{i=0}^{j-1}\dot g_{c',\mu'}\big(g^{\circ i}_{c',\mu'}(\kappa)\big)\right) \Big( \dot g_{c',\mu'}\big(g^{\circ j}_{c',\mu'}(\kappa)\big) - \dot g_{c,\mu}\big(g^{\circ j}_{c,\mu}(\kappa)\big) \Big) \right.\\ \left. \left(\prod_{i=j+1}^{J-1}\dot g_{c,\mu}\big(g^{\circ i}_{c,\mu}(\kappa)\big)\right) \dif\kappa \right|\end{multlined}\\
&\quad=: I_1(p) + \lambda \sum_{j=0}^{J-1}I^j_2(p),  \label{eq:UBAbsDiffIntsigma}
\end{align} 
where, for notation convenience, we have set $\prod_{i=0}^{-1}=1$ and $\prod_{i=J}^{J-1}=1$.  

Regarding $I_1$, we first observe that, by Lemma~\ref{lem:g} (e),
\begin{align}\label{eq:UBgjkappa}
g^{\circ j}_{c,\mu}(\kappa) \le c^{j}, \quad \kappa\in[0,1],\, j\in\bN. 
\end{align}
This together with  Assumption~\ref{assump:sigmaLip} implies
\begin{align}
I_1(p) 
& \le c'^{J} \bigg( \big| c'^{-J}\sigma_{\mu'}(1) - c^{-J}\sigma_{\mu}(1) \big|p + \frac12 \lambda' p^2 \bigg), \label{eq:UBI2temp}
\end{align}
where the presence of $\frac12\lambda'p^2$ rather than $\lambda'p^2$ arises from the monotonicity of $\sigma_\mu$ in Lemma~\ref{lem:Spectral}. In addition, by \eqref{eq:IDIntsigma} with $p=1$ (note $\int_0^1\sigma_\mu=1$ due to Lemma~\ref{lem:Spectral}), \eqref{eq:UBgjkappa}, and Assumption~\ref{assump:sigmaLip}, we have 
\begin{align}\label{eq:UBAbsErrsigma1}
\left| \frac{\sigma_{\mu}(1)}{c^J} \!-\! \frac{1}{\int_0^1 \prod_{j=0}^{J-1} \dot g_{c,\mu}\big(g^{\circ j}_{c,\mu}(\kappa)\big)\dif\kappa} \right| = \left|  \frac{\int_0^1 \left(\sigma_{\mu}(1)-\sigma_{\mu}\big(1-g^{\circ J}_{c,\mu}(\kappa)\big)\right) \prod_{j=0}^{J-1} \dot g_{c,\mu}\big(g^{\circ j}_{c,\mu}(\kappa)\big) \dif\kappa}{c^J\int_0^1\prod_{j=0}^{J-1} \dot g_{c,\mu}\big(g^{\circ j}_{c,\mu}(\kappa)\big)\dif\kappa} \right| \le \lambda'.
\end{align}
Note that \eqref{eq:UBAbsErrsigma1} together with $|\sigma_\mu(1)|\le\lambda$ also implies $\left(\int_0^1 \prod_{j=0}^{J-1} \dot g_{c,\mu}\big(g^{\circ j}_{c,\mu}(\kappa)\big)\dif\kappa\right)^{-1}\le \lambda c^{-J}+\lambda'$. It follows that 
\begin{align}
&\left| \frac{\sigma_{\mu'}(1)}{c'^J} - \frac{\sigma_{\mu}(1)}{c^J} \right|\\
&\quad\le 2\lambda' +  \frac{\left| \int_0^1\prod_{j=0}^{J-1} \dot g_{c,\mu}\big(g^{\circ j}_{c,\mu}(\kappa)\big)\dif\kappa - \int_0^1\prod_{j=0}^{J-1} \dot g_{c',\mu'}\big(g^{\circ j}_{c',\mu'}(\kappa)\big)\dif\kappa \right|}{\left(\int_0^1\prod_{j=0}^{J-1} \dot g_{c,\mu}\big(g^{\circ j}_{c,\mu}(\kappa)\big)\dif\kappa\right)\left(\int_0^1\prod_{j=0}^{J-1} \dot g_{c',\mu'}\big(g^{\circ j}_{c',\mu'}(\kappa)\big)\dif\kappa\right)}\\
&\quad\le 2\lambda' + (\lambda c^{-J}\!\!+\lambda')(\lambda c'^{-J}\!\!+\lambda') \left| \int_0^1\! \left(\prod_{j=0}^{J-1} \dot g_{c,\mu}\big(g^{\circ j}_{c,\mu}(\kappa)\big) \!-\! \prod_{j=0}^{J-1} \dot g_{c',\mu'}\big(g^{\circ j}_{c',\mu'}(\kappa)\big)\!\right)\!\dif\kappa \right|. \label{eq:UBAbsDiffsigmacJ}
\end{align}
By combining \eqref{eq:UBI2temp} and \eqref{eq:UBAbsDiffsigmacJ}, with the same telescoping sum argument leading to \eqref{eq:UBAbsDiffIntsigma}, we yield
\begin{align}
I_1(p) \le \left( 2\lambda'c^J + (\lambda +\lambda'c^{J})^2  \sum_{j=0}^{J-1} I^j_2(1) \right) p + \frac12\lambda'c'^J p^2.
\end{align}
Furthermore, by combining the above with inequality \eqref{eq:UBAbsDiffIntsigma} with the hypothesis that $c\ge c'$, we obtain
\begin{align}
\sup_{p\in[0,1]}\left|\int_{1-p}^1 \big(\sigma_{\mu'}(\kappa)-\sigma_{\mu}(\kappa)\big) \dif\kappa\right| &\le \sup_{p\in[0,1]} \left\{ \left( 2\lambda'c^J + (\lambda +\lambda'c^{J})^2  \sum_{j=0}^{J-1} I^j_2(1) \right) p + \frac12\lambda'c^J p^2 + \lambda \sum_{j=0}^{J-1}I^j_2(p) \right\} \\
&\le \frac52\lambda'c^J + \Big(\lambda+(\lambda +\lambda'c^{J})^2\Big) \sup_{p\in[0,1]} \sum_{j=0}^{J-1}I^j_2(p). \label{eq:UBAbsDiffIntsigma2}
\end{align}

It remains to bound $\sum_{j=0}^{J-1}I^j_2(p)$. Observe that by Lemma~\ref{lem:g} (f) and triangle inequality,
\begin{align}
I^j_2(p) &\le c^{-j-2} \left| \int_{0}^p \Big( \dot g_{c',\mu'}\big(g^{\circ j}_{c',\mu'}(\kappa)\big) - \dot g_{c,\mu}\big(g^{\circ j}_{c,\mu}(\kappa)\big) \Big)  \left(\prod_{i=0}^{j-1}\dot g_{c',\mu'}\big(g^{\circ i}_{c',\mu'}(\kappa)\big)\right) \dif\kappa \right|\\
&\le c^{-j-2} \left| \int_{0}^p \Big( \dot g_{c',\mu'}\big(g^{\circ j}_{c',\mu'}(\kappa)\big) - \dot g_{c,\mu}\big(g^{\circ j}_{c',\mu'}(\kappa)\big) \Big)  \left(\prod_{i=0}^{j-1}\dot g_{c',\mu'}\big(g^{\circ i}_{c',\mu'}(\kappa)\big)\right) \dif\kappa \right| \\
&\quad + c^{-1} \left| \int_{0}^p \Big( \dot g_{c,\mu}\big(g^{\circ j}_{c',\mu'}(\kappa)\big) - \dot g_{c,\mu}\big(g^{\circ j}_{c,\mu}(\kappa)\big) \Big)  \dif\kappa \right|\\
&=: \hat I^j_2(p) + \check I^j_2(p).\label{eq:UBI2jtemp}
\end{align}

Regarding $\hat I^j_2$, with multiple applications of change of variables, we obtain 
\begin{align}\label{eq:UBIhat2j}
\hat I^j_2(p) &= c^{-j-2} \left| \int_{0}^{g^{\circ j}(p)} \Big( \dot g_{c',\mu'}\big(\kappa\big) - \dot g_{c,\mu}\big(\kappa\big) \Big) \dif\kappa \right| \le c^{-j-2}\|g_{c',\mu'}-g_{c,\mu}\|_\infty.
\end{align}

As for $\check I_2^j$, we first observe the following, due to Lemma~\ref{lem:g} and Assumption~\ref{assump:sigmaLip}. That is, for any $p,p'\in(0,1)$, we have 
\begin{align}
&\left|\dot g_{c,\mu}(p) - \dot g_{c,\mu}(p')\right| =  c\left| \frac{\sigma_\mu(1-p)}{\sigma_\mu\big(1-g_{c,\mu}(p)\big)} - \frac{\sigma_\mu(1-p')}{\sigma_\mu\big(1-g_{c,\mu}(p')\big)} \right|\\
&\quad= c\left| \frac{\Big(\sigma_\mu(1-p)-\sigma_\mu(1-p')\Big)\sigma_\mu\big(1-g_{c,\mu}(p')\big) + \sigma_\mu(1-p')\Big(\sigma_\mu\big(1-g_{c,\mu}(p')\big)-\sigma_\mu\big(1-g_{c,\mu}(p)\big)\Big)}{\sigma_\mu\big(1-g_{c,\mu}(p)\big)\sigma_\mu\big(1-g_{c,\mu}(p')\big)}  \right| \\
&\quad\le \frac{c\lambda'|p-p'|}{\sigma_\mu\big(1-g_{c,\mu}(p)\big)} \le \frac{c\lambda'|p-p'|}{\sigma_\mu\big(1-g_{c,\mu}(1)\big)}, \label{eq:UBAbsDiffdg}
\end{align} 
where in the last inequality we have used the monotonicity of $\sigma_\mu$ and $g_{c,\mu}$ in Lemma~\ref{lem:Spectral} and Lemma~\ref{lem:g} (c), respectively. In addition, note that for $j>1$,
\begin{align}
\left\| g^{\circ j}_{c',\mu'} - g^{\circ j}_{c,\mu} \right\|_\infty &\le \left\| g_{c',\mu'}\circ g^{\circ j-1}_{c',\mu'} - g_{c',\mu'}\circ g^{\circ j-1}_{c,\mu} \right\|_\infty + \left\| g_{c',\mu'}\circ g^{\circ j-1}_{c,\mu} - g_{c,\mu}\circ g^{\circ j-1}_{c,\mu} \right\|_\infty\\
&\le c'\left\| g^{\circ j-1}_{c',\mu'} - g^{\circ j-1}_{c,\mu} \right\|_\infty + \left\| g_{c',\mu'} - g_{c,\mu} \right\|_\infty\\
&\le \sum_{i=0}^j c'^i \left\| g_{c',\mu'} - g_{c,\mu} \right\|_\infty  \le \frac{1}{1-c}\left\| g_{c',\mu'} - g_{c,\mu} \right\|_\infty, \label{eq:UBIcheck2j}
\end{align}
where the third inequality is a consequence of induction. By \eqref{eq:UBIhat2j}, \eqref{eq:UBAbsDiffdg}, and \eqref{eq:UBIcheck2j}, we have
\begin{align}\label{eq:UBIhat2j2}
\hat I^j_2(p) \le \frac{\lambda' \left\| g_{c',\mu'} - g_{c,\mu} \right\|_\infty}{(1-c)\sigma_\mu\big(1-g_{c,\mu}(1)\big)} p.
\end{align}
Finally, let $J$ be the largest integer such that $J\le\frac{\ln\left\| g_{c',\mu'} - g_{c,\mu} \right\|_\infty}{2\ln c}$. Putting together \eqref{eq:UBI2jtemp}, \eqref{eq:UBIhat2j}, and \eqref{eq:UBIhat2j2}, we yield that for any $p\in[0,1]$,
\begin{align}
\sum_{j=0}^{J-1} I^j_2(p) &\le \left(c^{-2}\frac{c^{-J}-1}{c^{-1}-1}+\frac{\lambda'J}{(1-c)\sigma_\mu\big(1-g_{c,\mu}(1)\big)}\right) \left\| g_{c',\mu'} - g_{c,\mu} \right\|_\infty \\
&\le \frac{1}{(1-c)c}\sqrt{\left\| g_{c',\mu'} - g_{c,\mu} \right\|_\infty} + \frac{\lambda'}{2\sigma_\mu\big(1-g_{c,\mu}(1)\big) (1-c) \ln c} \left\| g_{c',\mu'} - g_{c,\mu} \right\|_\infty \ln \left\| g_{c',\mu'} - g_{c,\mu} \right\|_\infty.
\end{align}
In view of \eqref{eq:UBAbsDiffIntsigma2}, we conclude the proof.
\end{proof}

\subsection{Proof of Proposition~\ref{prop:MultiChoiceSeparation}}\label{subsec:Proofprop:MultiChoiceSeparation}
We start by establishing a useful technical lemma.
\begin{lemma}\label{lem:NegPosIntervals}
For $i=1,\dots,L$, let $\sigma_i=\sigma_{\mu_i}$, where $\sigma_{\mu_i}$ is defined in \eqref{eq:Defsigma}. If $\mu_i\neq\mu_j$, then there exist $0\le c<d<1$ and $0\le s<t<1$ such that $\sigma_i(r)-\sigma_j(r)<0$ for $r\in(c,d)$ and $\sigma_i(r)-\sigma_j(r)>0$ for $r\in(s,t)$, respectively. 
\end{lemma}
\begin{proof}
We will proceed by contradiction. Suppose for any $(s,t)\subseteq[0,1)$, we have $\sigma_i(r)-\sigma_j(r)\le 0$ for some $r\in(c,d)$. Consequently, $\set{r\in[0,1): \sigma_i(r)-\sigma_j(r) \le 0}$ is dense in $[0,1)$. Combining this with the right continuity in Lemma~\ref{lem:Spectral}, we yield $\sigma_i(r)-\sigma_j(r) \le 0$ for $r\in[0,1)$. But since $\int_0^1\sigma_i(r)\dif r = \int_0^1\sigma_j(r)\dif r =1$, we must have $\sigma_i(r)=\sigma_j(r)$ for Lebesgue almost every $r\in[0,1)$, and thus every $r\in[0,1)$ due to right continuity again. This together with Lemma~\ref{lem:Spectral} and monotone class lemma \cite[Section 4.4, Lemma 4.13]{Aliprantis2006book} implies $\mu_i=\mu_j$, which contradicts the hypothesis that $\mu_i\neq\mu_j$. Analogously, there must exists  $0\le c<d<1$ such that $\sigma_i(r)-\sigma_j(r)<0$ for $r\in(c,d)$ as long as $\mu_i\neq\mu_j$.
\end{proof}


We are in position of proving Proposition~\ref{prop:MultiChoiceSeparation}.\\
\textbf{$\mathbf{L=2}$.} Although the statement for $L=2$ can be considered as a special case of Theorem~\ref{thm:Separation}, we present an alternative proof here in order to better illustrate a key mechanism used in the proof for $L\ge 3$. 
We select arbitrarily $x_0<x_1<x_2$. In view of Lemma~\ref{lem:NegPosIntervals}, we let $(c,d)$ and $(s,t)$ be non-empty interval on $[0,1]$ such that
\begin{align}\label{eq:sigma1sigma2}
\sigma_1(r)<\sigma_2(r),\,\, r\in(c,d) \quad\text{and}\quad \sigma_1(r)>\sigma_2(r),\,\, r\in(s,t).
\end{align} 
Here, we only present the argument for the case of $d\le s$ as the case of $c\ge t$ can be done analogously. By \eqref{eq:sigma1sigma2}, because $\sigma_i,\,i=1,2$ are non-negative and nondecreasing, we have $\sigma_2(r)>0$ for $r>c$. For $\varepsilon\in(0,d-c)$ sufficiently small, there is a $\delta_0\in(0, t-s)$ such that 
\begin{align*}
(x_1-x_0)\int_{c}^{c + \varepsilon} \sigma_2(r)\dif r = (x_2-x_1)\int_{t-\delta_0}^{t}\sigma_2(r)\dif r.
\end{align*} 
By \eqref{eq:sigma1sigma2} again, for the $\varepsilon$ and $\delta$ introduced above, we have 
\begin{multline}\label{eq:ComparexIntsigma}
(x_1-x_0)\int_{c}^{c+\varepsilon}\!\sigma_1(r)\dif r < (x_1-x_0)\int_{c}^{c+\varepsilon}\sigma_2(r)\dif r \\ 
= (x_2-x_1)\int_{t-\delta_0}^{t}\sigma_2(r)\dif r < (x_2-x_1)\int_{t-\delta_0}^{t}\sigma_1(r)\dif r.
\end{multline}
With similar reasoning as before, we have $\sigma_1(r)>0$ for $r>s$. Recall that we also have $\sigma_2(r)>0$ for $r>s$. Note $\sigma_1$ must be locally bounded due to Lemma~\ref{lem:Spectral}. The above together with \eqref{eq:ComparexIntsigma} implies that there exists $\delta\in(\delta_0, t-s)$ such that 
\begin{multline}\label{eq:ConstrEpsilontilde}
(x_1-x_0)\int_{c}^{c+\varepsilon}\sigma_1(r)\dif r < (x_2-x_1)\int_{t-\delta}^{t}\sigma_1(r)\dif r \\ \quad{\text{and}} \quad (x_1-x_0)\int_{c}^{c+\varepsilon}\sigma_2(r)\dif r > (x_2-x_1)\int_{t-\delta}^{t}\sigma_2(r)\dif r.
\end{multline}
Finally, we let $Y$ be a real-valued random variable such that 
\begin{align*}
\bP(Y=x_i) = \begin{cases}
c, & i = 0,\\
t-c, & i = 1, \\
1-t, & i = 2,
\end{cases}
\quad\text{thus}\quad F_{Y}^{-1}(u) = \begin{cases}
x_0, & u\in[0,c),\\
x_1, & u\in[c, t),\\
x_2, & u\in[t, 1],
\end{cases}
\end{align*}
and $Z$ be another real-valued random variable such that 
\begin{align*}
\bP(Z=x_\ell) = \begin{cases}
c+\varepsilon, & \ell = 0,\\
t-\delta-(c+\varepsilon), & \ell = 1, \\
1-(t-\delta), & \ell = 2,
\end{cases}
\quad\text{thus}\quad  F_{Z}^{-1}(u) = \begin{cases}
x_0, & u\in[0,c+\varepsilon),\\
x_1, & u\in[c+\varepsilon, t-\delta),\\
x_2, & u\in[t-\delta, 1].
\end{cases}
\end{align*}
In view of Lemma~\ref{lem:Spectral}, for $i=1,2$,
\begin{align*}
\rho_{\mu_i}(Y) - \rho_{\mu_i}(Z) = (x_1-x_0)\int_{\tilde c}^{c+\varepsilon}\sigma_i(r)\dif r - (x_2-x_1)\int_{t-\delta}^{t}\sigma_i(r)\dif r.
\end{align*}
This together with \eqref{eq:ConstrEpsilontilde} implies that $Y$ is preferred under $\rho_{\mu_1}$ while $Z$ is preferred under $\rho_{\mu_2}$. Constructing $\bX$ and $G$ according to $Y$ and $Z$ finishes the proof for $L=2$.

\noindent
\textbf{$\mathbf{L\ge 3}$.} Let $\bL=\set{1,\dots,L}$ and consider $i\neq j$. In view of the monotonicity in Lemma~\ref{lem:Spectral}, the support of $\sigma_\ell$ is of the form $[b_\ell,1]$ with $b_\ell\in[0,1)$. Without loss of generality, we suppose $\sigma_i$ has the smallest support among all $\sigma_\ell$'s. By Lemma~\ref{lem:NegPosIntervals}, we let  $I^{<}_{ij}$ and $I^{>}_{ij}$ be nonempty open intervals such that $\sigma_{i}(r)<\sigma_j(r)$ for all $r\in I^{<}_{ij}$ and $\sigma_{i}(r)>\sigma_j(r)$ for all $r\in I^{>}_{ij}$, respectively.  
Thanks to the monotonicity in Lemma~\ref{lem:Spectral} and assumption that $\sigma_i$ has the smallest support, we must have
\begin{equation}\label{eq:SupportI}
\text{one of $I^{<}_{ij}$ and $I^{>}_{ij}$ is included by $\bigcap_{\ell\in\bL}\supp\sigma_{\ell}$.}
\end{equation}
We construct $X^2_1$ and $X^2_2$ as in the case of $L=2$ such that
\begin{align}\label{eq:PreferentialOrder2}
\rho_{\mu_i}(X^2_1)<\rho_{\mu_i}(X^2_2)\quad\text{but}\quad\rho_{\mu_j}(X^2_1)>\rho_{\mu_j}(X^2_2)
\end{align}
Note that $X^2_1$ and $X^2_2$ have the same finite range with cardinality of at most $3$. Consequently, $F_{X^2_1}^{-1}$ and $F_{X^2_2}^{-1}$ are piecewise constant on $[0,1)$ with finitely many jumps. Moreover, in view of \eqref{eq:SupportI} and the construction procedure in the case of $L=2$, we can slightly perturb the probabilities associated with $X^2_1$ and $X^2_2$ such that $\rho_{\mu_{\ell}}(X^2_1)\neq\rho_{\mu_{\ell}}(X^2_2)$ for any $\ell\in\bL$ while the preference order in \eqref{eq:PreferentialOrder2} remains unchanged. Based on the discussion above, we divide $\bL$ into a partition $\bL^2_1$ and $\bL^2_2$ such that for any $k,k'\in\{1,2\}$ with $k\neq k'$ and $\ell\in\bL^2_k$ we have $\rho_{\mu_\ell}(X^2_k) < \rho_{\mu_{\ell}}(X^2_{k'})$. We also note that $\Delta^2:=\min_{k,k'\in\set{1,2},\,k'\neq k,\,\ell\in\bL^2_k }\big|\rho_{\mu_{\ell}}(X^2_k)-\rho_{\mu_{\ell}}(X^2_{k'})\big|>0.$

We will proceed by induction. Suppose for some $K\ge 2$, there are $(X^K_k)_{k=1}^K$ and $(\bL^K_k)_{k=1}^K$ such that
\begin{itemize}
\item for any $k\in\set{1,\dots,K}$, $\range(X^K_k)\subseteq\set{x_0,x_1,\dots,x_{2K}}$ with $x_0<x_1<\dots<x_{2K}$ and thus
\begin{align*}
\cJ^K_{k}:=\left\{u\in[0,1]:F^{-1}_{X^K_k}(u-)\neq F^{-1}_{X^K_{k}}(u+)\right\} \text{ is finite},
\end{align*}
where we set $F^{-1}(0-)=F^{-1}(0)$ and $F^{-1}(1+)=F^{-1}(1)$;
\item $\bL^K_k$ is nonempty for $k=1,\dots,K$ and $\bL=\bigcup_{k=1}^K \bL^K_k$,
\item for any $k,k'\in\set{1,\dots,K}$ with $k\neq k'$ and  $\ell\in\bL^K_k$, we have $\rho_{\mu_\ell}(X^K_k) < \rho_{\mu_{\ell} }(X^K_{k'})$,
\item $\Delta^K:=\min_{k,k'\in\set{1,\dots,K},\,k\neq k',\ell\in\bL^K_k }\big|\rho_{\mu_{\ell}}(X^K_k)-\rho_{\mu_{\ell}}(X^K_{k'})\big|>0.$
\end{itemize}
Without loss of generality, we assume $\bL^K_1$ has more than two elements. With similar reasoning leading to \eqref{eq:SupportI}, we pick $i,j\in \bL^K_1$ and non-empty open intervals $(c,d)$ and $(s,t)$ satisfying 
\begin{gather}
\text{$\sigma_{i}(r)<\sigma_{j}(r)$ for all $r\in I^{<}_{ij}$\quad and \quad $\sigma_{i}(r)>\sigma_{j}(r)$ for all $r\in I^{>}_{ij}$.} \label{eq:sigmaisigmaj}\\
\text{one of } I^{<}_{ij} \text{ and }  I^{>}_{ij} \text{ is included by } \bigcap_{\ell\in\bL^K_1}\supp\sigma_{\ell}.\label{eq:SupportIhat}
\end{gather}
Additionally, we require that $\overline{I^{<}_{ij}}$ and $\overline{I^{>}_{ij}}$ do not overlap with $\cJ^K_{1}$, which consists of finitely many points. This is viable as $\cJ^K_{1}$ is finite. In what follows, we denote $I^{<}_{ij}=(c,d)$ and $I^{>}_{ij}=(s,t)$.

Without loss of generality, we suppose $d<s$. Let $u_1,u_2, u_3,u_4 \in \cJ^K_{1}\cup\{0,1\}$ such that\footnote{$\subset$ means being strict subset of.}  
\begin{multline}\label{eq:uIntervals}
u_1<u_2,\quad u_3<u_4,\quad (c,d)\subset(u_1,u_2),\quad (s,t)\subset(u_3,u_4),\quad ((u_1, u_2)\cup(u_3,u_4))\cap\cJ^K_{1} = \emptyset,
\end{multline} 
where the last condition is viable because $[c,d]$ and $[s,t]$ do not overlap with $\cJ^K_{1}$, which is finite. Note that $u_1<u_4$ because $d<s$. We select arbitrarily 
\begin{align}\label{eq:Rangeh}
h^-, h^+\in\left(0,\frac12\left(\min_{k\in\set{1,\dots,K}}\left\{x_{k}-x_{k-1}\right\}\wedge \Delta^K\right)\right].
\end{align}

We proceed with a similar construction as in the case of $L=2$. To start with, by Lemma~\ref{lem:Spectral} and \eqref{eq:sigmaisigmaj}, we have $\sigma_{j}(r)>0$ for $r>c$. Thus, for $\varepsilon\in(0,d-c)$ sufficiently small, there is $\delta_0\in(0,t-s)$ such that
\begin{align*}
h^{-}\int_{c}^{c+\varepsilon} \sigma_{j}(r)\dif r = h^+\int_{t-\delta_0}^{t}\sigma_{j}(r)\dif r.
\end{align*} 
By \eqref{eq:sigmaisigmaj}, 
\begin{align}\label{eq:ComparehIntsigma}
h^{-}\int_{c}^{c+\varepsilon} \sigma_{i}(r)\dif r < h^{-}\int_{c}^{c+\varepsilon} \sigma_{j}(r)\dif r = h^+\int_{t-\delta_0}^{t}\sigma_{j}(r)\dif r < h^+\int_{t-\delta_0}^{t}\sigma_{i}(r)\dif r.
\end{align}
By Lemma~\ref{lem:Spectral} and \eqref{eq:sigmaisigmaj}, we also have $\sigma(r)>0$ for $r>s$. This together with \eqref{eq:ComparehIntsigma} implies that there exists a $\delta\in(\delta_0,t-s)$ such that 
\begin{align}\label{eq:ConstrEpsilonDeltaK1}
h^{-}\int_{c}^{c+\varepsilon} \sigma_{i}(r)\dif r < h^+\int_{t-\delta}^{t}\sigma_{i}(r)\dif r \quad{\text{and}} \quad h^{-}\int_{c}^{c+\varepsilon} \sigma_{j}(r)\dif r > h^+\int_{t-\delta}^{t}\sigma_{j}(r)\dif r.
\end{align}
Furthermore, in view of \eqref{eq:SupportIhat} and the assumptions that $d<s$, we can pick $\delta$ such that 
\begin{align}\label{eq:ConstrEpsilonDeltaK1ell}
h^{-}\int_{c}^{c+\varepsilon} \sigma_{\ell}(r)\dif r \neq h^+\int_{t-\delta}^{t}\sigma_{\ell}(r)\dif r,\quad \ell\in\bL^K_1.
\end{align}
We then define $Y$ with
\begin{align*}
F^{-1}_Y(u) := F^{-1}_{X^K_1}(u) - h^-\1_{[u_1,c)}(u) + h^+\1_{[t, u_4)}(u), \quad u\in[0,1],
\end{align*}
and define $Z$ with 
\begin{align*}
F^{-1}_Z(u) := F^{-1}_{X^K_1}(u) - h^-\1_{[u_1,c+\varepsilon)}(u) + h^+\1_{[t-\delta, u_4)}(u),\quad u\in[0,1),
\end{align*}
where we note that $F^{-1}_Y$ and $F^{-1}_{Z}$ are both valid inverse CDFs because of \eqref{eq:uIntervals} and \eqref{eq:Rangeh}. It follows from Lemma~\ref{lem:Spectral} that  
\begin{align*}
\rho_{\mu_{\ell}}(Y) - \rho_{\mu_{\ell}}(Z) = h^-\int_{c}^{c+\varepsilon}\sigma_{\ell}(r)\dif r - h^+\int_{t-\delta}^{t}\sigma_{\ell}(r)\dif r, \quad \ell\in\bL^K_1.
\end{align*}
This together with \eqref{eq:ConstrEpsilonDeltaK1} implies that $Y$ is preferred under $\rho_{\mu_i}$ and $Z$ is preferred under $\rho_{\mu_j}$. Moreover, in view of \eqref{eq:ConstrEpsilonDeltaK1ell}, $\rho_{\mu_\ell}(Y)\neq\rho_{\mu_\ell}(Z)$ for any $\ell\in\bL^K_1$. Furthermore, by \eqref{eq:Rangeh} and the fact that $\int_0^1\sigma_\mu(r)\dif r=1$ due to Lemma~\ref{lem:Spectral}, we have 
\begin{align*}
\max\big\{\big|\rho_{\mu_\ell}(X^K_1)-\rho_{\mu_\ell}(Y)|, \big|\rho_{\mu_\ell}(X^K_1)-\rho_{\mu_\ell}(Z)|\big\} < \frac12\Delta^K,
\end{align*}
and thus, for $k\in\set{2,\dots,K}$, we have
\begin{align}\label{eq:CompareYZX}
\min\big\{\rho_{\mu_\ell}(Y), \rho_{\mu_\ell}(Z)\big\} < \rho_{\mu_\ell}(X^K_k),\;\ell\in\bL^K_1 \quad\text{and}\quad \rho_{\mu_\ell}(X^K_k) < \min\big\{\rho_{\mu_\ell}(Y), \rho_{\mu_\ell}(Z)\big\},\;\ell\in\bL^K_k.
\end{align}
Finally, to finish the construction, we define $X^{K+1}_1=Y$, $X^{K+1}_2=Z$, $X^{K+1}_k=X^K_{k-1}$ for $k=3,\dots,K+1$. Furthermore, for $k=1,\dots,K+1$ we let 
\begin{align}\label{eq:DefbLK+1}
\bL^{K+1}_k:=\left\{\ell\in\bL: \rho_{\mu_\ell}(X^{K+1}_k) < \rho_{\mu_\ell}(X^{K+1}_{k'}), k'\neq k\right\}
\end{align}
It follows that 
\begin{itemize}
\item for any $k\in\set{1,\dots,K}$, $\range(X^{K+1}_k)\subseteq\set{x_0,x_1,\dots,x_{2k}} \cup \left\{F^{-1}_{X^K_1}(u_1)-h^-, F^{-1}_{X^K_1}(u_4)+h^+\right\}$;
\item $\bL^{K+1}_k$ is nonempty for $k=1,\dots,K+1$ due to the construction above, and $\bL=\bigcup_{k=1}^{K+1} \bL^{K+1}_k$;
\item for any $k,k'\in\set{1,\dots,K+1}$ with $k\neq k'$ and  $\ell\in\bL^{K+1}_k$, we have $\rho_{\mu_\ell}(X^{K+1}_k) < \rho_{\mu_{\ell} }(X^{K+1}_{k'})$ by \eqref{eq:DefbLK+1};
\item $\Delta^K:=\min_{k,k'\in\set{1,\dots,K+1},\,k\neq k',\ell\in\bL^K_k }\big|\rho_{\mu_{\ell}}(X^{K+1}_k)-\rho_{\mu_{\ell}}(X^{K+1}_{k'})\big|>0$ due to induction hypothesis, \eqref{eq:ConstrEpsilonDeltaK1ell} and \eqref{eq:CompareYZX}.
\end{itemize}
Note that the construction above introduces $2$ more elements to $\bigcup_{k=1}^{K}\range(X^{K}_k)$. After $L$ iterations, we obtain a partition of $\bL$ consists of singletons only. Constructing $\bX$ and $G$ accordingly, we conclude the proof.

\subsection{Supplementary settings and proof of Theorem~\ref{thm:Separation2}}\label{subsec:ProofSeparation2}

We first introduce the notation and key distinctions from the single-period case, particularly the incorporation of a discount factor. After this, we then presents the proof of Theorem \ref{thm:Separation2}, the identifiability for the infinite-horizon setting.
\paragraph{Risk aversion modeling.}
We consider a stationary infinite-horizon setting. Let $\cM(\bX)$ be the set of transition matrices on a Markov chain $\bX$. Below we introduce the client's evaluation process of the loss processes related to $M\in\cM(\bX)$ under the hypothetical risk aversion $(c,\mu,\gamma)\in(0,1)\times\cP([0,1])\times(0,1)$.
\begin{definition}\label{def:DRM}
We first introduce ${\bT}^{M}_{c,\mu,\gamma}:\ell^{\infty}(\bX)\to\ell^{\infty}(\bX)$ by how it acts on functions $u\in\ell^{\infty}(\bX)$ as follows\footnote{$\ell^\infty(\bX)$ stands for the set of bounded real valued functions on $\bX$.}   
\begin{align*}
\bT^{M}_{c,\mu,\gamma} u(x) := C(x) + \gamma\,\rho_{\mu}\big(u\left(X_{M^x}\right)\big), \quad X_{M^x}\sim M^x
\end{align*}
where $M^x$ represents the transition probability given the current state $x$. For $\tau\in\bN$, we then let\footnote{$\mathbf 0$ stands for a function that is constantly $0$.}
\begin{align*}
\varrho^{\tau}_{c,\mu,\gamma}\left(x,M\right) := \underbrace{{\bT}^{M}_{c,\mu,\gamma} \circ \dots \circ {\bT}^{M}_{c,\mu,\gamma}}_{\tau \text{ times}}\, \mathbf{0} (x),
\end{align*} 
and $\varrho_{c,\mu,\gamma}\left(x,M\right):=\displaystyle\lim_{\tau\to\infty}\varrho^{\tau}_{c,\mu,\gamma}(x,M)$. Finally, for an initial distribution $m\in\cP(\bX)$, with a slight abuse of notation, we define
\begin{align}
\varrho_{c,\mu,\gamma}(m,M) := \rho_\mu\big( \varrho_{c,\mu,\gamma}(X_m,M) \big), \quad X_m\sim m.
\end{align}
\end{definition}
It follows from the monotonicity and translation invariance of coherent risk measure \cite[Section 6.3]{Shapiro2021book} that the limit defining $\varrho_{c,\mu,\gamma}(x,{M})$ is valid and $\varrho_{c,\mu,\gamma}(x,{M}) \le (1-\gamma)^{-1}\|C\|_\infty$. This type of performance criteria was first proposed by \cite{Ruszczynski2010Risk} and is now widely used in many disciplines, such as finance and autonomous robotics \citep[cf.][and the reference therein]{coache2023conditionally, Wang2022Risk}. It is known \citep{Ruszczynski2010Risk} that $x\mapsto\varrho_{c,\mu,\gamma}(x,{M})$ is the unique solution of the following fixed-point equation below unknown $V \in\ell^\infty(\bX)$,
\begin{align}\label{eq:Evaluation}
V (x) = {\bT}^{{M}}_{c,\mu,\gamma} V(x) = C(x) + \gamma  \rho_\mu\big(V\left(X_{M^x}\right)\big), \quad x\in\bX.
\end{align}
\begin{remark}
Constant discount factors may be limited in modeling human risk aversion. It is believed that people typically adjust their discounting when facing different levels of uncertainty in the future. This belief is reflected in, for example, \cite{Zhao2020Ambiguity}.
\end{remark}

\paragraph{Interactive questioning.} We consider a similar iterative scheme as in the one period case.
At the $n$-th round of interactions, the learner presents the client with a binary-choice question characterized by $H_n\in\big(\cP(\bX)\times\cM(\bX)\big)^\bA$. More precisely, by using $a\in\set{\text A, \text B}$ as a dummy variable for choices, we let $H^a_n=(m^a_n,M^a_n)\in\cP(\bX)\times\cM(\bX)$ represent the initial distribution and transition matrix of a stationary infinite-horizon cash flow. The client then returns the optimal choice based on his or her risk aversion:
\begin{align}
a^*_n \in \argmin_{a\in\bA} \varrho_{c_0,\mu_0,\gamma_0}(H^a_n).
\end{align}
As discussed in the one-period case, this formulation assumes that the agent is an expert and always chooses the best option, which may not be true in practical situations.

\paragraph{Proof.} We are now ready to prove Theorem~\ref{thm:Separation2}.

\begin{remark}\label{rem:ComplexForm}
The proof below indicates that distinguishing binary-choice questions may require multiple formats, suggesting increased complexity in question design. As a consequence, the indifference curve from the static case, introduced in \eqref{eq:Defg}, may need to be generalized in more complex settings; for instance, an indifference surface. The convergence rate with design is elusive.
\end{remark}

\begin{proof}[Proof of Theorem~\ref{thm:Separation2}]
The proof proceeds in two cases. First, we consider the scenario where $(c_0,\mu_0)\neq(c',\mu')$ and $\gamma_0, \gamma'$ may or may not differ. In this case, we construct a distinguishing binary-choice question based on a state-homogeneous process, leveraging the finite-horizon result from Theorem~\ref{thm:Separation}. Second, we address the case where $(c_0,\mu_0)=(c',\mu')$ but $\gamma_0\neq \gamma'$, employing a construction that features an absorption state.\\

\textbf{Case 1.} In the first case, we suppose $(c_0,\mu_0)\neq(c',\mu')$ whereas $\gamma_0$ and $\gamma'$ may or may not be the same. In this case, we set $H_{p,q}$ with the components below
\begin{align}
m^{\text{A}}_{p,q} = \!\!\!\!\bordermatrix{ ~ & x_0 & x_1 & x_2 \cr ~ & 1-p & p & 0 \cr}, \quad M^{\text{A}}_{p,q} = \bordermatrix{~ & x_0 & x_1 & x_2 \cr x_0 & 1-p & p & 0 \cr x_1 & 1-p & p & 0 \cr x_2 & \bullet & \bullet & \bullet \cr},\label{eq:InfHorStateHomo1}\\
m^{\text{B}}_{p,q} = \!\!\!\!\bordermatrix{ ~ & x_0 & x_1 & x_2 \cr ~ & 1-q & 0 & q \cr}, \quad M^{\text{B}}_{p,q} = \bordermatrix{~ & x_0 & x_1 & x_2 \cr x_0 & 1-q & 0 & q \cr x_1 & \bullet & \bullet & \bullet \cr x_2 & 1-q & 0 & q \cr}.\label{eq:InfHorStateHomo2}
\end{align}
Here, we use $\bullet$ to omit the transition distribution in $M^{\text{A}}_{p,q}$ given $x_2$ (resp. $M^{\text{B}}_{p,q}$ given $x_1$), as $x_2$ (resp. $x_1$) is never visited for choice $\text{A}$ (resp. $\text{B}$).
Note that the transition is stationary and space-homogeneous. It follows from Definition~\ref{def:DRM} and \eqref{eq:Exprrho} that 
\begin{gather}
\varrho_{c_0,\mu_0,\gamma_0}\big(m^{\text{A}}_{p,q}, M^{\text{A}}_{p,q}\big) = \frac{c}{1-\gamma_0} \int_{1-p}^1\sigma_{\mu_0}(\alpha)\dif\alpha,\quad \varrho_{c_0,\mu_0,\gamma_0}\big(m^{\text{B}}_{p,q}, M^{\text{B}}_{p,q}\big) = \frac{1}{1-\gamma_0} \int_{1-q}^1\sigma_{\mu_0}(\alpha)\dif\alpha,\\
\varrho_{c',\mu',\gamma'}\big(m^{\text{A}}_{p,q}, M^{\text{A}}_{p,q}\big) = \frac{c'}{1-\gamma'} \int_{1-p}^1\sigma_{\mu'}(\alpha)\dif\alpha,\quad \varrho_{c',\mu',\gamma'}\big(m^{\text{B}}_{p,q}, M^{\text{B}}_{p,q}\big) = \frac{1}{1-\gamma'} \int_{1-q}^1\sigma_{\mu'}(\alpha)\dif\alpha.
\end{gather}
Observe that preference under $(c_0,\mu_0,\gamma_0)$ only relies on the comparison between $c_0\int_{1-p}^1\sigma_{\mu_0}(\alpha)\dif\alpha$ and $\int_{1-q}^1\sigma_{\mu_0}(\alpha)\dif\alpha$, which is independent of $\gamma$. The same is true for $(c',\mu',\gamma')$. This allows us to directly invoke Theorem~\ref{thm:Separation} for a distinguishing $(p,q)$. 

\textbf{Case 2.} In the second case, we suppose $(c_0,\mu_0)=(c',\mu')$ but $\gamma_0\neq\gamma'$. In this case, we set up $H_{p,q}$ with the components below
\begin{gather}
m^{\text{A}}_{p,q} = \!\!\!\!\bordermatrix{ ~ & x_0 & x_1 & x_2 \cr ~ & 1-p & p & 0 \cr}, \quad M^{\text{A}}_{p,q} = \bordermatrix{~ & x_0 & x_1 & x_2 \cr x_0 & 1 & 0 & 0 \cr x_1 & 1-p & p & 0 \cr x_2 & \bullet & \bullet & \bullet \cr},\\
m^{\text{B}}_{p,q} = \!\!\!\!\bordermatrix{ ~ & x_0 & x_1 & x_2 \cr ~ & 1-q & 0 & q \cr}, \quad M^{\text{B}}_{p,q} = \bordermatrix{~ & x_0 & x_1 & x_2 \cr x_0 & 1 & 0 & 0 \cr x_1 & \bullet & \bullet & \bullet \cr x_2 & 1-q & 0 & q \cr}.   
\end{gather}
In view of \eqref{eq:Evaluation}, we calculate $\varrho_{c,\mu,\gamma}\left(x_1, M^{\text{A}}_{p,q}\right)=:u$ by solving the following equation,
\begin{align}\label{eq:u}
u = C_0(x_1) + \gamma \rho_{\mu_0}\left(\varrho_{c_0,\mu_0,\gamma_0}\left(X_{{M}^{\text{A},x_1}_{p,q}}, {M}^{\text{A}}_{p,q}\right)\right) = c_0 + \gamma_0 u \int_{1-p}^1 \sigma_{\mu_0}(\alpha)\dif\alpha,
\end{align}
where the second equality follows from \eqref{eq:DefSRMOG} along with the settings that $C_0(x_1)=c_0$, $\varrho_{c_0,\mu_0,\gamma_0}(x_0,{M}^{\text{A}}_{p,q})=0$, and state $x_2$ will not be visited. It follows that
\begin{align}\label{eq:Idenrhopi}
\varrho_{c_0,\mu_0,\gamma_0}\left(x_1, M^{\text{A}}_{p,q}\right) = \left(1-\gamma_0\int_{1-p}^1 \sigma_{\mu_0}(\alpha)\dif\alpha\right)^{-1} c_0.
\end{align}
Furthermore, by \eqref{eq:DefSRMOG} again,
\begin{align}
\varrho_{c_0,\mu_0,\gamma_0}\left(m^{\text{A}}_{p,q}, M^{\text{A}}_{p,q}\right) = \frac{c_0\int_{1-p}^1\sigma_{\mu_0}(\alpha)\dif\alpha}{1-\gamma_0\int_{1-p}^1 \sigma_{\mu_0}(\alpha)\dif\alpha}.
\end{align}
With a similar calculation, we yield
\begin{align}
\varrho_{c_0,\mu_0,\gamma_0}\left(m^{\text{B}}_{p,q}, M^{\text{B}}_{p,q}\right) = \frac{\int_{1-q}^1\sigma_{\mu_0}(\alpha)\dif\alpha}{1-\gamma_0\int_{1-q}^1 \sigma_{\mu_0}(\alpha)\dif\alpha}.
\end{align}
The analogue is true for $(c_0,\mu_0,\gamma')$. Since $(c_0,\mu_0)=(c',\mu')$, it is sufficient to consider the substitution
\begin{align}
\tilde p := \int_{1-p}^1\sigma_{\mu_0}(\alpha)\dif\alpha \quad\text{and}\quad \tilde q:=\int_{1-q}^1\sigma_{\mu_0}(\alpha)\dif\alpha,
\end{align}
and investigate the following indifference curves for $\tilde p\in[0,1]$:
\begin{gather}
\tilde g_0(\tilde p) := \inf\left\{ \tilde q\in[0,1]: \frac{\tilde q}{1-\gamma_0 \tilde q} > \frac{c_0\tilde p}{1-\gamma_0\tilde p}  \right\} =  \frac{c_0\tilde p}{1-(1-c_0)\gamma_0 \tilde p},\label{eq:InfHorg1}\\
\tilde g'(\tilde p) := \inf\left\{ \tilde q\in[0,1]: \frac{\tilde q}{1-\gamma' \tilde q} > \frac{c_0\tilde p}{1-\gamma'\tilde p}  \right\} =  \frac{c_0\tilde p}{1-(1-c_0)\gamma' \tilde p}.\label{eq:InfHorg2}
\end{gather}
It is clear that the two indifference curves do not intercept except when $\tilde p=0$. The distinguishing $(\tilde p,\tilde q)$, and thus $(p,q)$, can be constructed accordingly.
\end{proof}

\end{document}

%% file: tikz-separating-game.tex
\begin{figure}[htbp]
\centering
\begin{tikzpicture}[scale=1.0]
  
\begin{axis}[
    ticklabel style={font=\scriptsize},
    axis lines = middle,
    axis y line*=left,
    xmin=-0.0, xmax=0.5,
    xtick={0,0.25,0.3}, extra x ticks={0}, xticklabels={$0$,$c\kappa$,$c\kappa'$},
    ymin=0.0, ymax=0.9,
    ytick={0,0.5,0.6}, yticklabels={$0$,$\kappa$,$\kappa'$},
    xlabel=$q$, ylabel=$p$     
]
 
 \node (g00) at (0,0){};
 \node (g01) at (0.25,0){};
 \node (g02) at (0.3,0){};
 \node (g10) at (0,0.5){};
 \node (g11) at (0.25,0.5){};
 \node (g20) at (0,0.6){};
 \node (g21) at (0.3,0.6){};

 \shadedraw[left color=red!50,right color=red!50, draw=red, fill opacity=0.3] (g00.center) -- (g11.center) -- (0.25,0.9) -- (0,0.9);
 \shadedraw[left color=blue!50,right color=blue!50, draw=blue, fill opacity=0.3] (g11.center) -- (g21.center) -- (0.3,0.9) -- (0.25,0.9);
 \shadedraw[left color=black!50,right color=black!50, draw=black, fill opacity=0.3] (g00.center) -- (0.5,0.0) -- (0.5,0.9) -- (0.3,0.9) -- (g21.center);
 
 \draw[very thick] (g00.center) -- (g21.center);
 \draw[very thick] (g21.center) -- (0.3,0.9);
 \draw[very thick] (g11.center) -- (0.25,0.9);
 \draw[dashed] (g01.center) -- (0.25,0.9);
 \draw[dashed] (g02.center) -- (g21.center);
 \draw[dashed] (g10.center) -- (g11.center);
 \draw[dashed] (g20.center) -- (g21.center);

\node[anchor=east,red,text width=2.5cm] (source1) at (0.2,0.7) {always pick B regardless};
\node (destination1) at (0.125, 0.4){};
\draw[red,->](source1)--(destination1);

\node[anchor=west,blue,text width=2.0cm] (source2) at (0.34,0.75) {A under $\mu$; B under $\mu'$};
\node (destination2) at (0.275, 0.70){};
\draw[blue,->](source2)--(destination2);

\node[anchor=south,black,text width=2.5cm] (source3) at (0.42,0.3) {always pick A regardless};
\node (destination3) at (0.2, 0.3){};
\draw[black,->](source3)--(destination3);

\end{axis}
\end{tikzpicture}
\caption{Illustration of a separating environment.}
\label{fig:DistinguishingGame}
\medskip
\small
Suppose $C_0(x_1)=c_0\in(0,1)$ is known. We consider $0<\kappa<\kappa'<1$ and set $\mu_0=\delta_{1-\kappa}$, $\mu'=\delta_{1-\kappa'}$, i.e., $\rho_\mu=\avar_{\kappa}$ and $\rho_{\mu'}=\avar_{\kappa'}$. Let $G_{p,q}$ be as defined in \eqref{eq:DefGpq}. Values of $p$ and $q$ affect the optimal actions under $\rho_{\mu_0}$ and $\rho_{\mu'}$. The blue region is where $(p,q)$ make the the optimal actions under $\rho_\mu$ and $\rho_{\mu'}$ distinct.
\end{figure}

%% file: fig-gibbs-oneperiod-set1.tex
\begin{figure}[htbp]
\centering

\begin{subfigure}[t]{0.55\textwidth}
    \centering
    \includegraphics[width=0.98\textwidth]{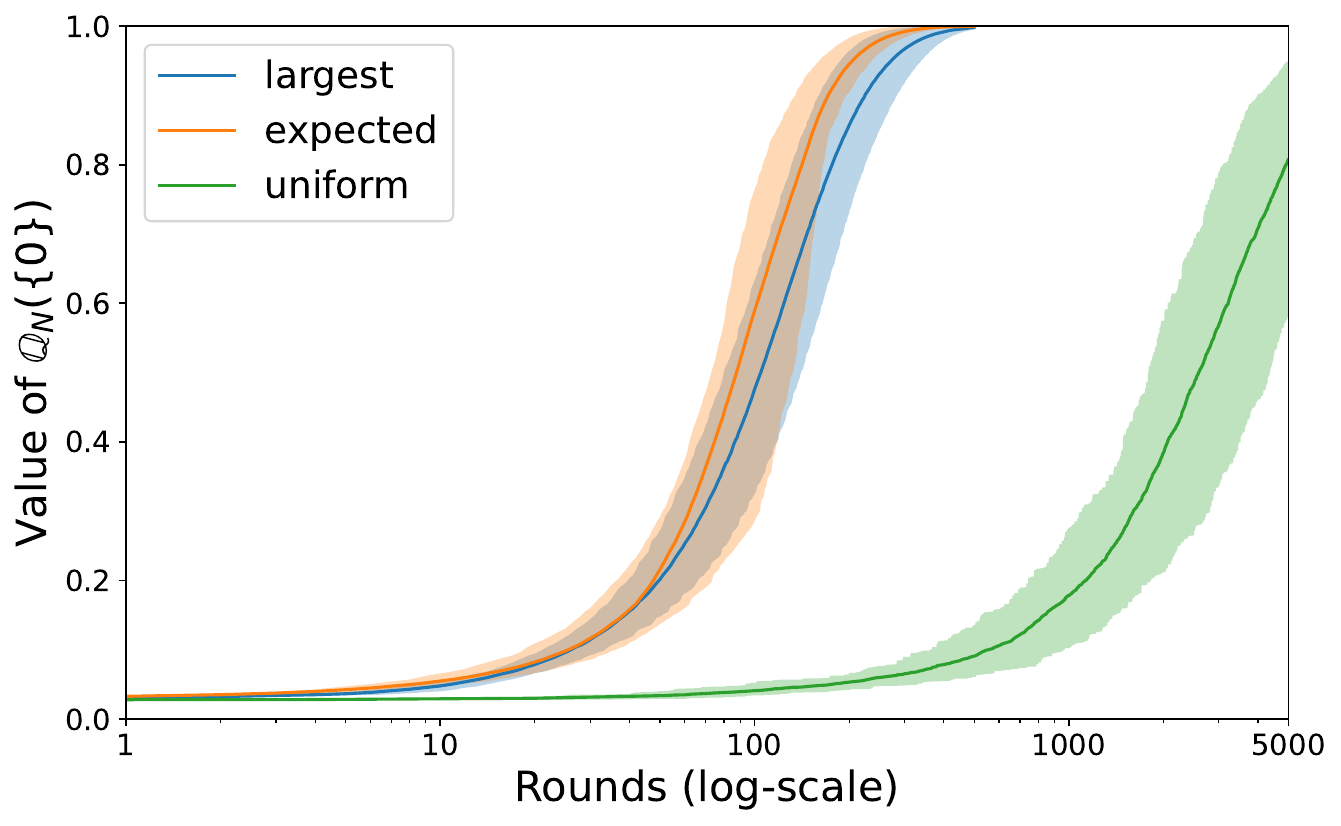}
    \caption{Fixed learning rate $k=4$.}
    \label{fig:AllNrounds-OnePeriod-Set1a}
\end{subfigure}
\hfill
\begin{subfigure}[t]{0.55\textwidth}
    \centering
    \includegraphics[width=0.98\textwidth]{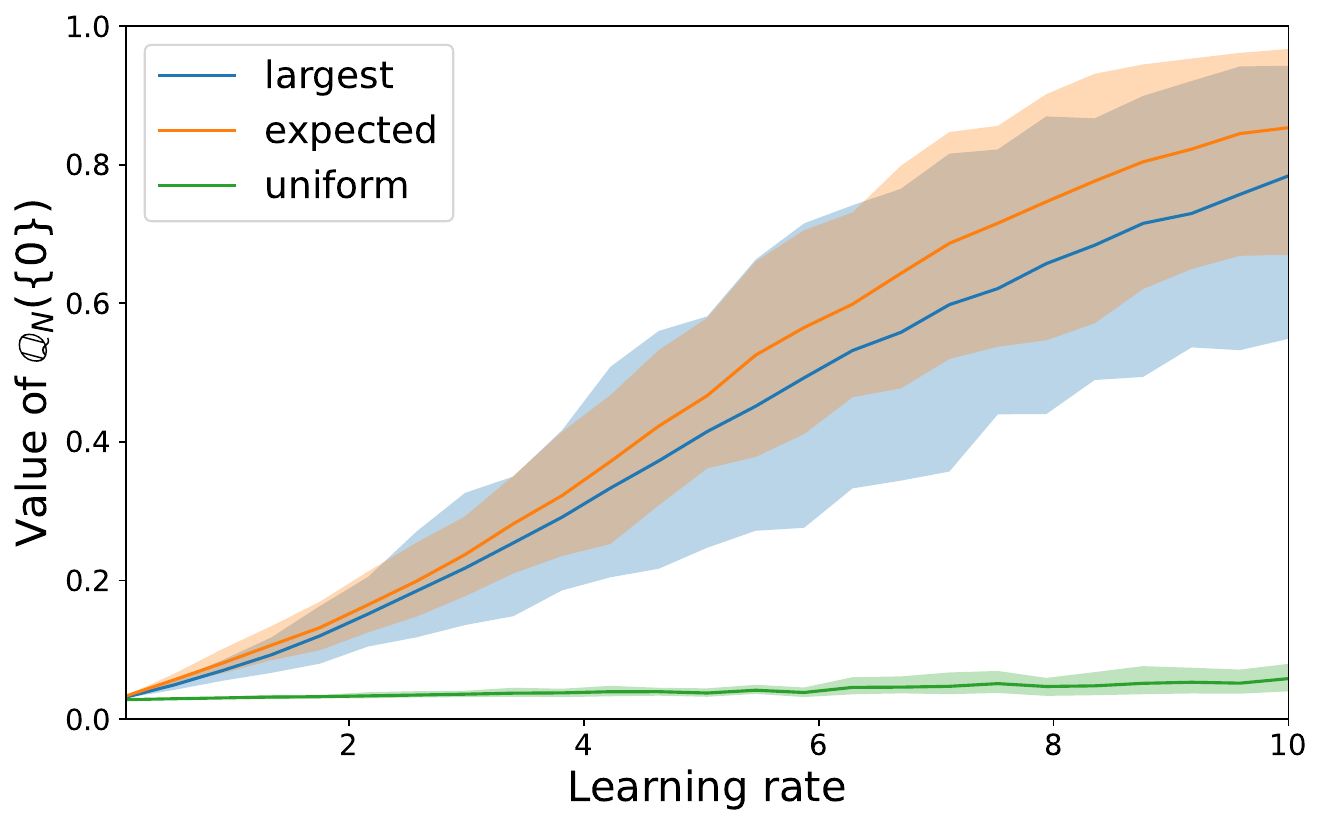}
    \caption{Fixed number of rounds $N=50$.}
    \label{fig:AllLrs-OnePeriod-Set1a}
\end{subfigure}

\caption{Convergence of $\bQ_N$ in the one-period setting.}

\medskip
\small

Measure values for the client's true risk-aversion at each round of the learning algorithm when selecting the next question fully at random (``uniform''), according to \eqref{eq:DesignGOnetoOne} (``largest''), or according to \eqref{eq:DesignGBatch} (``expected''), while varying the number of rounds (top) and the learning rate $k$ (bottom). Expectation, 10\% and 90\% quantiles are estimated over 25 runs.
\end{figure}

%% file: fig-oneperiod-set1.tex
\begin{figure}[htbp]
\centering

\begin{subfigure}[t]{0.48\linewidth}
    \centering
    \includegraphics[width=0.90\textwidth]{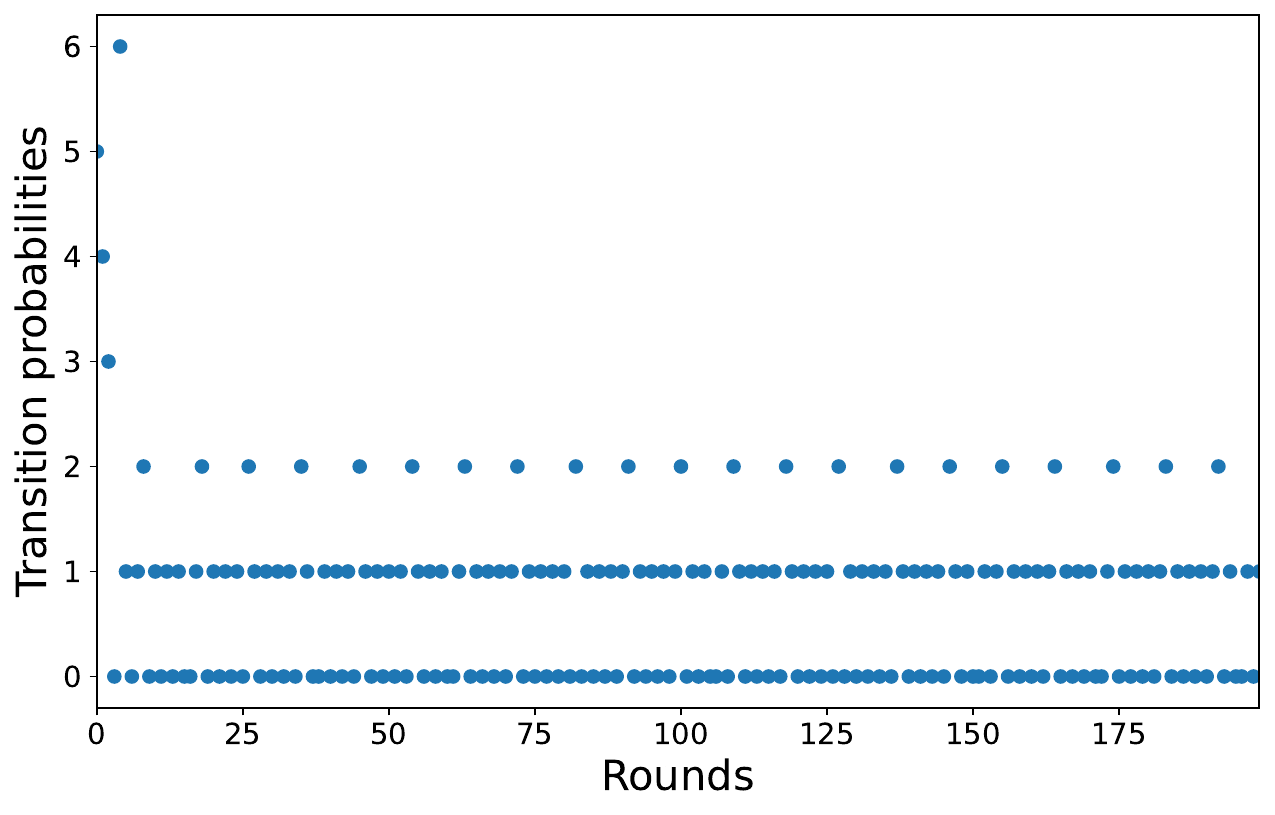}
    \caption{According to \eqref{eq:DesignGOnetoOne}.}
    \label{fig:Games-OnePeriod-Set1a-OnetoOne}
\end{subfigure}
\hfill
\begin{subfigure}[t]{0.48\linewidth}
    \centering
    \includegraphics[width=0.90\textwidth]{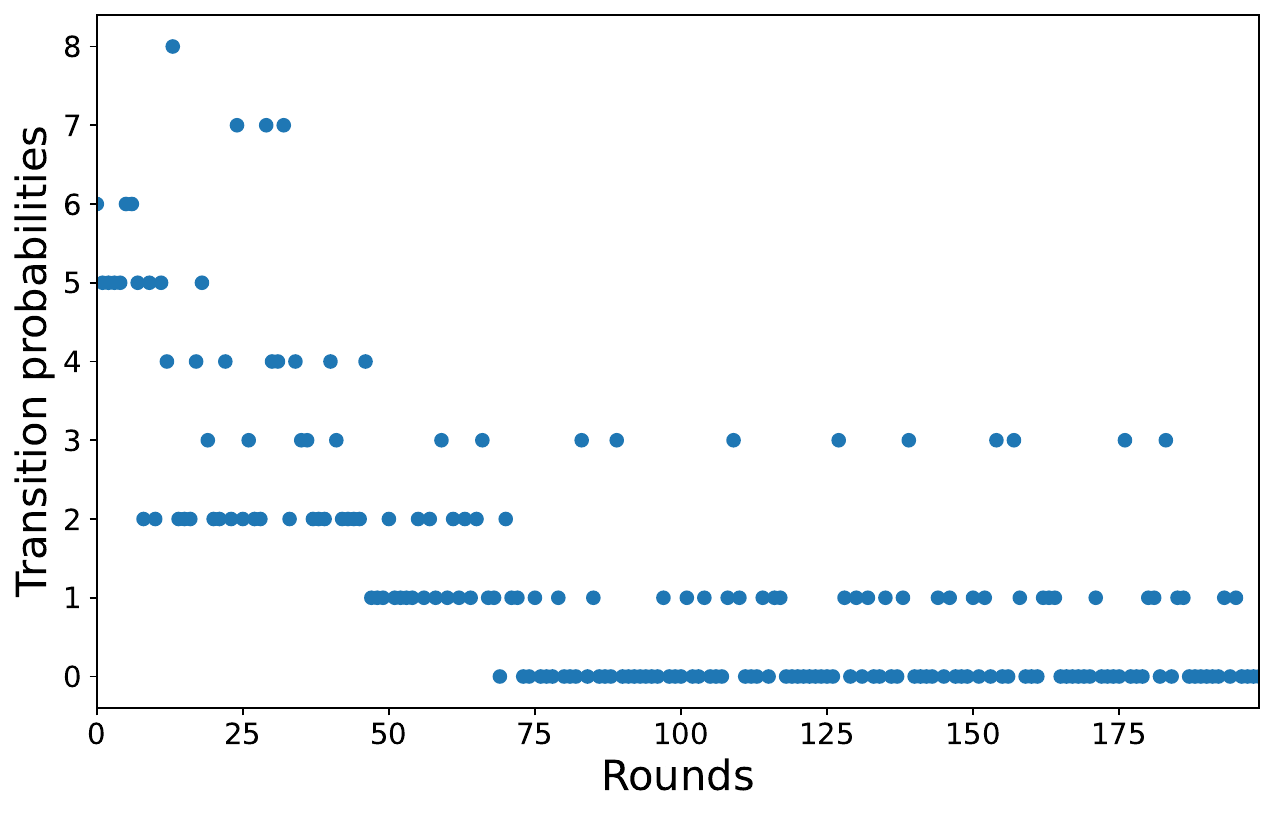}
    \caption{According to \eqref{eq:DesignGBatch}.}
    \label{fig:Games-OnePeriod-Set1a-Batch}
\end{subfigure}

\begin{subfigure}[t]{0.48\linewidth}
    \centering
    \includegraphics[width=0.90\textwidth]{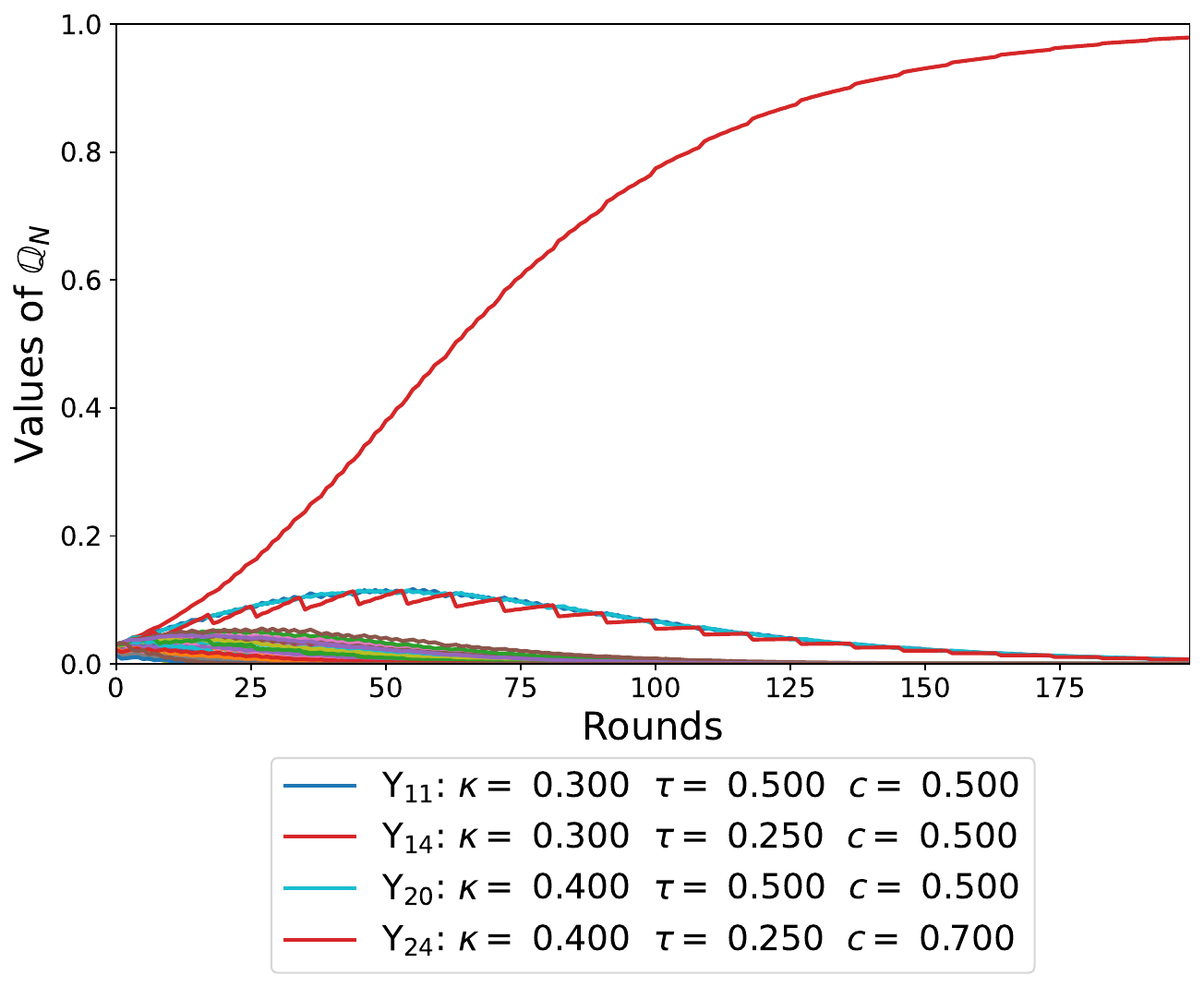}
    \caption{According to \eqref{eq:DesignGOnetoOne}.}
    \label{fig:Gibbs-OnePeriod-Set1a-OnetoOne}
\end{subfigure}
\hfill
\begin{subfigure}[t]{0.48\linewidth}
    \centering
    \includegraphics[width=0.90\textwidth]{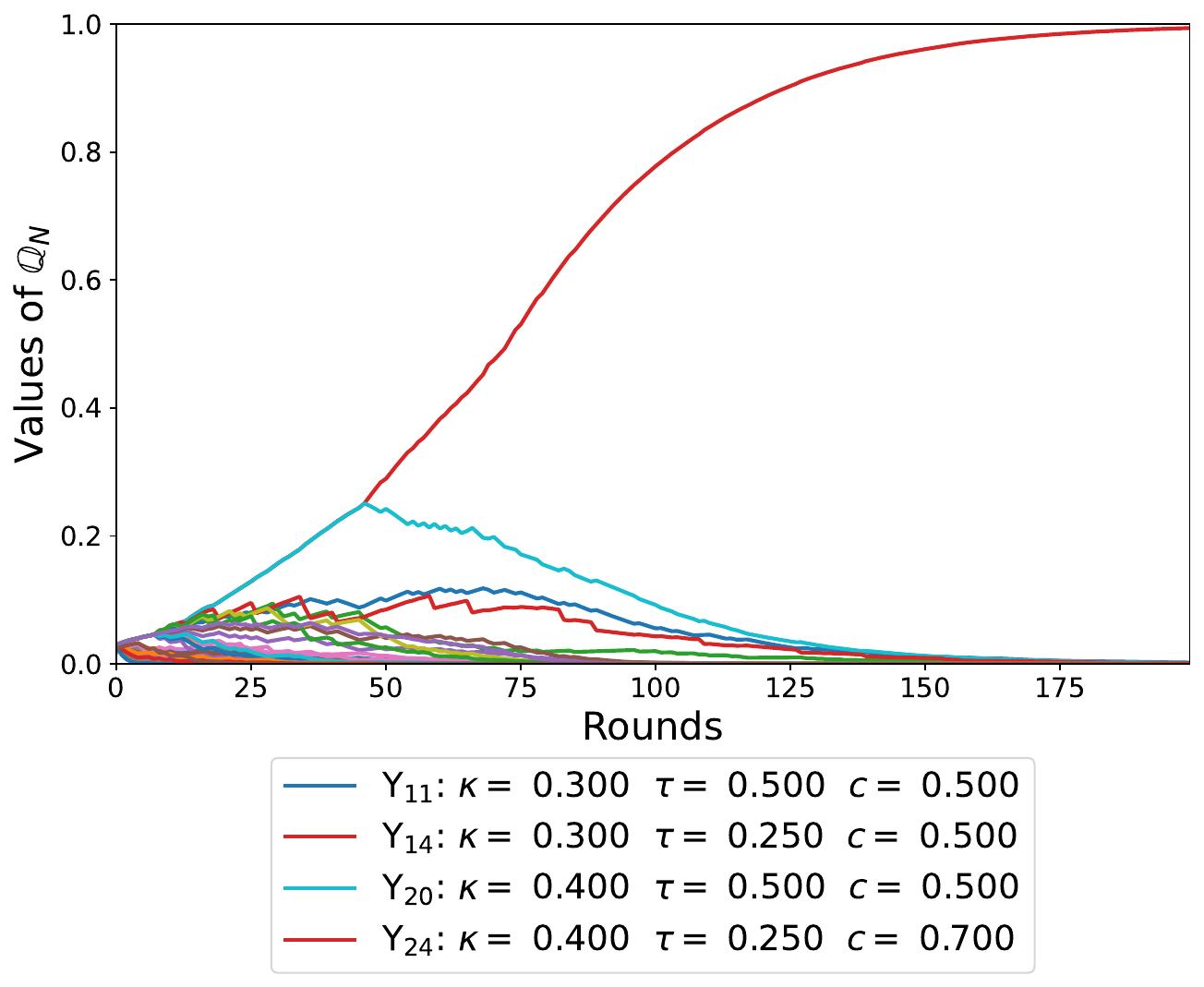}
    \caption{According to \eqref{eq:DesignGBatch}.}
    \label{fig:Gibbs-OnePeriod-Set1a-Batch}
\end{subfigure}

\caption{Evolution of the designed questions and $\bQ_N$ in the one-period setting for $k=4$.}
\label{fig:OnePeriod-Set1a}
\medskip
\small
Top: Evolution of the selected questions at each round of the learning algorithm, where each point represents the label of the chosen transition probability matrix. Bottom: Evolution of $\bQ_N$ at each round of the learning algorithm, where each line corresponds to one of the many risk aversion candidates. The legends in these plots give the closest risk aversion candidates to the client's risk aversion.
\end{figure}

\begin{figure}[htbp]
\centering

\begin{subfigure}[t]{0.48\linewidth}
    \centering
    \includegraphics[width=0.90\textwidth]{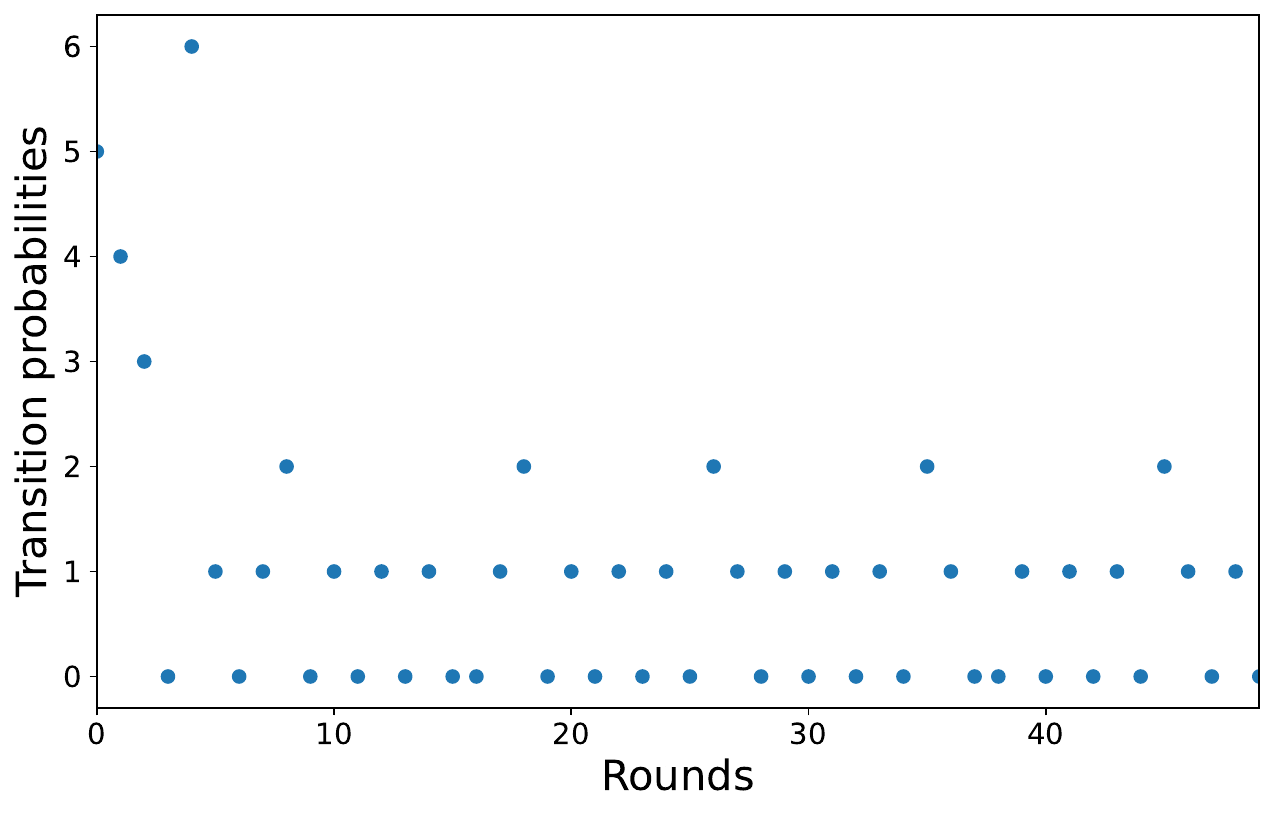}
    \caption{According to \eqref{eq:DesignGOnetoOne}.}
    \label{fig:Games-OnePeriod-Set1b-OnetoOne}
\end{subfigure}
\hfill
\begin{subfigure}[t]{0.48\linewidth}
    \centering
    \includegraphics[width=0.90\textwidth]{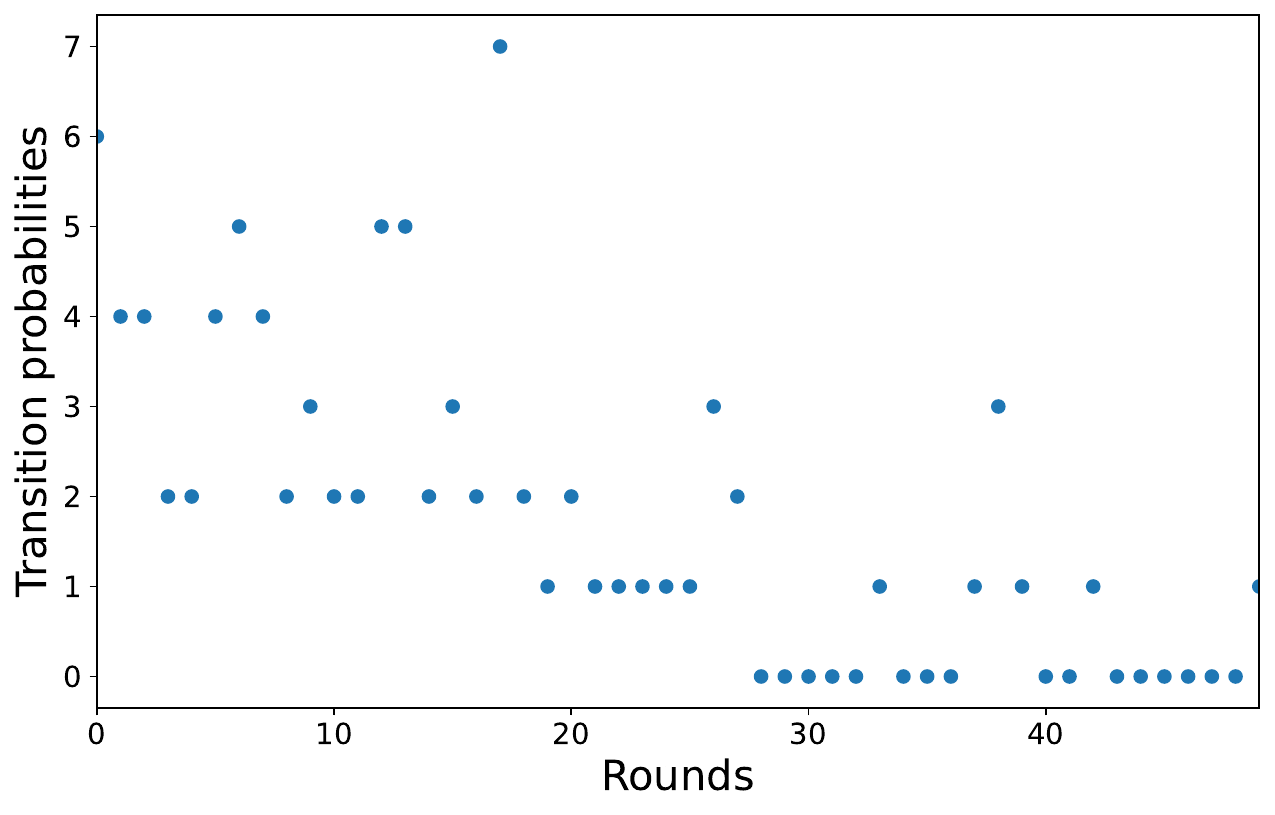}
    \caption{According to \eqref{eq:DesignGBatch}.}
    \label{fig:Games-OnePeriod-Set1b-Batch}
\end{subfigure}

\begin{subfigure}[t]{0.48\linewidth}
    \centering
    \includegraphics[width=0.90\textwidth]{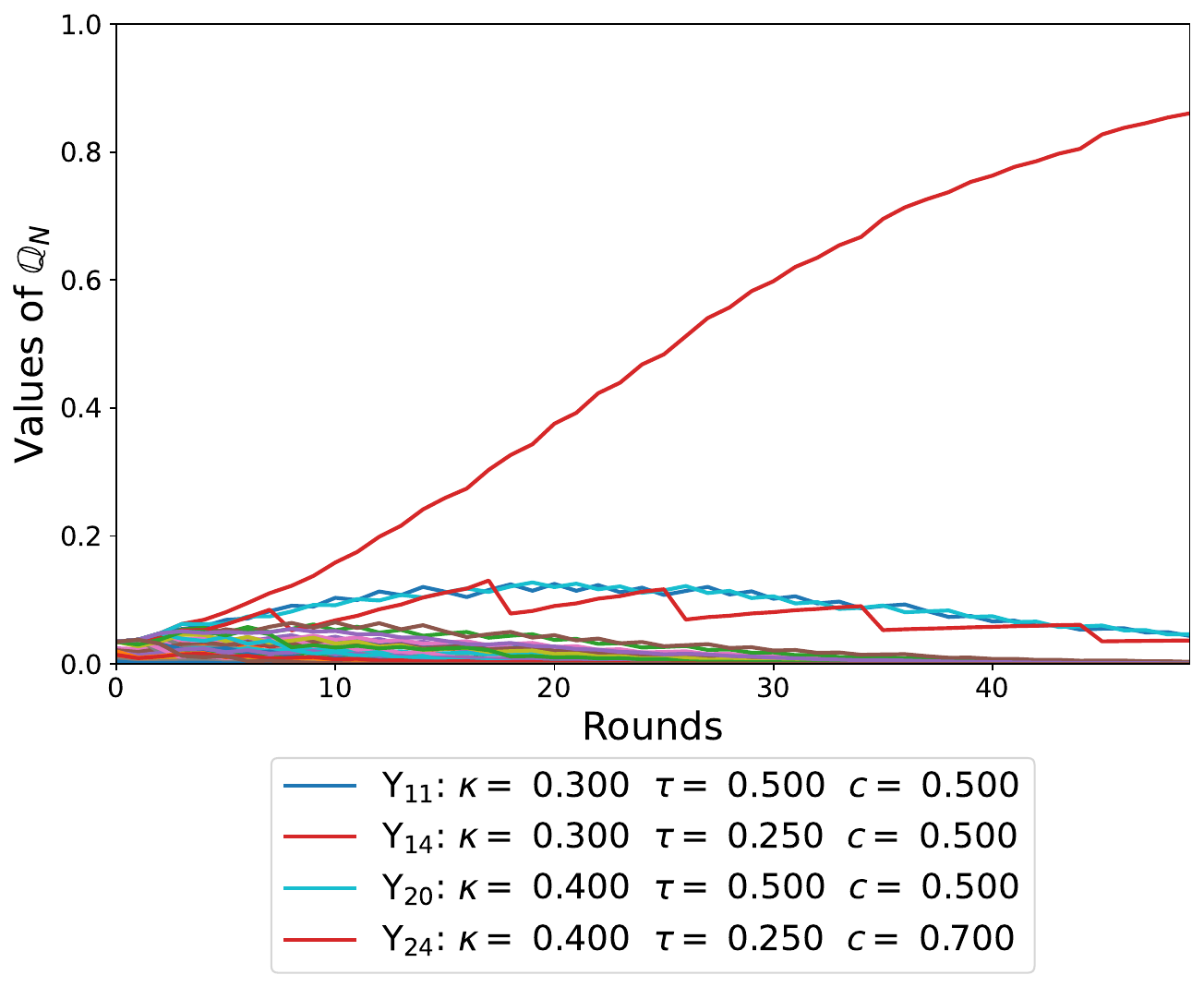}
    \caption{According to \eqref{eq:DesignGOnetoOne}.}
    \label{fig:Gibbs-OnePeriod-Set1b-OnetoOne}
\end{subfigure}
\hfill
\begin{subfigure}[t]{0.48\linewidth}
    \centering
    \includegraphics[width=0.90\textwidth]{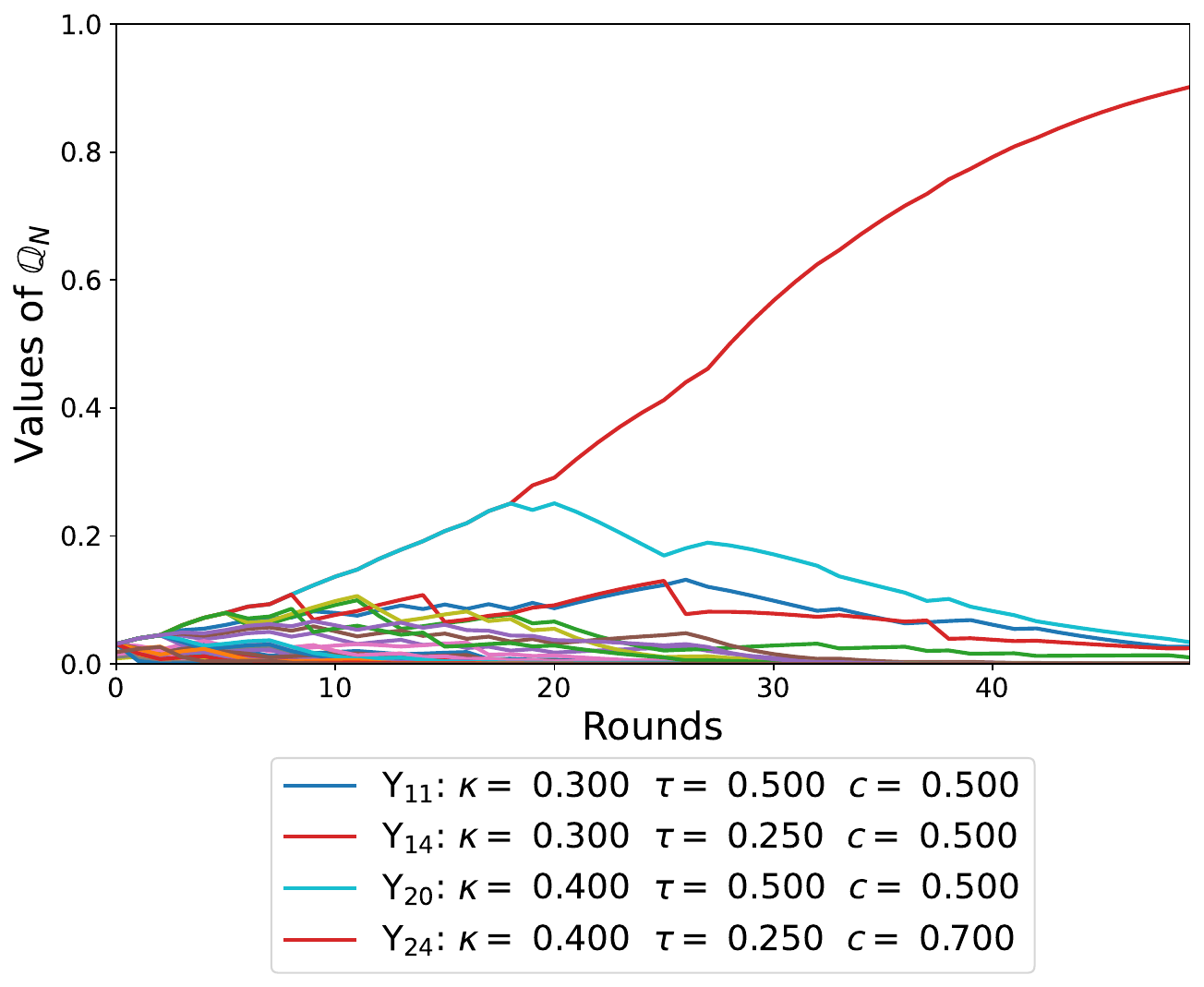}
    \caption{According to \eqref{eq:DesignGBatch}.}
    \label{fig:Gibbs-OnePeriod-Set1b-Batch}
\end{subfigure}

\caption{Evolution of the designed questions and $\bQ_N$ in the one-period setting for $k=10$.}
\label{fig:OnePeriod-Set1b}
\medskip
\small
Top: Evolution of the selected questions at each round of the learning algorithm, where each point represents the label of the chosen transition probability matrix. Bottom: Evolution of $\bQ_N$ at each round of the learning algorithm, where each line corresponds to one of the many risk aversion candidates. The legends in these plots give the closest risk aversion candidates to the client's risk aversion.
\end{figure}

%% file: fig-oneperiod-set2-msf.tex
\begin{figure}[htbp]
\centering

\begin{subfigure}[t]{0.48\linewidth}
    \centering
    \includegraphics[width=0.93\textwidth]{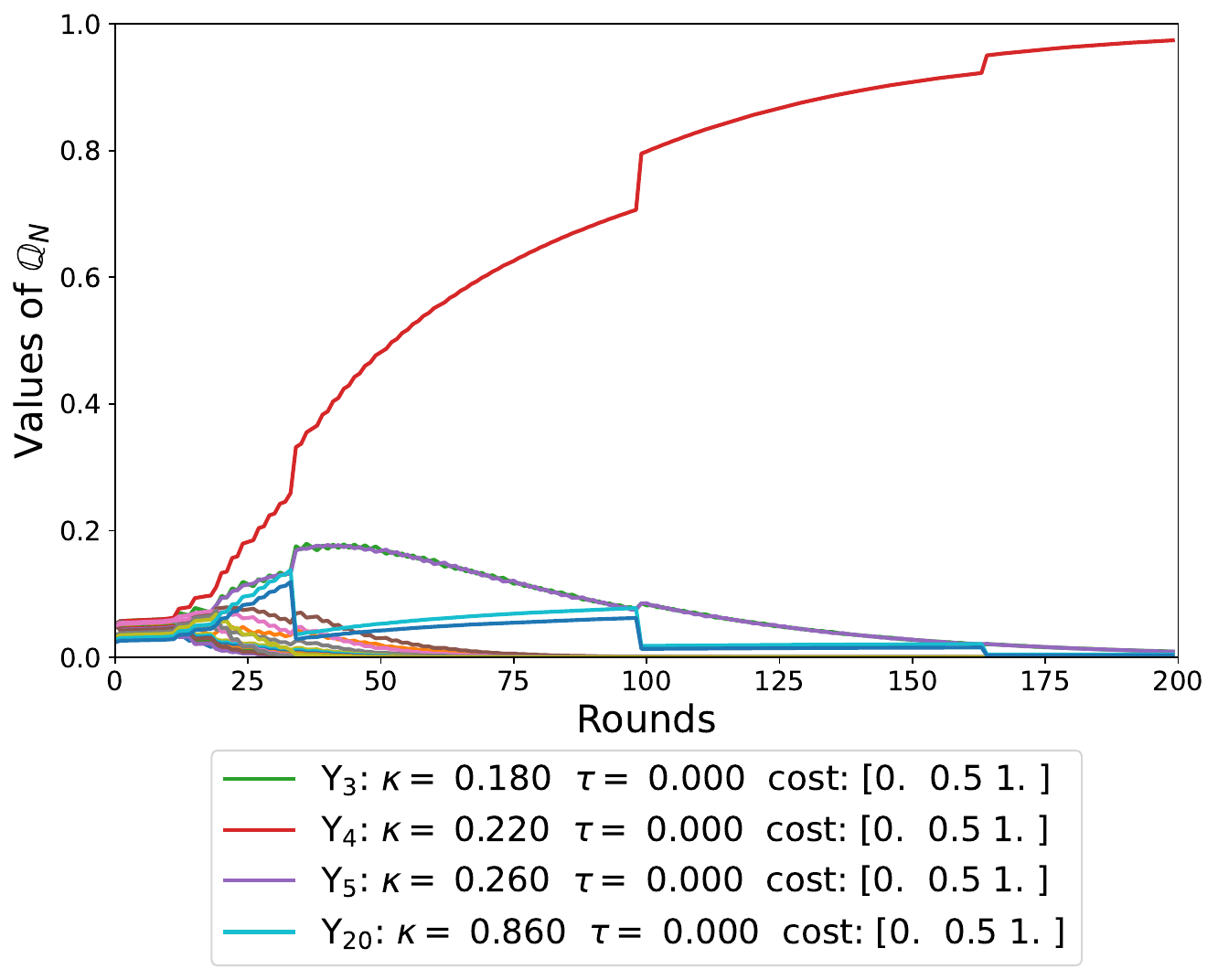}
    \caption{According to \eqref{eq:DesignGOnetoOne}.}
\end{subfigure}
\hfill
\begin{subfigure}[t]{0.48\linewidth}
    \centering
    \includegraphics[width=0.93\textwidth]{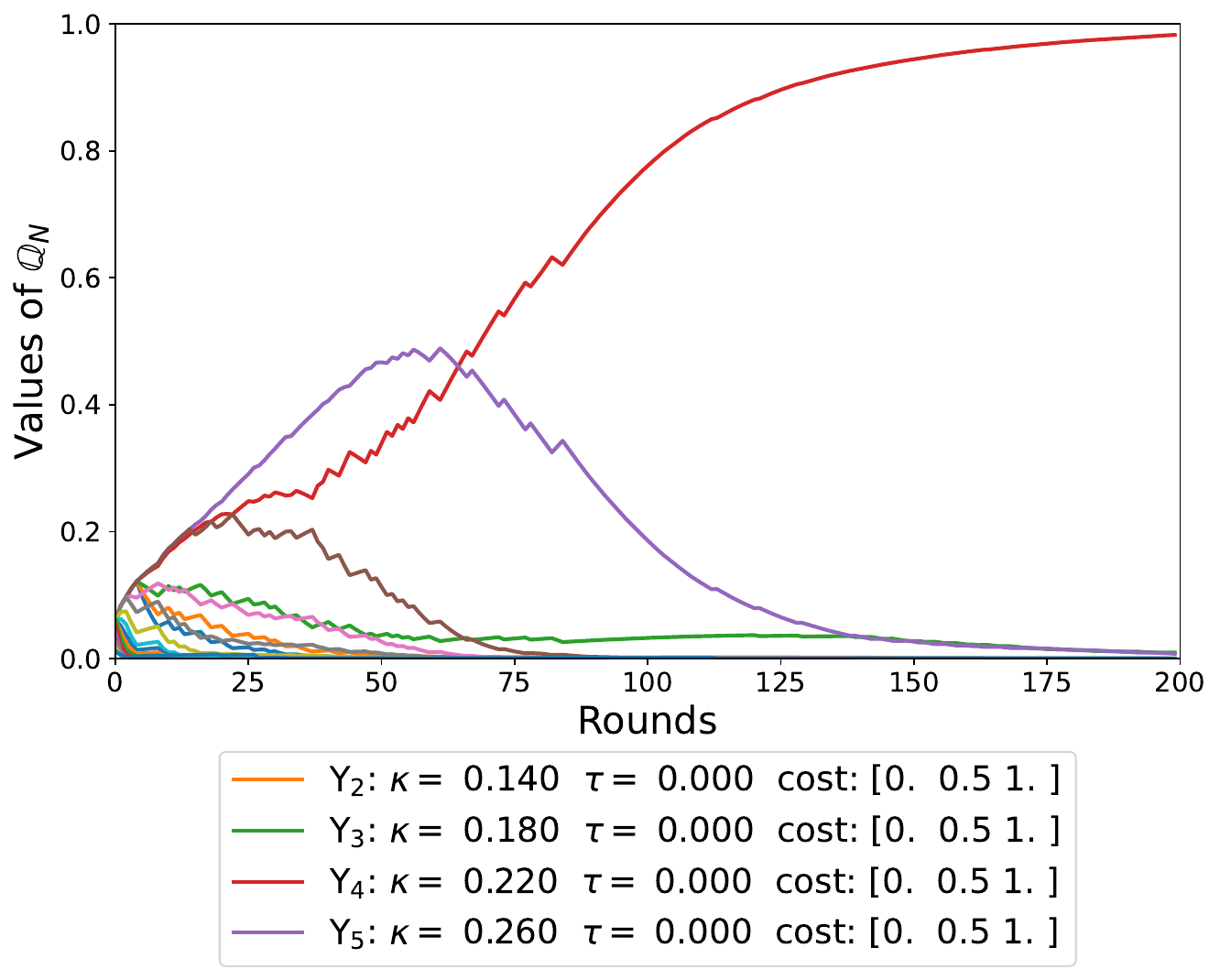}
    \caption{According to \eqref{eq:DesignGBatch}.}
\end{subfigure}
\hfill
\begin{subfigure}[t]{0.48\linewidth}
    \centering
    \includegraphics[width=0.93\textwidth]{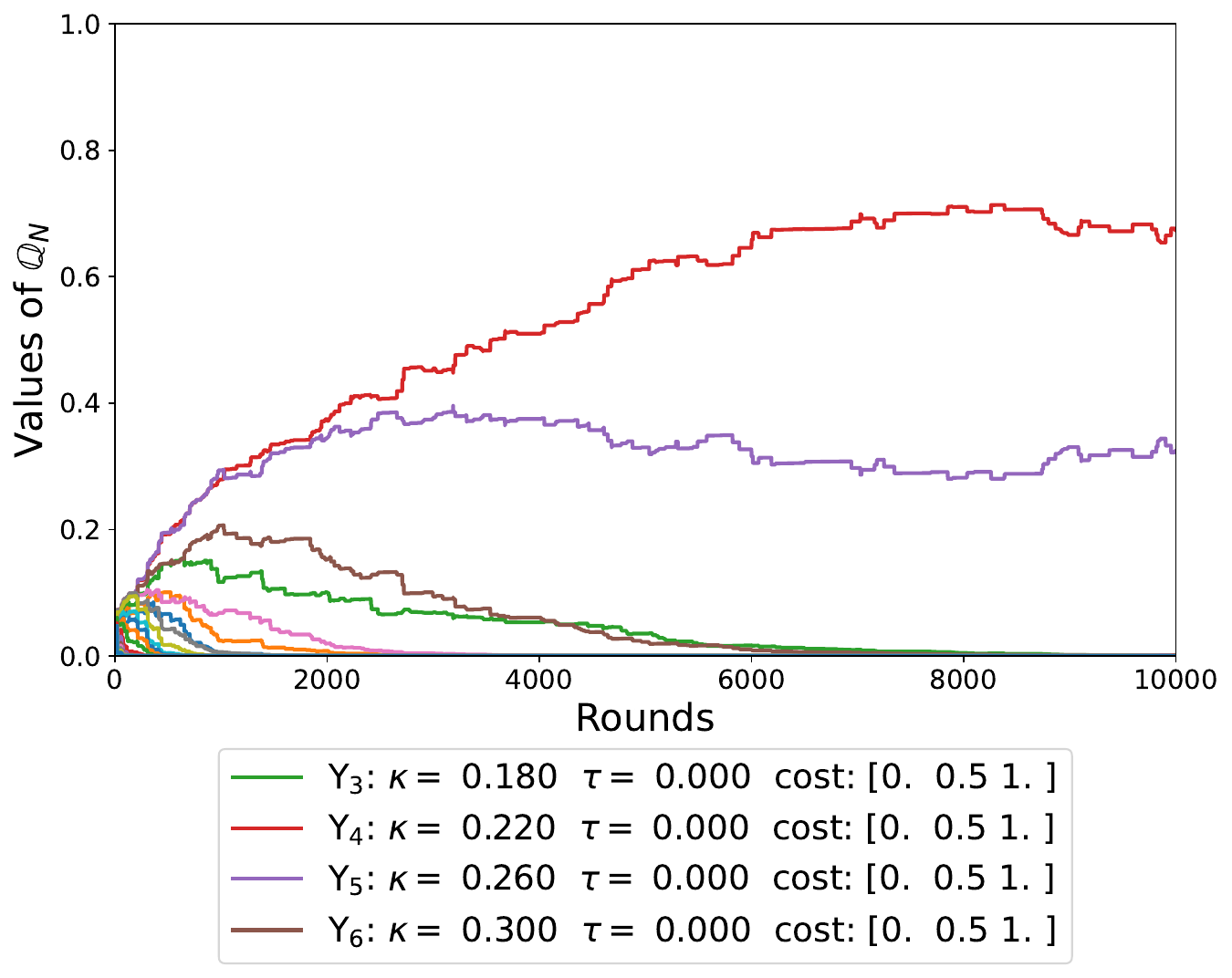}
    \caption{Fully at random.}
\end{subfigure}

\caption{Evolution of $\bQ_N$ in the one-period setting with misspecification.}
\label{fig:Gibbs-OnePeriod-Set2}
\medskip
\small
Evolution of $\bQ_N$ at each round of the learning algorithm, where each line corresponds to one of the many risk aversion candidates. The client's true risk-aversion does not belong to the risk aversion candidates. 
\end{figure}

%% file: fig-oneperiod-set4-particle.tex
\begin{figure}[htbp]
\centering

\begin{subfigure}[t]{0.48\linewidth}
    \centering
    \includegraphics[width=0.98\textwidth]{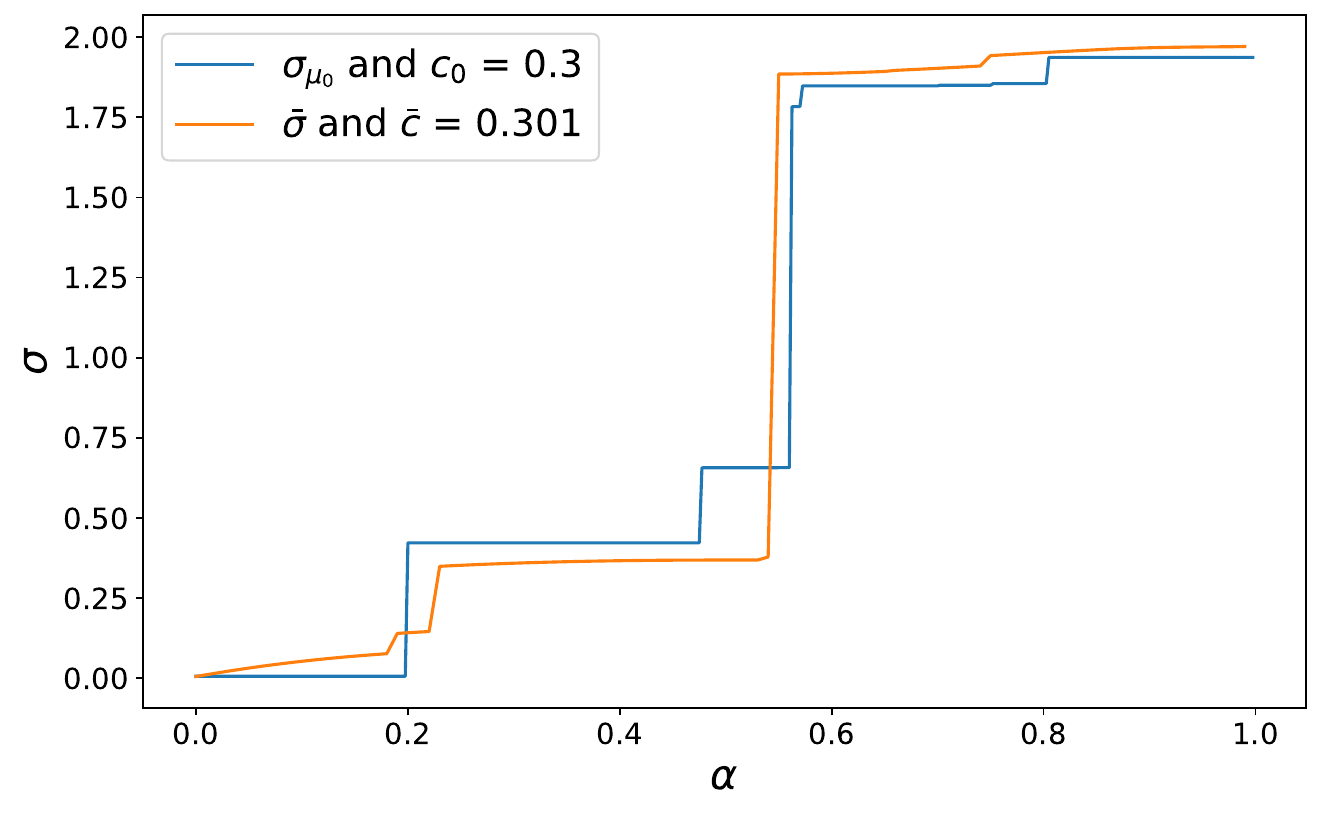}
    \caption{Estimation of $\sigma_{\mu_0}$.}
    \label{fig:Sigma-Set4}
\end{subfigure}
\hfill
\begin{subfigure}[t]{0.48\linewidth}
    \centering
    \includegraphics[width=0.98\textwidth]{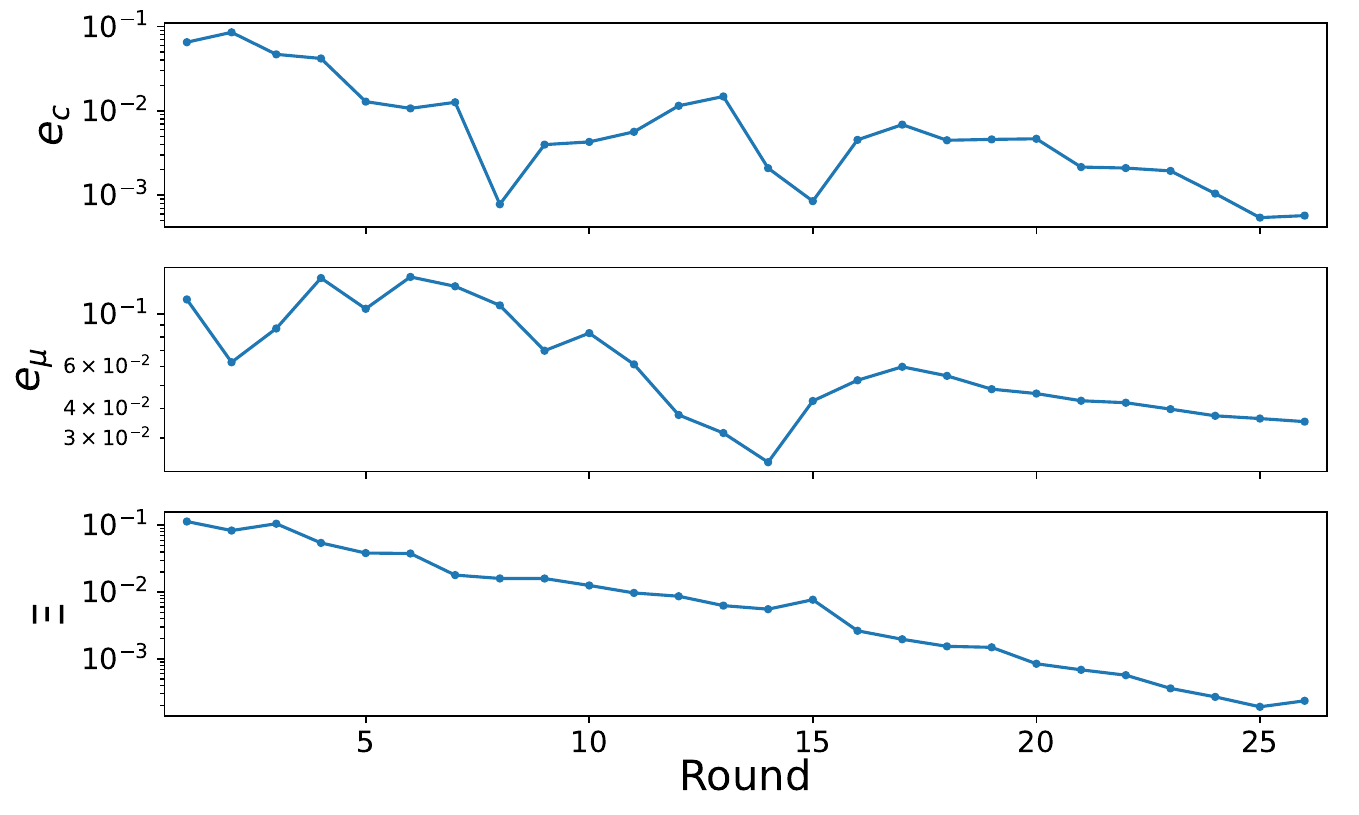}
    \caption{Progression of error metrics at each question.}
    \label{fig:ErrorsConvergence-Set4}
\end{subfigure}
\hfill
\begin{subfigure}[t]{0.48\linewidth}
    \centering
    \includegraphics[width=0.98\textwidth]{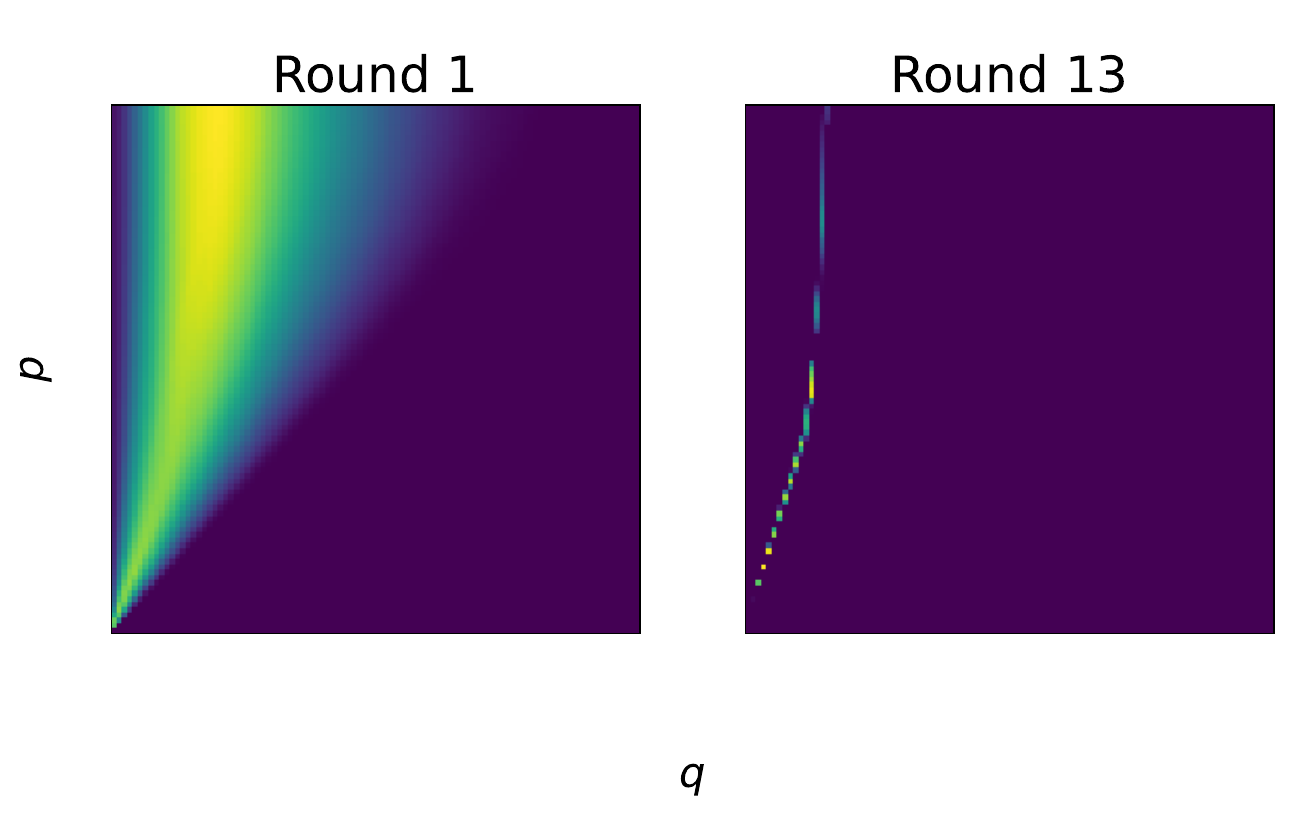}
    \caption{Distinguishing power.}
    \label{fig:GamePower-Set4}
\end{subfigure}
\hfill
\begin{subfigure}[t]{0.48\linewidth}
    \centering
    \includegraphics[width=0.98\textwidth]{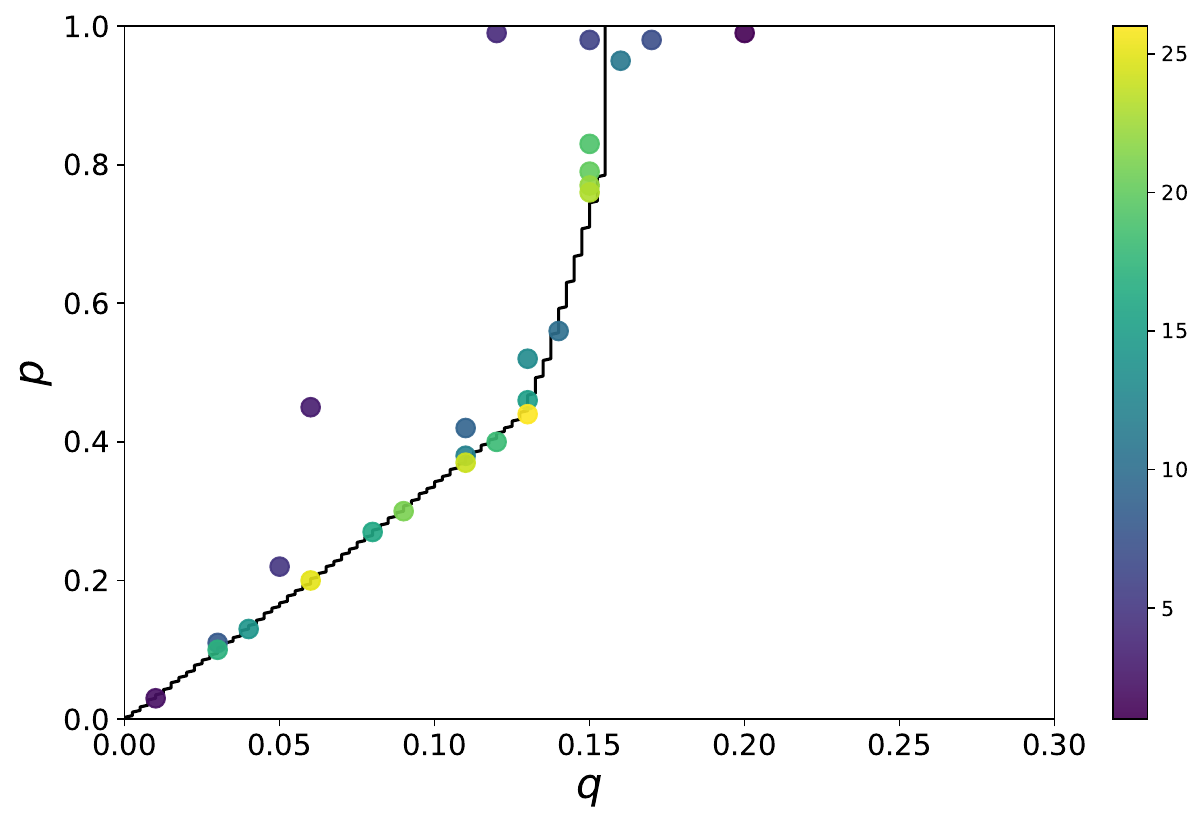}
    \caption{Selected questions and indifference line.}
    \label{fig:GameSelection-Set4}
\end{subfigure}

\caption{Results of the particle-based algorithm.}\label{fig:Set4}
\end{figure}

\begin{figure}[htbp]
\centering

\begin{subfigure}[t]{0.48\linewidth}
    \centering
    \includegraphics[width=0.98\textwidth]{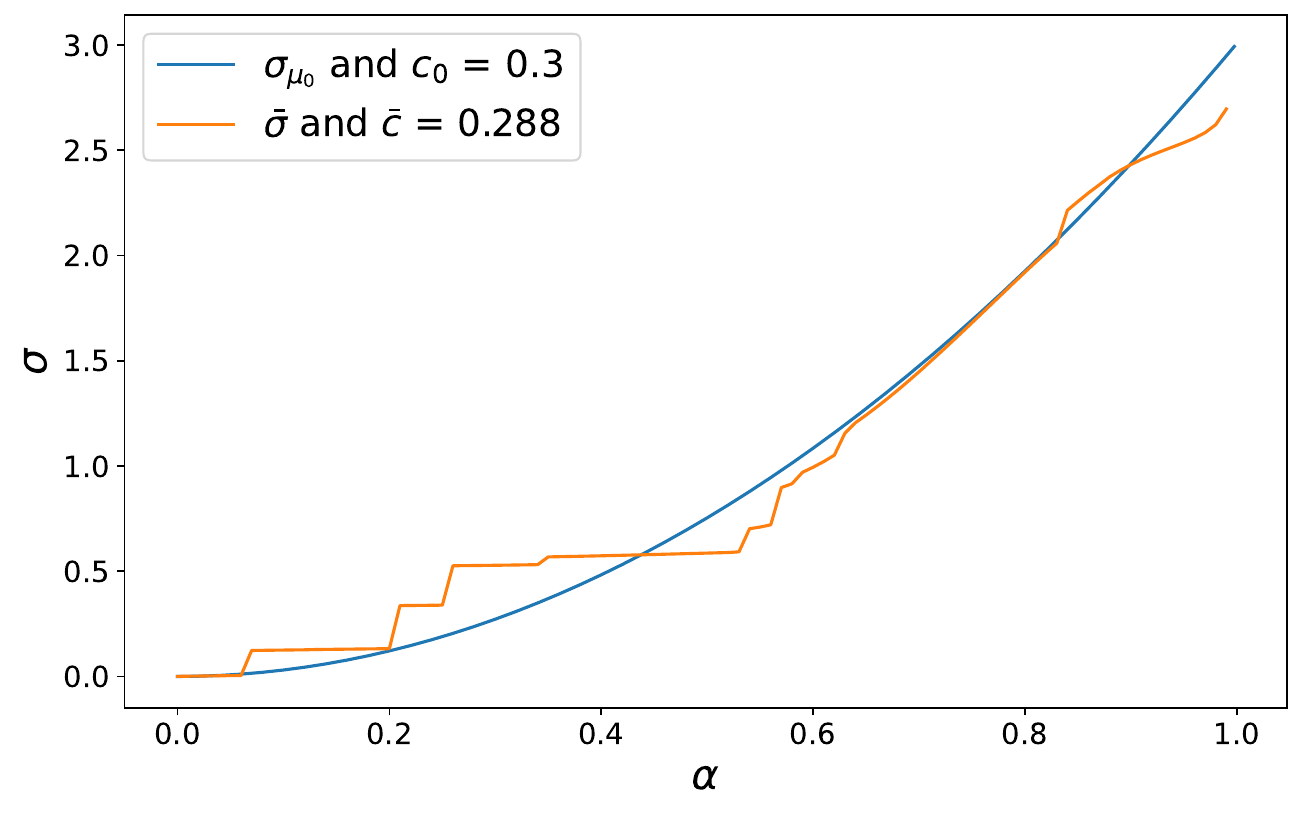}
    \caption{Estimation of $\sigma_{\mu_0}$.}
\end{subfigure}
\hfill
\begin{subfigure}[t]{0.48\linewidth}
    \centering
    \includegraphics[width=0.98\textwidth]{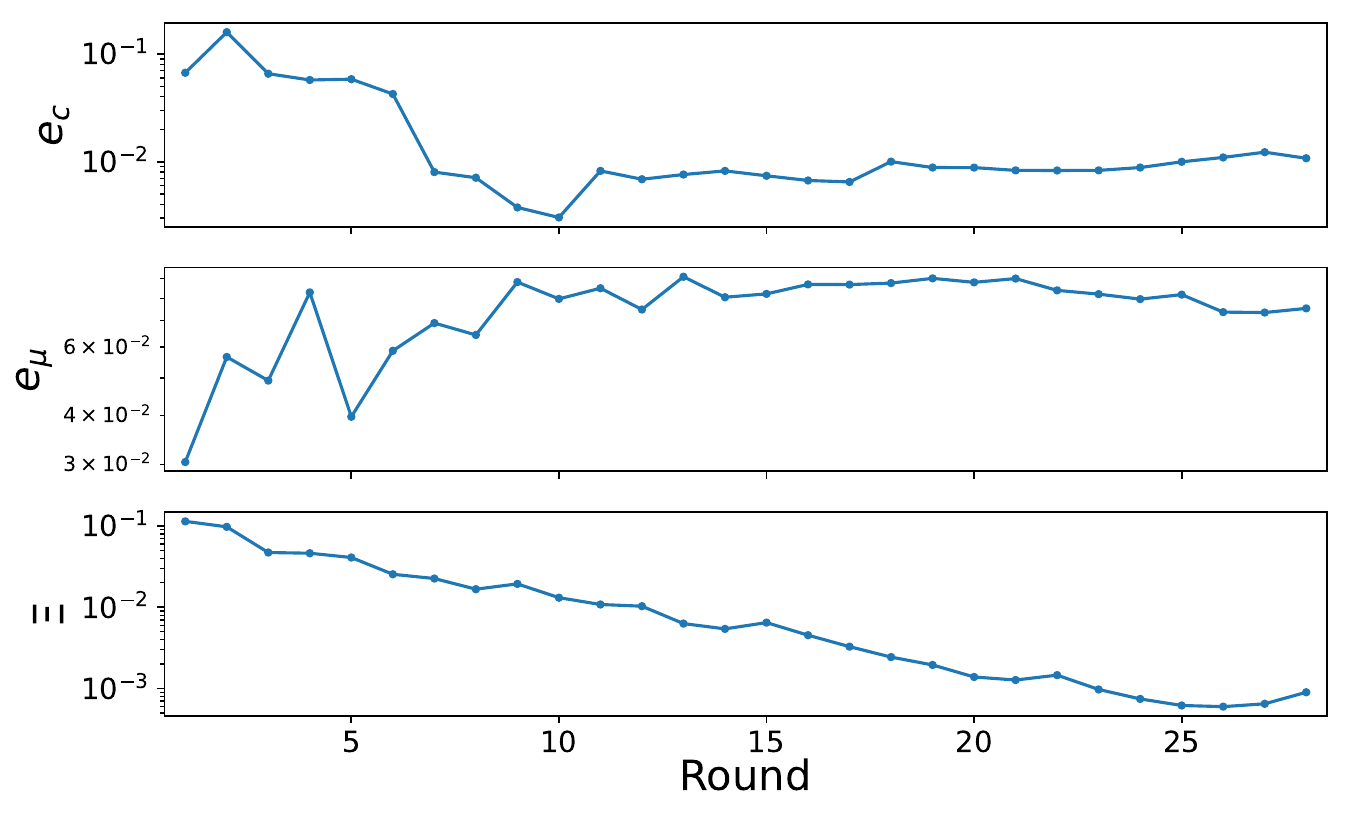}
    \caption{Progression of error metrics at each question.}
\end{subfigure}
\hfill
\begin{subfigure}[t]{0.48\linewidth}
    \centering
    \includegraphics[width=0.98\textwidth]{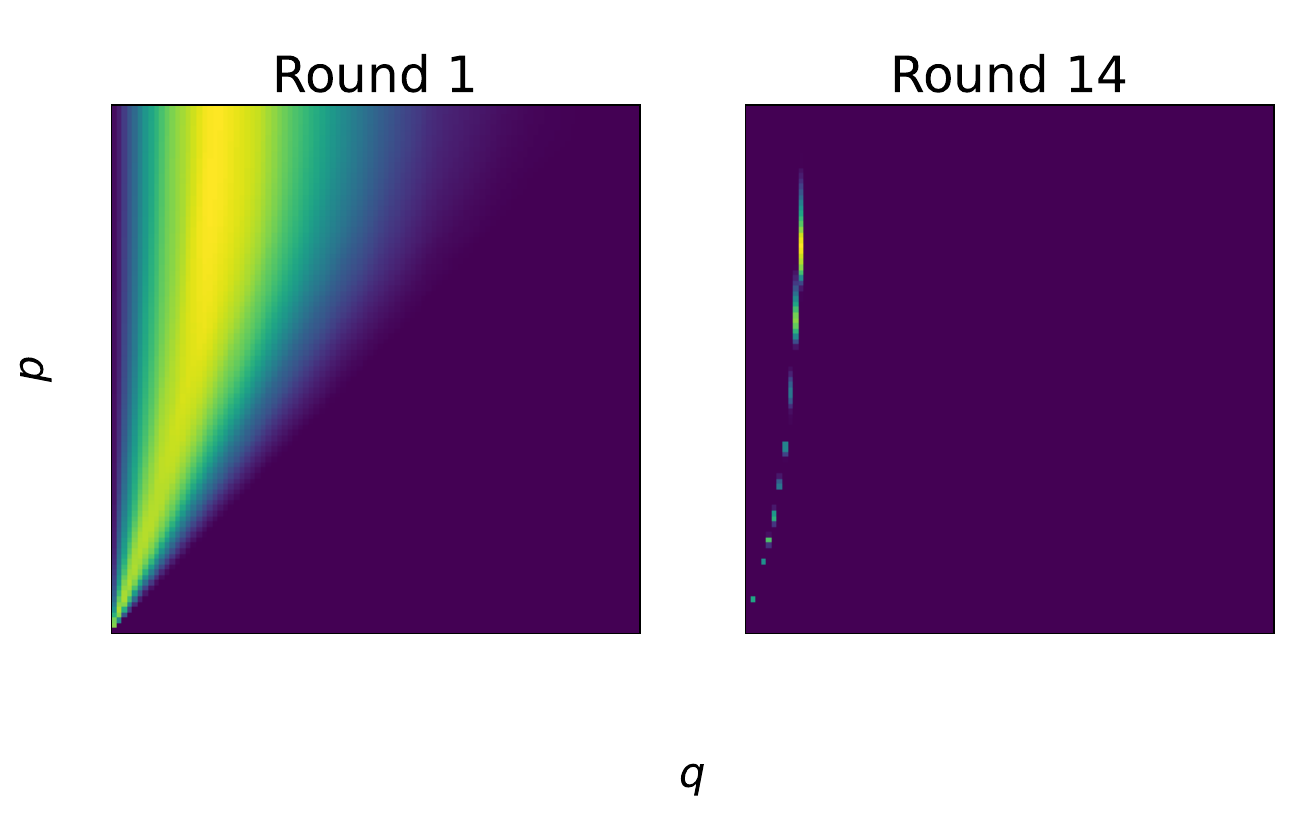}
    \caption{Distinguishing power.}
\end{subfigure}
\hfill
\begin{subfigure}[t]{0.48\linewidth}
    \centering
    \includegraphics[width=0.98\textwidth]{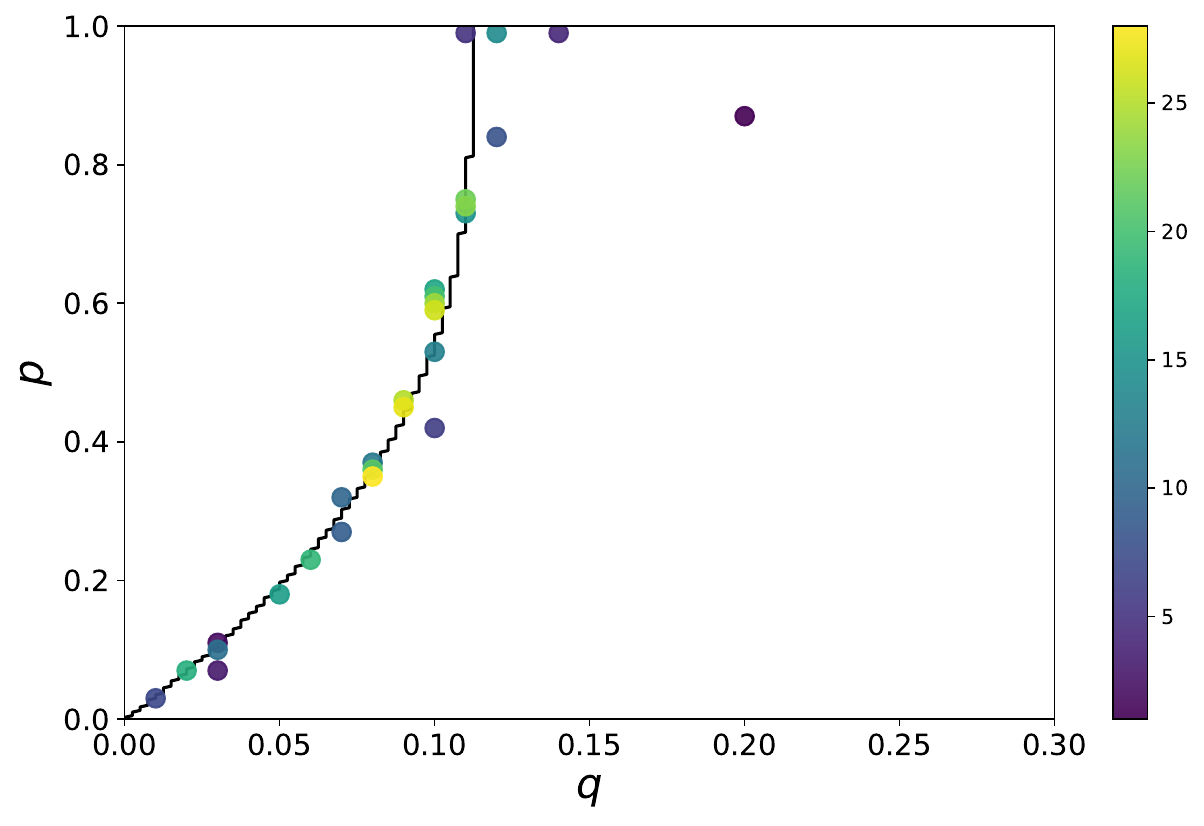}
    \caption{Selected questions and indifference line.}
\end{subfigure}

\caption{Results of the particle-based algorithm.}\label{fig:Set4-1}
\end{figure}

\begin{figure}[htbp]
\centering

\begin{subfigure}[t]{0.48\linewidth}
    \centering
    \includegraphics[width=0.98\textwidth]{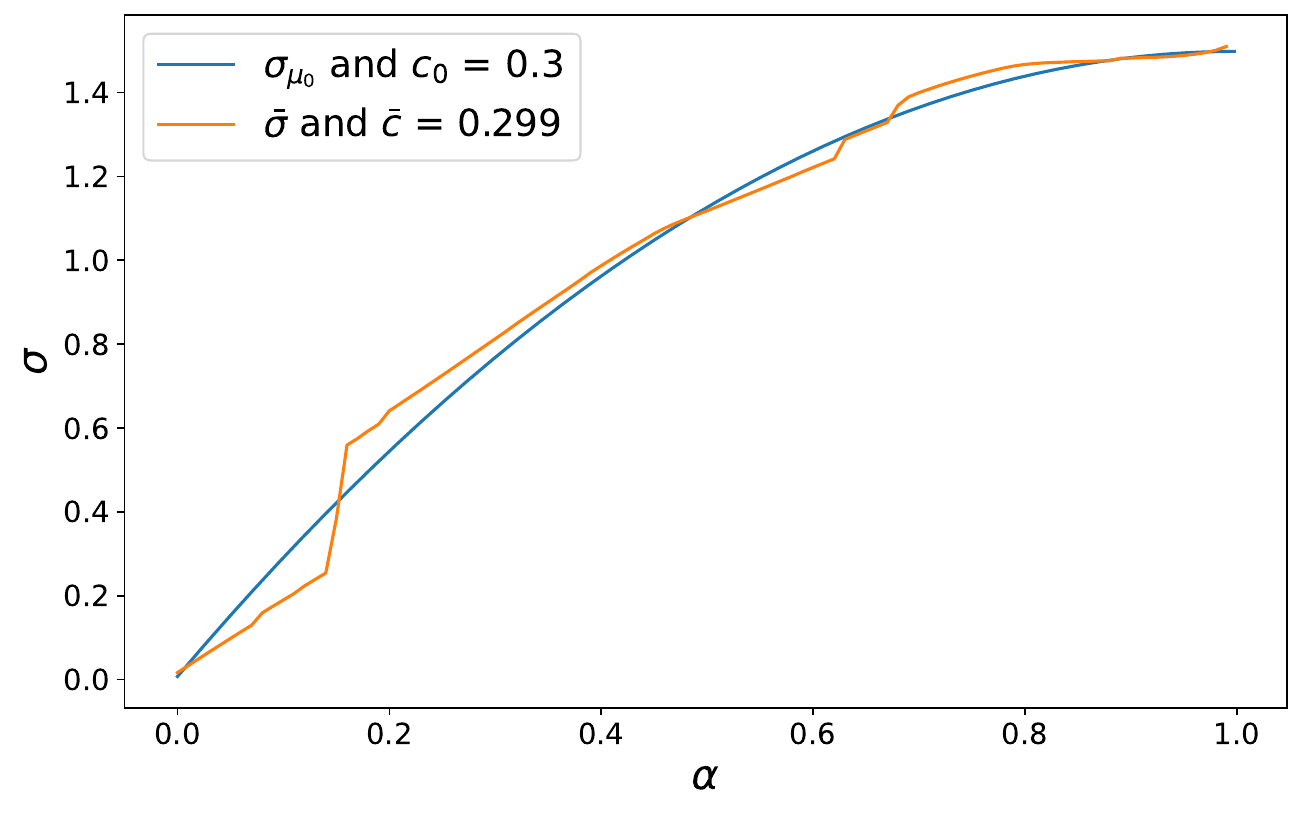}
    \caption{Estimation of $\sigma_{\mu_0}$.}
\end{subfigure}
\hfill
\begin{subfigure}[t]{0.48\linewidth}
    \centering
    \includegraphics[width=0.98\textwidth]{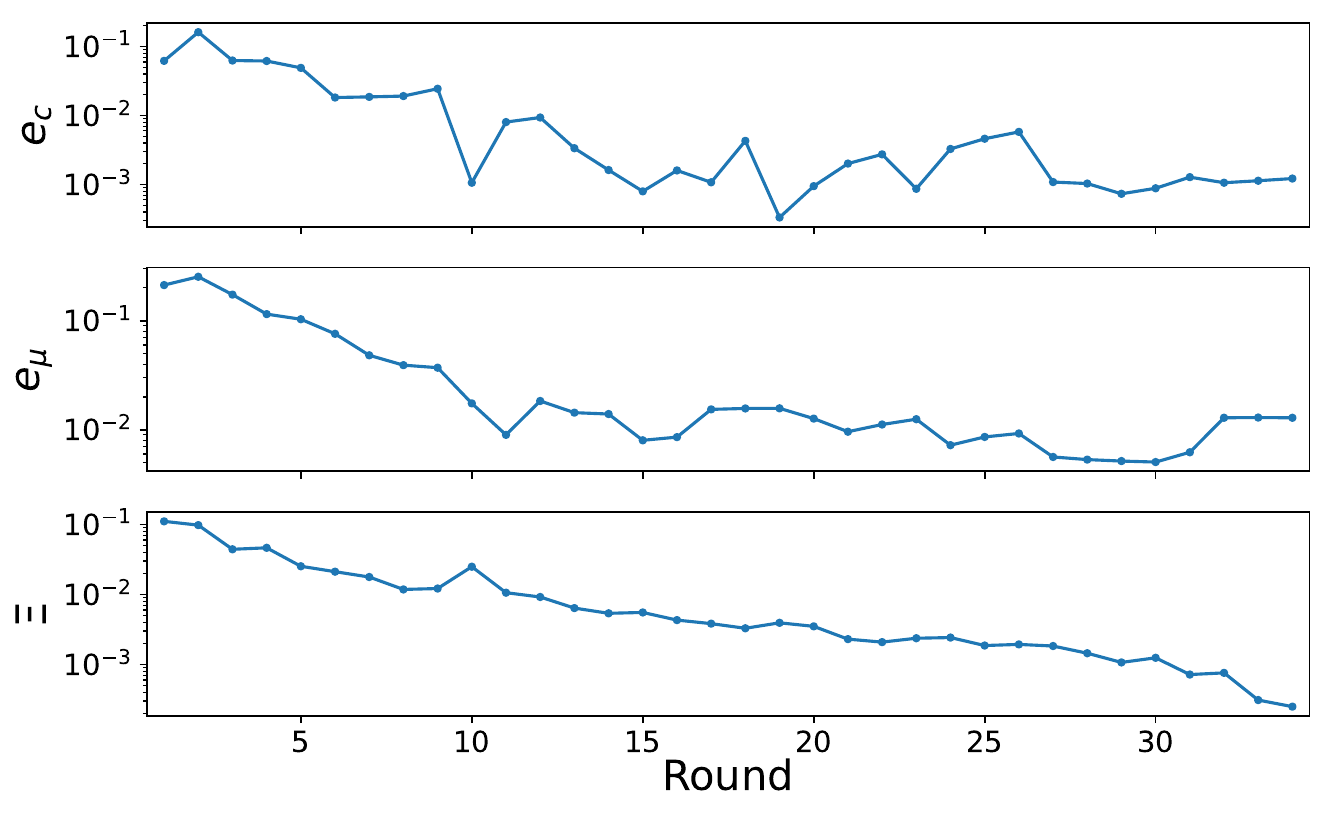}
    \caption{Progression of error metrics at each question.}
\end{subfigure}
\hfill
\begin{subfigure}[t]{0.48\linewidth}
    \centering
    \includegraphics[width=0.98\textwidth]{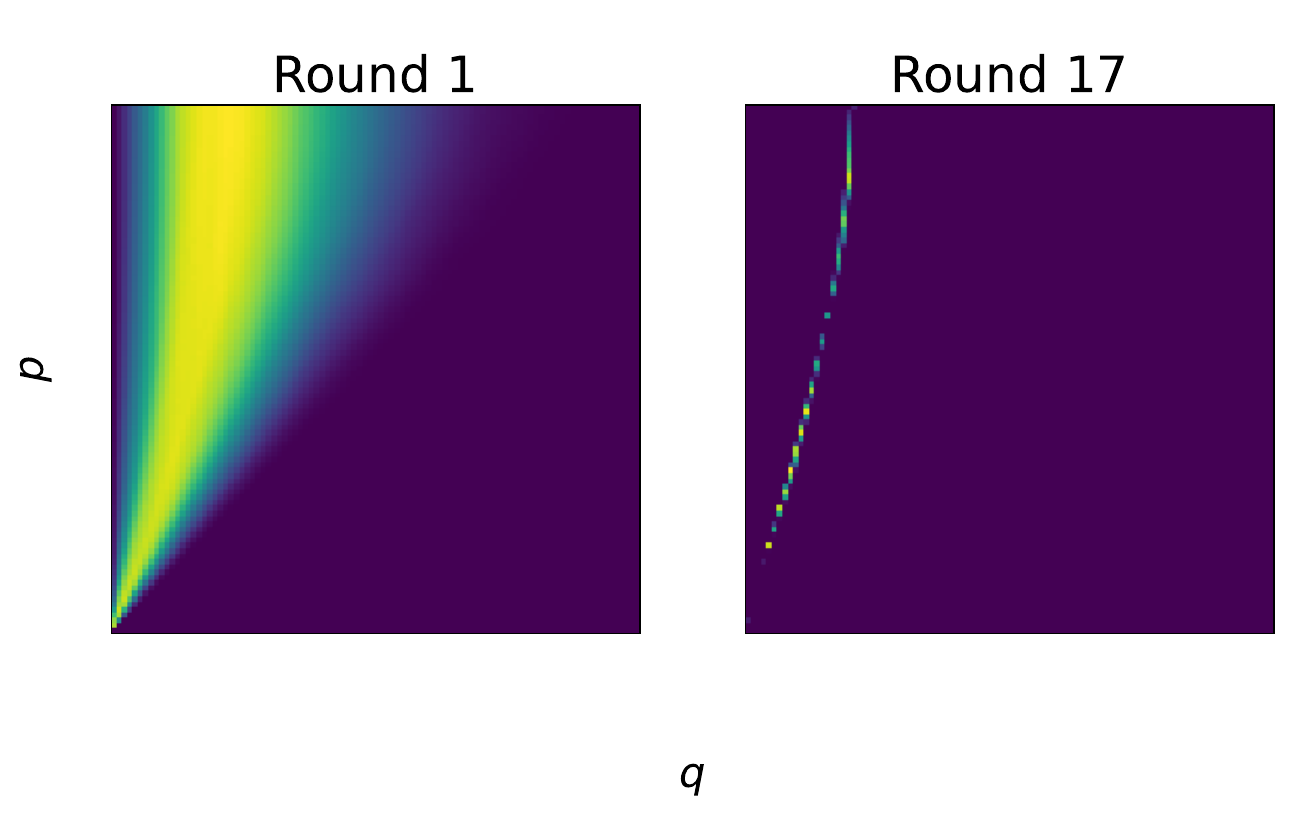}
    \caption{Distinguishing power.}
\end{subfigure}
\hfill
\begin{subfigure}[t]{0.48\linewidth}
    \centering
    \includegraphics[width=0.98\textwidth]{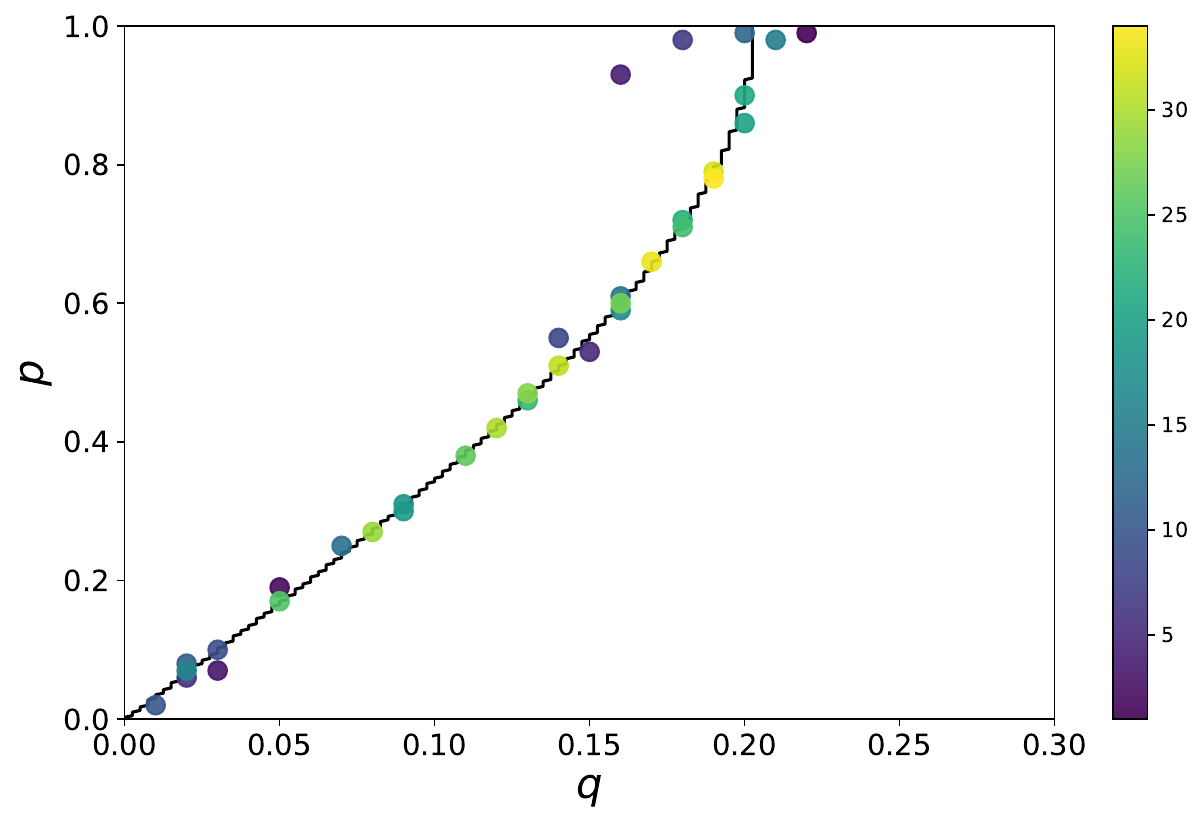}
    \caption{Selected questions and indifference line.}
\end{subfigure}

\caption{Results of the particle-based algorithm.}\label{fig:Set4-2}
\end{figure}

%% file: fig-oneperiod-set5-repeated.tex
\begin{figure}[htbp]
\centering
\includegraphics[width=0.90\textwidth]{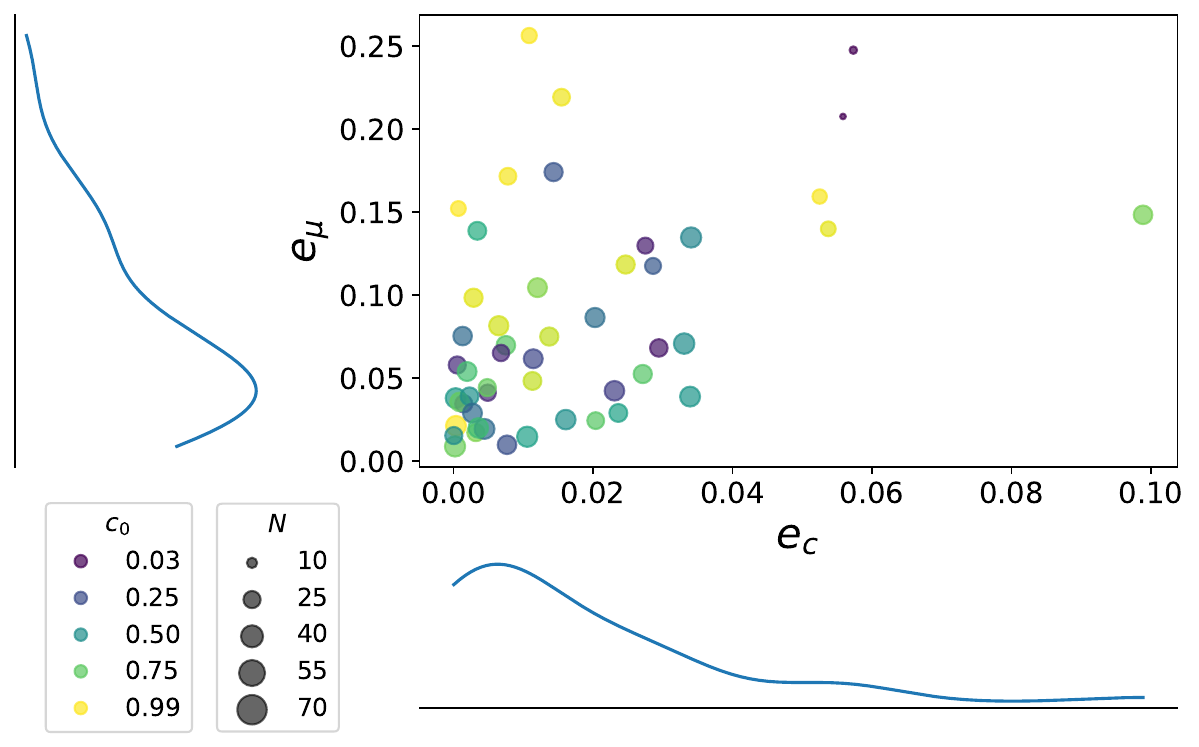}
\caption{Error-error plot for the repeated experiments analysis.}\label{fig:Set5}
\medskip
\small
Errors in terms of both $c_0$ and $\mu$ for each of the 50 random runs, where the color indicates the value of $c_0$ and the size gives the number of questions asked by the algorithm.
\end{figure}

\begin{table}[!ht]
\centering
\begin{tabular}{c|r r r | r r r}
    \hline\hline
    & \multicolumn{3}{c|}{All} & \multicolumn{3}{c}{$c_0 \in(0.1,0.9)$ \& $\sum_{j=\lceil0.9J_0\rceil}^{J_0} \mu\left(\frac{j-1}{J_0}\right) \leq 1$} \\
    & \multicolumn{3}{c|}{50 runs} & \multicolumn{3}{c}{28 runs} \\
    \hline
    & Mean & Max & Std. dev. & Mean & Max & Std. dev. \\
    \hline
    Number of questions & $26.78$ & $35.00$ & $6.13$ & $28.96$ & $35.00$ &  $3.48$ \\
    Particles remaining & $623.22$ & $946.00$ & $401.63$ & $677.82$ & $929.00$ & $360.09$ \\
    Error for $c$ & $0.0167$ & $0.0989$ & $0.0194$ & $0.0095$ & $0.0339$ & $0.0102$ \\
    Error for $\mu$ & $0.0826$ & $0.2564$ & $0.0642$ & $0.0465$ & $0.1741$ & $0.0362$ \\
    Terminal $\Psi$ & $0.0053$ & $0.1511$ & $0.0260$ & $3.8 \times 10^{-5}$ & $2.5\times 10^{-4}$ & $5.8 \times 10^{-5}$ \\
    \hline\hline
\end{tabular}
\caption{Descriptive statistics for the repeated experiments analysis.}
\label{tab:Set5}
\end{table}